\documentclass[twoside,11pt]{article}

\usepackage{thmtools} 
\usepackage{amssymb, amsmath, mathtools} 
\usepackage{graphicx, float} 
\usepackage{chngcntr} 
\usepackage{bbm} 
\usepackage{indentfirst} 
\usepackage{microtype} 
\usepackage{tikz}
\usetikzlibrary{matrix}
\usepackage{graphicx}
\usepackage{booktabs} 
\usepackage{xparse}   
\NewDocumentCommand{\rot}{O{45} O{1em} m}{\makebox[#2][l]{\rotatebox{#1}{#3}}}%
\usepackage{tabularx}

\usepackage{xcolor} 
\usepackage{subcaption}
\usepackage{enumerate}
\usepackage{comment}
\usepackage{thmtools} 
\usepackage{algorithm, algorithmicx, algpseudocode}
\usepackage{cuted}
\usepackage{arydshln}

\usepackage{jmlr2e}

\graphicspath{ {../figures/} }

\renewcommand{\P}{\mathbb{P}} 
\newcommand{\E}{\mathbb{E}} 
\newcommand{\Var}{\mathrm{Var}} 
\newcommand{\Cov}{\mathrm{Cov}} 
\newcommand{\Cor}{\mathrm{Cor}} 
\newcommand\indep{\protect\mathpalette{\protect\independenT}{\perp}} \def\independenT#1#2{\mathrel{\rlap{$#1#2$}\mkern2mu{#1#2}}} 
\newcommand{\1}{\mathbbm{1}} 
\newcommand{\R}{\mathbb{R}} 
\newcommand{\N}{\mathbb{N}} 
\renewcommand{\O}{\mathcal{O}} 
\newcommand{\iid}{\stackrel{i.i.d.}{\sim}} 
\DeclareMathOperator*{\argmin}{arg\,min} 
\DeclareMathOperator*{\argmax}{arg\,max} 
\DeclarePairedDelimiter\norm{\lVert}{\rVert} 
\newcommand\myeq{\mkern2.5mu{=}\mkern2.5mu} 
\renewcommand\mid{\mkern4mu{|}\mkern4mu} 

\newcommand{\F}{\ensuremath{{\mathcal F}}}
\renewcommand{\H}{\ensuremath{{\mathcal H}}}
\newcommand{\A}{\ensuremath{{\mathcal A}}}

\usepackage{lastpage}
\jmlrheading{23}{2022}{1-\pageref{LastPage}}{5/21; Revised 10/22}{10/22}{21-0585}{Domagoj \'{C}evid, Loris Michel, Jeffrey N\"{a}f, Peter B\"{u}hlmann and Nicolai Meinshausen}

\ShortHeadings{Distributional Random Forests}{\'{C}evid, Michel, N\"{a}f, B\"{u}hlmann and Meinshausen}

\firstpageno{1}

\begin{document}

\title{Distributional Random Forests: Heterogeneity Adjustment and Multivariate Distributional Regression}

\author{\name Domagoj \'{C}evid \email cevid@stat.math.ethz.ch
\AND
\name Loris Michel \email michel@stat.math.ethz.ch
\AND
\name Jeffrey N\"{a}f \email naef@stat.math.ethz.ch
\AND
\name Peter B\"{u}hlmann \email buhlmann@stat.math.ethz.ch
\AND
\name Nicolai Meinshausen \email meinshausen@stat.math.ethz.ch\\
\addr Seminar f\"{u}r Statistik\\
ETH Z\"{u}rich\\
8092 Z\"{u}rich, Switzerland
}

\editor{Mladen Kolar}

\maketitle

\begin{abstract}
Random Forest \citep{breiman2001random} is a successful and widely used regression and classification algorithm. Part of its appeal and reason for its versatility is its (implicit) construction of a kernel-type weighting function on training data, which can also be used for targets other than the original mean estimation. We propose a novel forest construction for multivariate responses based on their joint conditional distribution, independent of the estimation target and the data model. It uses a new splitting criterion based on the MMD distributional metric, which is suitable for detecting heterogeneity in multivariate distributions. The induced weights define an estimate of the full conditional distribution, which in turn can be used for arbitrary and potentially complicated targets of interest. The method is very versatile and convenient to use, as we illustrate on a wide range of examples. The code is available as \texttt{Python} and \texttt{R} packages \textbf{\texttt{drf}}.
\end{abstract}


\section{Introduction}

In practice, one often encounters heterogeneous data, whose distribution is not constant, but depends on certain covariates. For example, data can be collected from several different sources, its distribution might differ across certain subpopulations or it could even change with time, etc. Inferring valid conclusions about a certain target of interest from such data can be very challenging as many different aspects of the distribution could potentially change. As an example, in medical studies, the effectiveness of a certain treatment might not be constant throughout the population but depend on certain patient characteristics such as age, race, gender, or medical history. Another issue could be that different patient groups were not equally likely to receive the same treatment in the observed data.

Obviously, pooling all available data together can result in invalid conclusions. On the other hand, if for a given test point of interest one only considers similar training data points, i.e.\ a small homogeneous subpopulation, one may end up with too few samples for accurate statistical estimation. In this paper, we propose a method based on the Random Forest algorithm \citep{breiman2001random} which in a data-adaptive way determines for any given test point which training data points are relevant for it. This in turn can be used for drawing valid conclusions or for accurately estimating any quantity of interest.

Let $\bold{Y}=(Y_1, Y_2, \ldots, Y_d) \in \mathbb{R}^{d}$ be a multivariate random variable representing the data of interest, but whose joint distribution is heterogeneous and depends on some subset of a potentially large number of covariates $\bold{X}=(X_1, X_2, \ldots, X_p) \in \mathbb{R}^{p}$. Throughout the paper, vector quantities are denoted in bold. We aim to estimate a certain target object $\tau(\bold{x})$ that depends on the conditional distribution $\P(\bold{Y} \mid \bold{X} \myeq \bold{x}) = \P(\bold{Y} \mid X_1 \myeq x_1, \ldots, X_p \myeq x_p)$, where $\bold{x} = (x_1, \ldots, x_p)$ is an arbitrary point in $\mathbb{R}^{p}$. The estimation target $\tau(\bold{x})$ can range from simple quantities, such as the conditional expectations $\E[f(\bold{Y}) \mid \bold{X}]$ \citep{breiman2001random} or quantiles $Q_\alpha[f(\bold{Y}) \mid \bold{X}]$ \citep{meinshausen2006quantile} for some function $f:\R^d \to \R$, to some more complicated aspects of the conditional distribution $\P(\bold{Y} \mid \bold{X} \myeq \bold{x})$, such as conditional copulas or conditional independence measures. Given the observed data $\{(\bold{x}_i,\bold{y}_i)\}_{i=1}^n$, the most straightforward way of estimating $\tau(\bold{x})$ nonparametrically would be to consider only the data points in some neighborhood $\mathcal{N}_{\bold{x}}$ around $\bold{x}$, e.g.\ by considering the $k$ nearest neighbors according to some metric. However, such methods typically suffer from the curse of dimensionality even when $p$ is only moderately large: for a reasonably small neighborhood, such that the distribution $\P(\bold{Y}\mid\bold{X}\in \mathcal{N}_\bold{x})$ is close to the distribution $\P(\bold{Y}\mid\bold{X}\myeq\bold{x})$, the number of training data points contained in $\mathcal{N}_{\bold{x}}$ will be very small, thus making the accurate estimation of the target $\tau(\bold{x})$ difficult. The same phenomenon occurs with other methods which locally weight the training observations such as kernel methods \citep{silverman1986density}, local MLE \citep{fan1998local} or weighted regression \citep{cleveland1979robust} even for the relatively simple problem of estimating the conditional mean $\E[\bold{Y}\mid\bold{X}\myeq\bold{x}]$ for fairly small $p$. For that reason, more importance should be given to the training data points $(\bold{x}_i, \bold{y}_i)$ for which the response distribution $\P(\bold{Y}\mid\bold{X}\myeq\bold{x}_i)$ at point $\bold{x}_i$ is similar to the target distribution $\P(\bold{Y}\mid\bold{X}\myeq\bold{x})$, even if $\bold{x}_i$ is not necessarily close to $\bold{x}$ in every component. 

In this paper, we propose the Distributional Random Forest (DRF) algorithm which estimates the multivariate conditional distribution $\P(\bold{Y}\mid\bold{X}\myeq\bold{x})$ in a locally adaptive fashion. This is done by repeatedly dividing the data points in the spirit of the Random Forest algorithm \citep{breiman2001random}: at each step, we split the data points into two groups based on some feature $X_j$ in such a way that the distribution of $\bold{Y}$ for which $X_j \leq l$, for some level $l$, differs the most compared to the distribution of $\bold{Y}$ when $X_j > l$, according to some distributional metric. One can use any multivariate two-sample test statistic, provided it can detect a wide variety of distributional changes. As the default choice, we propose a criterion based on the Maximal Mean Discrepancy (MMD) statistic \citep{gretton2007kernel} with many interesting properties. 
This splitting procedure partitions the data such that the distribution of the multivariate response $\bold{Y}$ in the resulting leaf nodes is as homogeneous as possible, thus defining neighborhoods of relevant training data points for every $\bold{x}$. Repeating this many times with randomization induces a weighting function $w_{\bold{x}}(\bold{x}_i)$ as in \citet{lin2002random, lin2006random}, described in detail in Section \ref{method_description}, which quantifies the relevance of each training data point $\bold{x}_i$ for a given test point $\bold{x}$. The conditional distribution is then estimated by an empirical distribution determined by these weights \citep{meinshausen2006quantile}. This construction is data-adaptive as it assigns more weight to the training points $\bold{x}_i$ that are closer to the test point $\bold{x}$ in the components which are more relevant for the distribution of $\bold{Y}$.

Our forest construction does not depend on the estimation target $\tau(\bold{x})$, but it rather estimates the conditional distribution $\P(\bold{Y} \mid \bold{\mathbf{X}=\mathbf{x}})$ directly and the induced forest weights can be used to estimate $\tau(\bold{x})$ in a second step. This approach has several advantages. First, only one DRF fit is required to obtain estimates of many different targets, which has a big computational advantage. Furthermore, since those estimates are obtained from the same forest fit, they are mutually compatible. For example, if the conditional correlation matrix $\{\Cor(Y_i,\, Y_j \mid \bold{X}\myeq\bold{x})\}_{i,j=1}^d$ is estimated componentwise using some other method, the resulting matrix might not be positive semidefinite, and as another example, the CDF estimates $\hat{\P}(\bold{Y} \leq \bold{y} \mid \bold{X}\myeq\bold{x})$ might not be monotone in $\bold{y}$, see Figure \ref{comparison}. Finally, it could be extremely difficult to tailor forest construction to some complex targets $\tau(\bold{x})$.
The induced weighting function can thus be used not only for obtaining simple distributional aspects such as, for example, the conditional quantiles, conditional correlations, or joint conditional probability statements, but also to obtain more complex objectives, such as conditional independence tests \citep{zhang2012kernel}, heterogeneous regression (see also Section 
\ref{sec: causality} for more details) \citep{kunzel2019metalearners, wager2018estimation} or semiparametric estimation by fitting a parametric model for $\bold{Y}$, having nonparametrically adjusted for $\bold{X}$ \citep{bickel1993efficient}. Representation of the conditional distribution via the weighting function has a great potential for applications in causality such as causal effect estimation or as a way of implementing do-calculus \citep{pearl2009causality} for finite samples, as we discuss in Section \ref{sec: causality}.

Therefore, DRF is used in two steps: in the first step, we obtain the weighting function $w_\bold{x}(\cdot)$ describing the conditional distribution $\P(\bold{Y}\mid\bold{X}\myeq\bold{x})$ in a target- and model-free way, which is then used as an input for the second step. Even if the method used in the second step does not directly support weighting of the training data points, one can easily resample the data set with the sampling probabilities equal to $\{w_\bold{x}(\bold{x}_i)\}_{i=1}^n$. This two-step approach is visualized in the following diagram:
\begin{center}
\hspace{2cm}
\begin{tikzpicture}
  \matrix (m) [matrix of math nodes,row sep=2em, column sep=12em,minimum width=2em]
  {
     \P(\bold{Y}\mid\bold{X}\myeq\bold{x}) & \hat{\P}(\bold{Y}\mid\bold{X}\myeq\bold{x}) \\
     \tau(\P) & \tau(\hat{\P}) \\};
  \path[-stealth]
    (m-1-1) edge node [above] {1) get $w_\bold{x}(\cdot)$ with DRF} (m-1-2)
    (m-1-1) edge [dashed] node [left] {objective} (m-2-1)
    (m-1-2) edge node [right]{2) compute from $w_\bold{x}(\cdot)$} (m-2-2)
    (m-2-2) edge [dashed] node [above] {induced estimator} (m-2-1);
\end{tikzpicture}
\end{center}

\subsection{Related work and our contribution}

Several adaptations of the Random Forest algorithm have been proposed for targets beyond the original one of the univariate conditional mean $\E[Y \mid \bold{\mathbf{X}\myeq\mathbf{x}}]$: for survival analysis \citep{hothorn2006survival}, conditional quantiles \citep{meinshausen2006quantile}, density estimation \citep{pospisil2018rfcde}, CDF estimation \citep{hothorn2017transformation} or heterogeneous treatment effects \citep{wager2018estimation}. Almost all such methods use the weights induced by the forest, as described in Section \ref{method_description}, rather than averaging the estimates obtained per tree. This view of Random Forests as a powerful adaptive nearest neighbor method is well known and dates back to \citet{lin2002random, lin2006random}. It was first used for targets beyond the conditional mean in \citet{meinshausen2006quantile}, where the original forest construction with univariate $Y$ was used \citep{breiman2001random}. However, the univariate response setting considered there severely restricts the number of interesting targets $\tau(\bold{x})$ and DRF can thus be viewed as an important generalization of this approach to the multivariate setting.

In order to be able to perform certain tasks or to achieve a better accuracy, many forest-based methods adapt the forest construction by using a custom splitting criterion tailored to their specific target, instead of relying on the standard CART criterion.
In \citet{zeileis2008model} and \citet{hothorn2017transformation}, a parametric model for the response $\bold{Y} \mid \bold{\mathbf{X}\myeq\mathbf{x}} \sim f(\theta(\bold{x}), \cdot)$ is assumed and recursive splitting is performed based on a permutation test which uses the user-provided score functions. Similarly, \citet{athey2019generalized} estimate certain univariate targets for which there exist corresponding score functions defining the local estimating equations. The data is split so that the estimates of the target in resulting child nodes differ the most. This is different, though, to the target-free splitting criterion of DRF, which splits so that the distribution of $\bold{Y}$ in child nodes is as different as possible.

Since the splitting step is extensively used in the algorithm, its complexity is crucial for the overall computational efficiency of the method, and one often needs to resort to approximating the splitting criterion \citep{pospisil2018rfcde, athey2019generalized} to obtain good computational run time. We propose a splitting criterion based on a fast random approximation of the MMD statistic \citep{gretton2012kernel, zhao2015fastmmd}, which is commonly used in practice for two-sample testing as it is able to detect any change in the multivariate distribution of $\bold{Y}$ with good power \citep{gretton2007kernel}. DRF with the MMD splitting criterion also has interesting theoretical properties as shown in Section \ref{sec: theory} below.

The multivariate response case has not received much attention in the Random Forest literature. Most of the existing forest-based methods focus on either a univariate response $Y$ or on a certain univariate target $\tau(\bold{x})$. One interesting line of work considers density estimation \citep{pospisil2018rfcde} and uses aggregation of the CART criteria for different response transformations. Another approach \citep{kocev2007ensembles, segal2011multivariate, ishwaran2014randomforestsrc} is based on aggregating standard univariate CART splitting criteria for $Y_1, \ldots, Y_d$ and targets only the conditional mean of the responses, a task which could also be solved by separate regression fits for each $Y_i$. In order to capture any change in the distribution of the multivariate response $\bold{Y}$, one needs to not only consider the marginal distributions for each component $Y_i$, but also to determine whether their dependence structure changes, see e.g.\ Figure \ref{fig: conditional-independence}.

There is an increasing number of methods that nonparametrically estimate the joint multivariate conditional distribution $\P(\bold{Y} \mid \bold{X}\myeq\bold{x})$ in the statistics and machine learning literature. In addition to a few simple classical methods such as $k$-nearest neighbors and kernel regression, there exist methods based on normalizing flows such as Inverse Autoregressive Flow \citep{kingma2016improving} or Masked Autoregressive Flow \citep{papamakarios2017masked} and also conditional variants of several popular generative models such as Conditional Generative Adversarial Networks \citep{mirza2014conditional} or Conditional Variational Autoencoder \citep{sohn2015learning}. The focus of these methods is more on the settings with large response dimension $d$ and small covariate dimension $p$, such as image or text generation. Another interesting and related line of research focuses on estimating the conditional mean embedding (CME), as described e.g., in \cite{CMEinDynamicalSystems, kernelmeanembeddingreview, CMEinGraphicalModels, OurapproachtoCME}, rather than estimating the conditional distribution directly. CMEs generalize the concept of embedding (marginal) probability distributions into a Reproducing Kernel Hilbert Space (RKHS) to the conditional case. Interestingly, DRF with the MMD-based splitting criterion can also be viewed as a method for estimating the CME, as discussed in Section \ref{sec: theory} below. This viewpoint provides a natural connection between the Random Forest and kernel embedding literature.
A comparison of DRF with the methods for distributional estimation listed above can be found in Section \ref{sec: benchmarks}.

Our contribution, resulting in the proposal of the Distributional Random Forest (DRF), can be summarized as follows: First, we introduce the idea of forest construction based on sequential multivariate two-sample test statistics. It does not depend on a particular estimation target and is completely nonparametric, which makes its implementation and usage very simple and universal. Not only does it not require additional user input such as the log-likelihoods or score functions, but it can be used even for complicated targets for which there is no obvious forest construction. Furthermore, it has a computational advantage as only a single forest fit is needed for producing estimates of many different targets that are additionally compatible with each other.
Second, we propose an MMD-based splitting criterion with good statistical and computational properties, for which we also derive interesting theoretical results in Section \ref{sec: theory}. It underpins our implementation, which we provide as \texttt{R} and \texttt{Python} packages \texttt{drf}.
Finally, we show on a broad range of examples in Section \ref{sec: applications} how many different statistical estimation problems, some of which not being easily tractable by existing forest-based methods, can be cast to our framework, thus illustrating the usefulness and versatility of DRF.

\section{Method} \label{method_description}
In this section we describe the details of the Distributional Random Forest (DRF) algorithm. We closely follow the implementations of the \texttt{grf} \citep{athey2019generalized} and \texttt{ranger} \citep{wright2015ranger} \texttt{R}-packages. A detailed description of the method and its implementation and the corresponding pseudocode can be found in the Appendix \ref{appendix: impdetails}.

\subsection{Forest Building} 

The trees are grown recursively in a model-free and target-free way as follows: 
For every parent node $P$, we determine how to best split it into two child nodes of the form $C_L = \{X_j \leq l\}$ and $C_R = \{X_j > l\}$, where the variable $X_j$ is one of the randomly chosen splitting candidates and $l$ denotes its level based on which we perform the splitting. The split is chosen such that we maximize a certain (multivariate) two-sample test statistic
\begin{equation} \label{eq: splitcrit}
\mathcal{D}\left(\{\bold{y}_i\mid \bold{x}_i \in C_L\} \,,\, \{\bold{y}_i \mid \bold{x}_i \in C_R\}\right),
\end{equation}
which measures the difference of the empirical distributions of the data $\bold{Y}$ in the two resulting child nodes $C_L$ and $C_R$. Therefore, in each step we select the candidate predictor $X_j$ which seems to affect the distribution of $\bold{Y}$ the most, as measured by the metric $\mathcal{D}(\cdot,\cdot)$. Intuitively, in this way we ensure that the distribution of the data points in every leaf of the resulting tree is as homogeneous as possible, which helps mitigate the bias caused by pooling the heterogeneous data together. A related idea can be found in GRF \citep{athey2019generalized}, where one attempts to split the data so that the resulting estimates $\hat{\tau}_L$ and $\hat{\tau}_R$, obtained respectively from data points in $C_L$ and $C_R$, differ the most:
\begin{equation}\label{eq: GRF split}
\frac{n_Ln_R}{n_P^2} \left(\hat{\tau}_L - \hat{\tau}_R\right)^2,
\end{equation}
where we write $n_P = |\{i \mid \bold{x}_i \in P\}|$ and $n_L, n_R$ are defined analogously.

One could construct the forest using any metric $\mathcal{D}(\cdot,\cdot)$ for empirical distributions. However, in order to have a good accuracy of the overall method, the corresponding two-sample test using $\mathcal{D}(\cdot,\cdot)$ needs to have a good power for detecting any kind of change in distribution, which is a difficult task in general, especially for multivariate data \citep{bai1996effect, szekely2004testing}. Another very important aspect of the choice of distributional metric $\mathcal{D}(\cdot, \cdot)$ is the computational efficiency; one needs to be able to sequentially compute the values of $\mathcal{D}\left(\{\bold{y}_i\mid \bold{x}_i \in C_L\} \,,\, \{\bold{y}_i \mid \bold{x}_i \in C_R\}\right)$ for every possible split very fast for the overall algorithm to be computationally feasible, even for moderately large data sets. Below, we propose a splitting criterion based on the MMD two-sample test statistic \citep{gretton2007kernel} which has both good statistical and computational properties.

In contrast to other forest-based methods, we do not use any information about our estimation target $\tau$ in order to find the best split of the data, which comes with a certain trade-off. On one hand, it is sensible that tailoring the splitting criterion to the target should improve the estimation accuracy; for example, some predictors might affect the conditional distribution of $\bold{Y}$, but not necessarily the estimation target $\tau$ and splitting on such predictors unnecessarily reduces the number of training points used for estimating $\tau$. On the other hand, our approach has multiple benefits: it is easier to use as it does not require any user input such as the likelihood or score functions and it can also be used for very complicated targets for which one could not easily adapt the splitting criterion. Furthermore, only one DRF fit is necessary for producing estimates of many different targets, which has both computational advantage and the practical advantage that the resulting estimates are mutually compatible (see e.g.\ Figure \ref{fig: functionals}). 

Interestingly, sometimes it could even be beneficial to split based on a predictor which does not affect the target of estimation, but which affects the conditional distribution. This is illustrated by the following toy example. Suppose that for a bivariate response $(Y_1, Y_2)$ we are interested in estimating the slope of the linear regression of $Y_2$ on $Y_1$ conditionally on $p=30$ predictors $\bold{X}$, i.e. our target is $\tau(\bold{x}) = \Cov(Y_1, Y_2 \mid \bold{X}\myeq\bold{x})/\Var(Y_1 \mid \bold{X}\myeq\bold{x})$. This is one of the main use cases for GRF and its variant which estimates this target is called Causal Forest \citep{wager2018estimation, athey2019generalized}. Let us assume that the data has the following distribution:
\begin{equation} \label{eq: toy_example}
\P\left(\begin{bmatrix}Y_1 \\ Y_2\end{bmatrix} \,\,\middle|\,\, \bold{X}\myeq\bold{x}\right) \sim N\left(\begin{bmatrix} x_1 \\ x_1 \end{bmatrix}, \begin{bmatrix}\sigma^2 &0\\ 0 &\sigma^2\end{bmatrix}\right) \qquad \bold{X} \sim  N(\bold{0}, I_p),
\end{equation}
i.e. $X_1$ affects only the mean of the responses, while the other $p-1$ predictors have no effect. In Figure \ref{fig: toy_example} we illustrate the distribution of the data when $n=300, p=30, \sigma=0.2$, together with the DRF and GRF splitting criteria. The true value of the target is $\tau(\bold{x})=0$, but when $\sigma$ is not too big, the slope estimates $\hat{\tau}$ on pooled data will be closer to $1$. Therefore, the difference of $\hat{\tau}_L$ and $\hat{\tau}_R$ between the induced slope estimates for a candidate split, which is used for splitting criterion \eqref{eq: GRF split} of GRF, might not be large enough for us to decide to split on $X_1$, or the resulting split might be too unbalanced. This results in worse forest estimates for this toy example, see Figure \ref{fig: toy_example}.

\begin{figure}[h]
    \centering
    \hspace*{0cm}
    \includegraphics[width=0.95\textwidth]{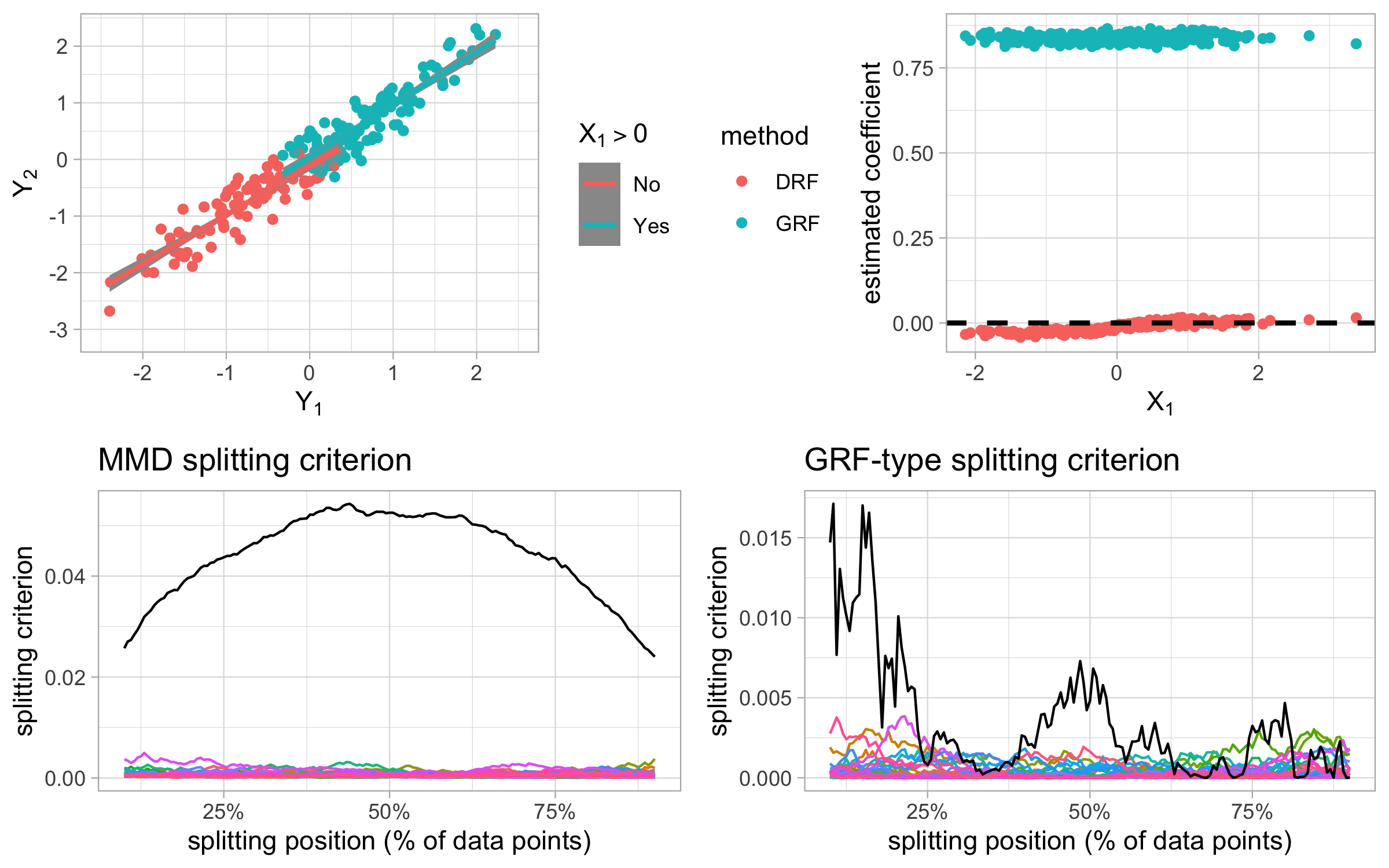}
    \caption{Top left: Illustration of data distribution for the toy example \eqref{eq: toy_example} when $n=300,\, p=30$. Bottom: The corresponding MMD \eqref{eq: MMD splitcrit} (left) and GRF \eqref{eq: GRF split} splitting criteria (right) at the root node. The curves of different colors correspond to different predictors, with $X_1$ denoted in black. Top right: Comparison of the estimates of DRF and Causal Forest \citep{athey2019generalized} which respectively use those splitting criteria. Test points were randomly generated from the same distribution as the training data. Black dashed line indicates the correct value of the target quantity.}
    \label{fig: toy_example}
\end{figure}

\subsection{Weighting Function}
\label{sec: weighting function}
Having constructed our forest, just as the standard Random Forest \citep{breiman2001random} can be viewed as the weighted nearest neighbor method \citep{lin2002random}, we can use the induced weighting function to estimate the conditional distribution at any given test point $\bold{x}$ and thus any other quantity of interest $\tau(\bold{x})$. This approach is commonly used in various forest-based methods for obtaining predictions, see e.g., \citet{hothorn2017transformation, pospisil2018rfcde, athey2019generalized}.

Suppose that we have built $N$ trees $\mathcal{T}_1, \ldots, \mathcal{T}_N$. Let $\mathcal{L}_k(\bold{x})$ be the set of the training data points which end up in the same leaf as $\bold{x}$ in the tree $\mathcal{T}_k$. The weighting function $w_{\bold{x}}(\bold{x}_i)$ is defined as the average of the corresponding weighting functions per tree \citep{lin2006random}:
\begin{equation} \label{eq: weighting function}
w_{\bold{x}}(\bold{x}_i) = \frac{1}{N} \sum_{k=1}^N \frac{\1\left(\bold{x}_i \in \mathcal{L}_k(\bold{x})\right)}{|\mathcal{L}_k(\bold{x})|}.
\end{equation}
The weights are positive and add up to $1$: $\sum_{i=1}^n w_{\bold{x}}(\bold{x}_i) = 1$. In the case of equally sized leaf nodes, the assigned weight to a training point $\bold{x}_i$ is proportional to the number of trees where the test point $\bold{x}$ and $\bold{x}_i$ end up in the same leaf node. This shows that forest-based methods can in general be viewed as adaptive nearest neighbor methods. The sets $\mathcal{L}_k(\bold{x})$ of DRF will contain data points $(\bold{x}_i, \bold{y}_i)$ such that $\P(\bold{Y}\mid \bold{X}=\bold{x}_i)$ is close to $\P(\bold{Y}\mid \bold{\mathbf{X}=\mathbf{x}})$, thus removing bias due to heterogeneity of $\bold{Y}$ caused by $\bold{X}$. On the other hand, since the trees are constructed randomly and are thus fairly independent \citep{breiman2001random}, the leaf sets $\mathcal{L}_k(\bold{x})$ will be different enough so that the induced weights $w_{\bold{x}}(\bold{x}_i)$ are not concentrated on a small set of data points, which would lead to high estimation variance. Such good bias-variance tradeoff properties of forest-based methods are also implied by their asymptotic properties \citep{biau2012analysis, wager2014asymptotic}, even though this is a still active area of research and not much can be shown rigorously.

One can estimate the conditional distribution  $\P(\bold{Y}\mid\bold{X}=\bold{x})$ from the weighting function by using the corresponding empirical distribution:
\begin{equation} \label{eq: distributional estimate}
\hat{\P}(\bold{Y} \mid \bold{X}=\bold{x}) = \sum_{i=1}^n w_{\bold{x}}(\bold{x}_i)\cdot\delta_{\bold{y}_i},
\end{equation}
where $\delta_{\bold{y}_i}$ is the point mass at $\bold{y}_i$.

\paragraph{Two-step approach using weights.}
The weighting function $w_{\bold{x}}(\bold{x}_i)$ can directly be used for any target $\tau(\bold{x})$ in a second step and not just for estimating the conditional distribution. For example, the estimated conditional joint CDF is given by 
\begin{equation}\label{eq: CDF estimate}
\hat{F}_{\bold{Y}\mid\bold{X}\myeq\bold{x}}(\bold{t}) = \hat{\P}(Y_1 \leq t_1, \ldots, Y_d \leq t_d \mid \bold{X}\myeq\bold{x}) = \sum_{i=1}^n w_\bold{x}(\bold{x}_i) \1(\cap_{j=1}^d \{(\bold{y}_i)_j \leq t_j\}).
\end{equation}

It is important to point out that using the induced weighting function for locally weighted estimation is different than the approach of averaging the noisy estimates obtained per tree \citep{wager2018estimation}, originally used in standard Random Forests \citep{breiman2001random}. Even though the two approaches are equivalent for conditional mean estimation, the former approach is often much more efficient for more complex targets \citep{athey2019generalized}, since the number of data points in a single leaf is very small, leading to large variance of the estimates.

For the univariate response, the idea of using the induced weights for estimating targets different than the original target of conditional mean considered in \citet{breiman2001random} dates back to Quantile Regression Forests (QRF) \citep{meinshausen2006quantile}, where a lot of emphasis is put on the quantile estimation, as the number of interesting targets is quite limited in the univariate setting. In the multivariate case, on the other hand, many interesting quantities such as, for example, conditional quantiles, conditional correlations or various conditional probability statements can easily be directly estimated from the weights.

By using the weights as an input for some other method, we can accomplish some more complicated objectives, such as conditional independence testing, causal effect estimation, semiparametric learning, time series prediction or tail-index estimation in extreme value analysis. As an example, suppose that our data $\bold{Y}$ come from a certain parametric model, where the parameter $\theta$ is not constant, but depends on $\bold{X}$ instead, i.e. $\bold{Y}\mid\bold{\mathbf{X}=\mathbf{x}} \sim f(\theta(\bold{x}), \cdot)$, see also \citet{zeileis2008model}. One can then estimate the parameter $\theta(\bold{x})$ by using weighted maximum likelihood estimation: $$\hat{\theta}(\bold{x}) = \argmax_{\theta \in \Theta} \sum_{i=1}^n w_\bold{x}(\bold{x}_i) \log f(\theta, \bold{y}_i).$$
Another example is heterogeneous regression, where we are interested in the regression fit of an outcome $Y \in \R$ on certain predicting variables $\bold{W}\in \R^s$ conditionally on some event $\{\bold{X} = \bold{x}\}$. This can be achieved by weighted regression of $Y$ on $\bold{W}$, where the weights $w_\bold{x}(\bold{x}_i)$ assigned to each data point $(\bold{w}_i, y_i)$ are obtained from DRF with the multivariate response $(Y, \bold{W}) \in \R^{s+1}$ and predictors $\bold{X} \in \R^p$, for an illustration see Section \ref{sec: causality}.

\begin{figure}
    \centering
    \includegraphics[width=0.95\textwidth]{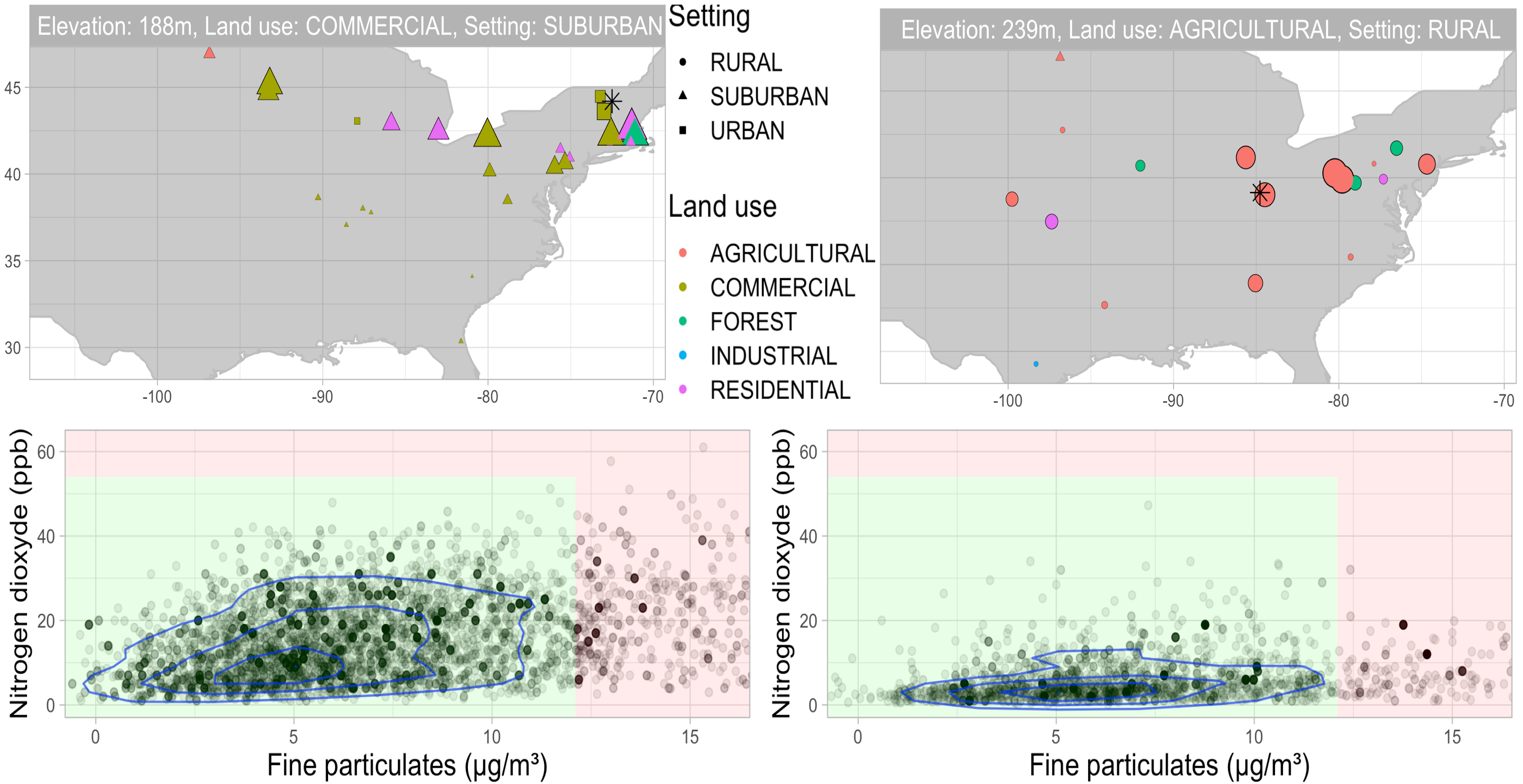}
    \caption{Top: the characteristics of the important training sites, for a fixed test site whose position is indicated by a black star and whose characteristics are indicated in the title. The total weight assigned corresponds to the symbol size. Bottom: estimated joint conditional distribution of two pollutants NO$_2$ and PM$2.5$, where the weights correspond to the transparency of the data points. Green area corresponds to 'Good' air quality category ($\text{AQI} \leq 50$).}
    \label{air_data}
\end{figure}

The weighting function of DRF is illustrated on the air quality data in Figure \ref{air_data}. Five years ($2015-2019$) of air pollution measurements were obtained from the US Environmental Protection Agency (EPA) website. Six main air pollutants (nitrogen dioxide (NO$_2$), carbon monoxide (CO), sulphur dioxide (SO$_2$), ozone (O$_3$) and coarse and fine particulate matter (PM$10$ and PM$2.5$)) that form the air quality index (AQI) were measured at many different measuring sites in the US for which we know the longitude, latitude, elevation, location setting (rural, urban, suburban) and how the land is used within a $1/4$ mile radius. Suppose we would want to know the distribution of the pollutant measurements at some new, unobserved, measurement site. We train DRF with the measurements (intraday maximum) of the two pollutants PM$2.5$ and NO$_2$ as the responses, and the site longitude, latitude, elevation, land use and location settings as the predictors and choose two decommissioned measurement sites as test points. For each test point we obtain the weights to all training measurements. We further combine the weights for all measurements corresponding to the same site. The top row illustrates for a given test site, whose characteristics are indicated in the plot title, how much weight in total is assigned to the measurements from a specific training site. We see that the important sites share many characteristics with the test site and that DRF determines the relevance of each characteristic in a data-adaptive way. The bottom row shows the corresponding estimates of the joint conditional distribution of the pollutants (we choose $2$ of them for visualization purposes), where the transparency of each training point reflects the assigned weight. One can clearly see how the estimated pollution levels are larger for the suburban site than for the rural site. The forest weights can be used, for example, for estimating the joint density (whose contours can be seen in the plot) or for estimating the probability that the AQI is below a certain value by summing the weights in the corresponding region of space.

\subsection{Distributional Metric}

In order to determine the best split of a parent node $P$, i.e. such that the distributions of the responses $\bold{Y}$ in the resulting child nodes $C_L$ and $C_R$ differ the most, one needs a good distributional metric $\mathcal{D}(\cdot,\cdot)$ (see Equation \eqref{eq: splitcrit}) which can detect change in distribution of the response $\bold{Y}$ when additionally conditioning on an event $\{X_j > l\}$. Testing equality of distributions from the corresponding samples is an old problem in statistics, called two-sample testing problem. For univariate data, many good tests exist such as Wilcoxon rank test \citep{wilcoxon1946individual}, Welch's t-test \citep{welch1947generalization}, Wasserstein two-sample testing \citep{ramdas2017wasserstein}, Kolmogorov-Smirnov test \citep{massey1951kolmogorov} and many others, but obtaining an efficient test for multivariate distributions has proven to be quite challenging due to the curse of dimensionality \citep{friedman1979multivariate, baringhaus2004new}. 

Additional requirement for the choice of distributional metric $\mathcal{D}(\cdot,\cdot)$ used for data splitting is that it needs to be computationally very efficient as splitting is used extensively in the algorithm. If we construct $N$ trees from $n$ data points and in each node we consider $\text{mtry}$ candidate variables for splitting, the complexity of the standard Random Forest algorithm \citep{breiman2001random} in the univariate case is $\O(N \times \text{mtry} \times n \log n)$ provided our splits are balanced. It uses the CART splitting criterion, given by: 
\begin{equation} \label{eq: CART}
\frac{1}{n_P}\left(\sum_{\bold{x}_i \in C_L} (y_i - \overline{y}_L)^2 + \sum_{\bold{x}_i \in C_R} (y_i - \overline{y}_R)^2 \right),
\end{equation}
where $\overline{y}_L = \tfrac{1}{n_L}\sum_{\bold{x}_i \in C_L} y_i$ and $\overline{y}_R$ is defined analogously. This criterion has an advantage that not only it can be computed in $\O(n_P)$ complexity, but this can be done for all possible splits $\{X_j \leq l\}$ as cutoff level $l$ varies, since updating the splitting criterion when moving a single training data point from one child node to the other requires only $\O(1)$ computational steps (most easily seen by rewriting the CART criterion as in \eqref{eq: CART equivalent}).

If the time complexity of evaluating the DRF splitting criterion \eqref{eq: splitcrit} for a single splitting candidate $X_j$ and all cutoffs $l$ of interest (usually taken to range over all possible values) is at least $n^c$ for some $c>1$, say $\O(f(n))$ for some function $f:\R \to \R$, then by solving the recursive relation we obtain that the overall complexity of the method is given by $\O(N \times \text{mtry} \times f(n))$ \citep{akra1998solution}, which can be unfeasible even for moderately large $n$ if $f$ grows too fast.

The problem of sequential two-sample testing is also central to the field of change-point detection \citep{wolfe1984nonparametric, brodsky2013nonparametric}, with the slight difference that in the change-point problems the distribution is assumed to change abruptly at certain points in time, whereas for our forest construction we only are interested in finding the best split of the form $\{X_j \leq l\}$ and the conditional distribution $\P(\bold{Y}\mid \{\bold{X}\in P\} \cap\{X_j \leq l\})$ usually changes gradually with $l$. The testing power and the computational feasibility of the method play a big role in change-point detection as well. However, the state-of-the-art change-point detection algorithms \citep{li2015scan, matteson2014nonparametric} are often too slow for our purpose as sequential testing is done $\O(N \times \text{mtry} \times n)$ times for forest construction, much more frequently than in change-point problems.

\subsubsection{MMD splitting criterion}
\label{sec: MMD splitcrit}

Even though DRF could in theory be constructed with any distributional metric $\mathcal{D}(\cdot, \cdot)$, as a default choice we propose splitting criterion based on the Maximum Mean Discrepancy (MMD) statistic \citep{gretton2007kernel}. Let $(\mathcal{H}, \langle\cdot,\cdot\rangle_\mathcal{H})$ be the RKHS of real-valued functions on $\R^d$ induced by some positive-definite kernel $k$, and let $\varphi:\R^d \to \mathcal{H}$ be the corresponding feature map satisfying that $k(\bold{u}, \bold{v}) = \langle\varphi(\bold{u}), \varphi(\bold{v})\rangle_\mathcal{H}$.

The MMD statistic $\mathcal{D}_{\text{MMD}(k)}\left(U, V\right)$ for kernel $k$ and two samples $U = \{\bold{u}_1, \ldots, \bold{u}_{|U|}\}$ and $V = \{\bold{v}_1, \ldots, \bold{v}_{|V|}\}$ is given by:
\begin{equation} \label{eq: MMD}
\mathcal{D}_{\text{MMD}(k)}\left(U, V\right) = \frac{1}{|U|^2}\sum_{i,j=1}^{|U|} k(\bold{u}_i, \bold{u}_j) + \frac{1}{|V|^2}\sum_{i,j=1}^{|V|} k(\bold{v}_i, \bold{v}_j)
 - \frac{2}{|U||V|}\sum_{i=1}^{|U|}\sum_{j=1}^{|V|} k(\bold{u}_i, \bold{v}_j).
\end{equation} 
MMD compares the similarities, described by the kernel $k$, within each sample with the similarities across samples and is commonly used in practice for two-sample testing. It is based on the idea that one can assign to each distribution $\mathcal{P}$ its embedding $\mu(\mathcal{P})$ into the RKHS $\mathcal{H}$, which is the unique element of $\mathcal{H}$ given by 
\begin{equation} \label{eq: RKHS embedding}
\mu(\mathcal{P}) = \E_{\bold{Y} \sim \mathcal{P}}[\varphi(\bold{Y})].
\end{equation}
The MMD two-sample statistic \eqref{eq: MMD} can then equivalently be written as the squared distance between the embeddings of the empirical distributions with respect to the RKHS norm $\norm{\cdot}_\mathcal{H}$:
\begin{equation} \label{eq: MMD embedding}
\mathcal{D}_{\text{MMD}(k)}\left(U, V\right) = \norm*{\mu\left(\frac{1}{|U|}\sum_{i=1}^{|U|} \delta_{\bold{u}_i}\right) - \mu\left(\frac{1}{|V|}\sum_{i=1}^{|V|} \delta_{\bold{v}_i}\right)}_\mathcal{H}^2,
\end{equation}
recalling that $\delta_{\bold{y}}$ is the point mass at $\bold{y}$. 

As the sample sizes $|U|$ and $|V|$ grow, the MMD statistic \eqref{eq: MMD embedding} converges to its population version, which is the squared RKHS distance between the corresponding embeddings of the data-generating distributions of $U$ and $V$. Since the embedding map $\mu$ is injective for a characteristic kernel $k$, we see that MMD is able to detect any difference in the distribution. 
Even though the power of the MMD two sample test also deteriorates as the data dimensionality grows, since the testing problem becomes intrinsically harder \citep{reddi2014decreasing}, it still has good empirical power compared to other multivariate two-sample tests for a wide range of $k$ \citep{gretton2012kernel}.

\paragraph{Fast random splitting criterion approximation.} The $\O((|U|+|V|)^2)$ complexity for computing $\mathcal{D}_{\text{MMD}(k)}(U, V)$ from \eqref{eq: MMD} is nevertheless too large for many applications. For that reason, several fast approximations of MMD have been suggested in the literature \citep{gretton2012kernel, NIPS2012_4727, zaremba2013b, NIPS2015_5685, NIPS2016_6148}. As already mentioned, the complexity of the distributional metric $\mathcal{D}(\cdot,\cdot)$ used for DRF is crucial for the overall method to be computationally efficient, since the splitting step is used extensively in the forest construction. We therefore propose splitting based on an MMD statistic computed with an approximate kernel $\tilde{k}$, which is also a fast random approximation of the original MMD statistic \citep{zhao2015fastmmd}. 

Bochner's theorem (see e.g.\ \citet[Theorem 6.6]{wendland_2004}) gives us that any bounded shift-invariant kernel can be written as 
\begin{equation} \label{eq: bochner}
k(\bold{u}, \bold{v}) = \int_{\R^d} e^{i\boldsymbol{\omega}^T(\bold{u}-\bold{v})}d\nu(\boldsymbol{\omega}),
\end{equation}
i.e.\ as a Fourier transform of some measure $\nu$. Therefore, by randomly sampling the frequency vectors $\boldsymbol{\omega}_1, \ldots, \boldsymbol{\omega}_B$ from normalized $\nu$, we can approximate our kernel $k$ by another kernel $\tilde{k}$ (up to a scaling factor) as follows: 
$$k(\bold{u}, \bold{v}) = \int_{\R^d} e^{i\boldsymbol{\omega}^T(\bold{u}-\bold{v})}d\nu(\boldsymbol{\omega}) \approx \frac{1}{B} \sum_{b=1}^B e^{i\boldsymbol{\omega}_b^T(\bold{u}-\bold{v})} = \tilde{k}(\bold{u}, \bold{v}),$$
where we define $\tilde{k}(\bold{u}, \bold{v}) = \langle \bold{\widetilde{\varphi}}(\bold{u}), \bold{\widetilde{\varphi}}(\bold{v})\rangle_{\mathbb{C}^B}$ as the kernel function with the feature map given by $$\bold{\widetilde{\varphi}}(\bold{u}) = \frac{1}{\sqrt{B}}\left(\tilde{\varphi}_{\boldsymbol{\omega}_1}(\bold{u}), 
\ldots, \tilde{\varphi}_{\boldsymbol{\omega}_B}(\bold{u})\right)^T = \frac{1}{\sqrt{B}}\left(e^{i\boldsymbol{\omega}_1^T\bold{u}}, 
\ldots, e^{i\boldsymbol{\omega}_B^T\bold{u}}\right)^T,$$
which is a random vector consisting of the Fourier features $\widetilde{\varphi}_{\boldsymbol{\omega}}(\bold{u}) = e^{i\boldsymbol{\omega}^T\bold{u}} \in \mathbb{C}$ \citep{rahimi2008random}.
Such kernel approximations are frequently used in practice for computational efficiency \citep{rahimi2009weighted, le2013fastfood}. As a default choice of $k$ we take the Gaussian kernel with bandwidth $\sigma$, since in this case we have a convenient expression for the measure $\nu$ and we sample $\boldsymbol{\omega}_1, \ldots, \boldsymbol{\omega}_B \sim N_{d}(\bold{0}, \sigma^{-2}I_{d})$. The bandwidth $\sigma$ is chosen as the median pairwise distance between all training responses $\{\bold{y}_i\}_{i=1}^n$, commonly referred to as the 'median heuristic' \citep{gretton2012optimal}. 

From the representation of MMD via the distribution embeddings \eqref{eq: MMD embedding}, we can obtain that MMD two-sample test statistic $\mathcal{D}_{\text{MMD}(\tilde{k})}$ using the approximate kernel $\tilde{k}$ is given by
\begin{equation*}
\mathcal{D}_{\text{MMD}(\tilde{k})}\left(\{\bold{u}_i\}_{i=1}^{|U|}, \{\bold{v}_i\}_{i=1}^{|V|}\right) = \frac{1}{B}\sum_{b=1}^B\left\lvert \frac{1}{|U|}\sum_{i=1}^{|U|} \tilde{\varphi}_{\boldsymbol{\omega}_b}(\bold{u}_i) -  \frac{1}{|V|}\sum_{i=1}^{|V|} \tilde{\varphi}_{\boldsymbol{\omega}_b}(\bold{v}_i) \right\rvert^2.
\end{equation*}
Interestingly, $\mathcal{D}_{\text{MMD}(\tilde{k})}$ is not only an MMD statistic on its own, but can also be viewed as a random approximation of the original MMD statistic $\mathcal{D}_{\text{MMD}(k)}$ \eqref{eq: MMD} using kernel $k$; by using the kernel representation \eqref{eq: bochner}, it can be written as
\begin{equation*}
\mathcal{D}_{\text{MMD}(k)}\left(\{\bold{u}_i\}_{i=1}^{|U|}, \{\bold{v}_i\}_{i=1}^{|V|}\right) = \int_{\R^d}\left\lvert \frac{1}{|U|}\sum_{i=1}^{|U|} \tilde{\varphi}_{\boldsymbol{\omega}}(\bold{u}_i) -  \frac{1}{|V|}\sum_{i=1}^{|V|} \tilde{\varphi}_{\boldsymbol{\omega}}(\bold{v}_i) \right\rvert^2 d\nu(\boldsymbol{\omega}).
\end{equation*}

Finally, our DRF splitting criterion $\mathcal{D}(\cdot, \cdot)$ \eqref{eq: splitcrit} is then taken to be the (scaled) MMD statistic $\tfrac{n_Ln_R}{n_P^2}\mathcal{D}_{\text{MMD}(\tilde{k})}\left(\{\bold{y}_i\mid \bold{x}_i \in C_L\} \,,\, \{\bold{y}_i \mid \bold{x}_i \in C_R\}\right)$ with the approximate random kernel $\tilde{k}$ used instead of $k$, which can thus be conveniently written as:
\begin{align} \label{eq: MMD splitcrit}
\frac{1}{B}\sum_{b=1}^B \frac{n_Ln_R}{n_P^2} \left\lvert \frac{1}{n_L}\sum_{\bold{x}_i \in C_L} \tilde{\varphi}_{\boldsymbol{\omega}_b}(\bold{y}_i) -  \frac{1}{n_R}\sum_{\bold{x}_i \in C_R} \tilde{\varphi}_{\boldsymbol{\omega}_b}(\bold{y}_i) \right\rvert^2,
\end{align}
where we recall that $n_P = |\{i \mid \bold{x}_i \in P\}|$ and $n_L, n_R$ are defined analogously.
The additional scaling factor $\tfrac{n_Ln_R}{n_P^2}$ in \eqref{eq: MMD splitcrit} occurs naturally and compensates the increased variance of the test statistic for unbalanced splits; it also appears in the GRF \eqref{eq: GRF split} and CART (see representation \eqref{eq: CART equivalent}) splitting criteria.

The main advantage of the splitting criterion based on $\mathcal{D}_{\text{MMD}(\tilde{k})}$ is that by using the representation \eqref{eq: splitcrit} it can be easily computed for every possible splitting level $l$ in $\O(Bn_P)$ complexity, whereas the MMD statistic $\mathcal{D}_{\text{MMD}(k)}$ using kernel $k$ would require $\O(n_P^2)$ computational steps, which makes the overall complexity of the algorithm $\O\left(B \times N \times \text{mtry} \times n \log n\right)$ instead of much slower $\O\left(N \times \text{mtry} \times n^2\right)$.

We do not use the same approximate random kernel $\tilde{k}$ for different splits; for every parent node $P$ we resample the frequency vectors $\{\omega_b\}_{b=1}^B$ defining the corresponding feature map $\tilde{\varphi}$. Using different $\tilde{k}$ at each node might help to better detect different distributional changes. Furthermore, having different random kernels for each node agrees well with the randomness of the Random Forests and helps making the trees more independent. Since the MMD statistic $\mathcal{D}_{\text{MMD}(\tilde{k})}$ used for our splitting criterion is not only an approximation of $\mathcal{D}_{\text{MMD}(k)}$, but is itself an MMD statistic, it inherits good power for detecting any difference in distribution of $\bold{Y}$ in the child nodes for moderately large data dimensionality $d$, even when $B$ is reasonably small. One could even consider changing the number of random Fourier features $B$ at different levels of the tree, as $n_P$ varies, but for simplicity we take it to be fixed.

\paragraph{Relationship to CART.}
There is some similarity of our MMD-based splitting criterion \eqref{eq: MMD splitcrit} with the standard variance reduction CART splitting criterion \eqref{eq: CART} when $d=1$, which can be rewritten as:
\begin{equation} \label{eq: CART equivalent}
\frac{n_Ln_R}{n_P^2}\left(\frac{1}{n_L}\sum_{\bold{x}_i \in C_L} y_i - \frac{1}{n_R}\sum_{\bold{x}_i \in C_R} y_i \right)^2.
\end{equation}
The derivation can be found in Appendix \ref{appendix: proofs}.
From this representation, we see that the CART splitting criterion \eqref{eq: CART} is also equivalent to the GRF splitting criterion \eqref{eq: GRF split} when our target is the univariate conditional mean $\tau(\bold{x}) = \E[Y \mid \bold{X}\myeq\bold{x}]$ which is estimated for $C_L$ and $C_R$ by the sample means $\hat{\tau}_L = \overline{y}_L$ and $\hat{\tau}_R = \overline{y}_R$.
Therefore, as it compares the means of the univariate response $Y$ in the child nodes, the CART criterion can only detect changes in the response mean well, which is sufficient for prediction of $Y$ from $\bold{X}$, but might not be suitable for more complex targets. Similarly, for multivariate applications, aggregating the marginal CART criteria \citep{kocev2007ensembles, segal2011multivariate} across different components $Y_i$ of the response can only detect changes in the means of their marginal distributions. However, it is possible in the multivariate case that the pairwise correlations or the variances of the responses change, while the marginal means stay (almost) constant. For an illustration on simulated data, see Figure \ref{fig: gaussian_copula}. Additionally, aggregating the splitting criteria over $d$ components of the response $\bold{Y}$ can reduce the signal size if only the distribution of a few components change. Our MMD-based splitting criterion \eqref{eq: MMD splitcrit} is able to avoid such difficulties as it implicitly inspects all aspects of the multivariate response distribution.

If one takes a trivial kernel $k_{\text{id}}(y_i, y_j) = y_i y_j$ with the identity feature map $\varphi_{\text{id}}(y) = y$, the distributional embedding \eqref{eq: RKHS embedding} is given by $\mu(\mathcal{P}) = \E_{Y\sim\mathcal{P}}[ Y]$ and thus the corresponding splitting criterion based on $\mathcal{D}_{\text{MMD}(k_{\text{id}})}$ \eqref{eq: MMD embedding} is exactly equal to the CART splitting criterion \eqref{eq: CART}, which can be seen from its equivalent representation \eqref{eq: CART equivalent}.
Interestingly, Theorem \ref{thm: MMD_CART_equivalence} in Section \ref{sec: theory} shows that the MMD splitting criterion with general kernel $k$ can also be viewed as the abstract version of the CART criterion in the RKHS $\mathcal{H}$ corresponding to $k$ \citep{fan2010cw}, with the response variable being the feature map $\varphi(\bold{Y}) \in \mathcal{H}$. Therefore, DRF with the MMD splitting criterion can also be viewed as a forest-based method for estimation of the conditional embedding, which further justifies the proposed method. 
In Section \ref{sec: theory} below, we use this relationship to derive interesting theoretical properties of DRF with the MMD splitting criterion.


\section{Theoretical Results}
\label{sec: theory}

In this section we first use the properties of the kernel mean embedding in order to relate DRF with the MMD splitting criterion to an abstract version of the standard Random Forest with the CART splitting criterion \citep{breiman2001random}, where the response is taking values in the corresponding RKHS. This representation reveals that DRF with the MMD splitting criterion can be viewed as a Random Forest estimator of the conditional mean embedding (CME) \citep{OurapproachtoCME}, similarly as the standard Random Forest estimates the conditional mean. This relationship is further exploited to adapt the existing theoretical results from the Random Forest literature to show that our estimate \eqref{eq: distributional estimate} of the conditional distribution of the response is consistent with respect to the MMD metric for probability measures and with a good rate. Finally, we show that this implies consistency of the induced DRF estimates for a range interesting targets $\tau(\bold{x})$, such as conditional CDFs or quantiles. The proofs of all results can be found in the Appendix \ref{appendix: proofs}.

\subsection{Casting DRF as a Random Forest in an RKHS}
Recalling the notation from above, let $\left(\mathcal{H}, \langle\cdot,\cdot\rangle_{\mathcal{H}}\right)$ be the Reproducing kernel Hilbert space induced by the positive definite kernel $k:\R^d\times\R^d \to \R$ and let $\varphi:\R^d \to \mathcal{H}$ be its corresponding feature map. The kernel embedding function $\mu:\mathcal{M}_{b}(\R^d) \to \mathcal{H}$
maps any bounded signed Borel measure $\mathcal{P}$ on $\R^d$ to an element $\mu(\mathcal{P}) \in \mathcal{H}$ defined by
\begin{equation}
\label{eq: kernel embedding}
\mu(\mathcal{P}) = \int_{\R^d}\varphi(\bold{y})\,d\mathcal{P}(\bold{y}),
\end{equation}
see also \eqref{eq: RKHS embedding}. Boundedness of $k$ ensures that $\mu$ is indeed defined on all of $\mathcal{M}_{b}(\R^d)$, while continuity of $k$ ensures that $\mathcal{H}$ is separable \citep{hilbertspacebook}. 


By considering the kernel embedding $\mu(\cdot)$ and using its linearity, the embedding of the distributional estimate $\mu(\hat{\P}(\bold{Y}\mid\bold{X}\myeq\bold{x}))$ of DRF \eqref{eq: distributional estimate} can be written as the average of the embeddings of the empirical distributions of $\bold{Y}$ in the leaves containing $\bold{x}$ over all trees:
\begin{equation} \label{eq: embedding estimate}
\mu(\hat{\P}(\bold{Y}\mid\bold{X}\myeq\bold{x})) = \frac{1}{N}\sum_{k=1}^N \mu\left( \frac{1}{|\mathcal{L}_k(\bold{x})|}\sum_{\bold{x}_i \in \mathcal{L}_k(\bold{x})} \delta_{\bold{y}_i}\right) = \frac{1}{N}\sum_{k=1}^N \frac{1}{|\mathcal{L}_k(\bold{x})|}\sum_{\bold{x}_i \in \mathcal{L}_k(\bold{x})} \mu(\delta_{\bold{y}_i}).
\end{equation}
This is analogous to the prediction of the response for the standard univariate Random Forest, but where we average the embeddings $\mu(\delta_{\bold{y}_i}) = \varphi(\bold{y}_i) \in \mathcal{H}$ instead of the response values $y_i$ themselves.

Furthermore, one can relate the MMD splitting criterion to the original CART criterion \eqref{eq: CART}, which measures the mean squared prediction error for splitting a certain parent node $P$ into children $C_L$ and $C_R$. On one hand, from Equation \eqref{eq: CART equivalent} we see that the CART criterion also measures the squared distance between the response averages $\tfrac{1}{n_L}\sum_{\bold{x}_i \in C_L} y_i$ and $\tfrac{1}{n_R}\sum_{\bold{x}_i \in C_R} y_i$ in the child nodes, but on the other hand, Equation \eqref{eq: MMD embedding} shows that the MMD splitting criterion measures the RKHS distance between the embeddings of the empirical response distributions in $C_L$ and $C_R$. This is summarized in the following theorem, which not only shows that the MMD splitting criterion can be viewed as the abstract CART criterion in the RKHS $\mathcal{H}$ \citep{fan2010cw}, but also that DRF with the MMD splitting criterion can be viewed asymptotically as a greedy minimization of the average squared MMD distance between our estimate $\hat{\P}(\bold{Y}\mid\bold{X}\myeq\bold{x})$ and the truth $\P(\bold{Y}\mid\bold{X}\myeq\bold{x})$:

\begin{restatable}{theorem}{MMDCART}\label{thm: MMD_CART_equivalence}
For any split of a parent node $P$ into child nodes $C_L$ and $C_R$, let $\hat{\P}_\text{split}(\bold{x}) = \sum_{j\in\{L, R\}}\1(\bold{x}\in C_j) \tfrac{1}{n_j} \sum_{\bold{x}_i \in C_j} \delta_{\boldsymbol{y}_i}$ denote the resulting estimate of the distribution $\P(\bold{Y} \mid \bold{X}\myeq\bold{x})$ when $\bold{x} \in P$. Then the MMD splitting criterion can be viewed as the version of the CART criterion \eqref{eq: CART} on $\mathcal{H}$:
\begin{align*}
&\argmax_{\text{split}} \frac{n_Ln_R}{n_P^2}\mathcal{D}_{\text{MMD}(k)}\left(\{\bold{y}_i \mid \bold{x}_i \in C_L\} , \{\bold{y}_i\mid \bold{x}_i\in C_R\}\right) \\
&=\argmin_{\text{split}} \frac{1}{n_P}\sum_{\bold{x}_i \in P} \norm*{\mu(\delta_{\bold{y}_i}) - \mu(\hat{\P}_{\text{split}}(\bold{x}_i))}_\mathcal{H}^2.
\end{align*}
\normalsize
Moreover, for any node $P$ and any fixed distributional estimator $\hat{\P}(\bold{Y}\mid\bold{X}\myeq\bold{x})$, we have: 
\begin{align*}
&\frac{1}{n_P}\sum_{\bold{x}_i \in P} \norm*{\mu(\delta_{\bold{y}_i}) - \mu(\hat{\P}(\bold{Y}\mid\bold{X}\myeq\bold{x}_i))}_\mathcal{H}^2 \\ &= V_P + \E\left[\norm{\mu(\hat{\P}(\bold{Y}\mid\bold{X})) - \mu(\P(\bold{Y} \mid \bold{X}))}_\mathcal{H}^2 \mid \bold{X} \in P\right] + \O_{p}(n^{-1/2}),
\end{align*}
\normalsize
where $V_P = \E\left[\norm{\mu(\delta_\bold{Y}) - \mu(\P(\bold{Y} \mid \bold{X}))}_\mathcal{H}^2\mid \bold{X} \in P\right]$ is a deterministic term not depending on the estimates $\hat{\P}(\bold{Y}\mid\bold{X}\myeq\bold{x})$.
\end{restatable}

In conclusion, from the above results we see that by applying the kernel embedding \eqref{eq: kernel embedding}, we can shift the perspective to the RKHS $\mathcal{H}$ and view DRF as the analogue of the original Random Forest for estimation of the CME $\mu(\P(\bold{Y}\mid\bold{X}=\bold{x})) = \E[\varphi(\bold{Y})\mid\bold{X}\myeq\bold{x}]$ in an abstract Hilbert space $\mathcal{H}$. Like some traditional CME estimators \citep{CMEinDynamicalSystems, CMEinGraphicalModels, kernelmeanembeddingreview,OurapproachtoCME} it is also of the form
\begin{align}\label{formwewant}
    \mu(\hat{\P}(\bold{Y}\mid\bold{X}\myeq\bold{x})) = \sum_{i=1}^n w_{\bold{x}}(\bold{x}_i)\cdot k(\bold{y}_i, \cdot).
\end{align}

It can be shown that, since DRF produces nonnegative weights that sum to one in \eqref{formwewant}, there is one-to-one correspondence between the resulting estimate in $\mathcal{H}$ and the empirical probability distribution in $\mathcal{M}_{b}(\R^d)$. Thus DRF can be seen as a CME estimator through \eqref{formwewant}, or directly as an estimator for the conditional distribution through \eqref{eq: distributional estimate}. By contrast, other CME estimators of the form \eqref{formwewant} have weights that are unconstrained and can be negative. Finding an appropriate distribution on $\R^d$ for a given mean-embedding (sometimes referred to as ``distributional inverse image problem'', see e.g. \citet{inverseproblem1, kernelmeanembeddingreview}) is not straightforward in general. For certain tasks, such as sampling from the estimated conditional distribution $\hat{\P}(\bold{Y}\mid\bold{X}\myeq\bold{x})$ or obtaining the plug-in estimates of some target functionals $\tau(\P(\bold{Y}\mid\bold{X}\myeq\bold{x}))$, this is crucial.

\subsection{Convergence of Conditional Distribution Estimates}

As we have seen, DRF can be viewed as the abstract version of the standard Random Forest when the response takes value in an RKHS. In principle, one could thus derive properties of DRF by adapting any existing theoretical result from the literature to the RKHS case. However, a lot of care is needed for making the results rigorous in this abstract setup, as many useful properties of $\R$ need not hold for infinite-dimensional $\mathcal{H}$. This section is inspired by the results from \citet{wager2018estimation}. 

We suppose that the forest construction satisfies the following properties, which significantly facilitate the theoretical considerations of the method and ensure that our forest estimator is well behaved, as stated in \citet{wager2018estimation}:
\begin{itemize}
    \item[\textbf{(P1)}] (\textit{Data sampling}) The bootstrap sampling with replacement, usually used in forest-based methods, is replaced by a subsampling step, where for each tree we choose a random subset of size $s_n$ out of $n$ training data points. We consider $s_n$ going to infinity with $n$, with the rate specified below. 
    \item[\textbf{(P2)}] (\textit{Honesty}) The data used for constructing each tree is split into two partOn the pitfalls of Gaussian scoring for causal discoverys; the first is used for determining the splits and the second for populating the leaves and thus for estimating the response.
    \item[\textbf{(P3)}] (\textit{$\alpha$-regularity}) Each split leaves at least a fraction $0 < \alpha \leq 0.2$ of the available training sample on each side. Moreover, the trees are grown until every leaf contains between $\kappa$ and $2\kappa - 1$ observations, for some fixed tuning parameter $\kappa \in \N$. 
    \item[\textbf{(P4)}] (\textit{Symmetry}) The (randomized) output of a tree does not depend on the ordering of the training samples.
    \item[\textbf{(P5)}] (\textit{Random-split}) At every split point, the probability that the split occurs along the feature $X_j$ is bounded below by $\pi/p$, for some $\pi > 0$ and for all $j=1,\ldots,p$.
\end{itemize}
The validity of the above properties are easily ensured by the forest construction used.

From Equation \eqref{eq: embedding estimate}, the prediction of DRF for a given test point $\mathbf{x}$ can be viewed as an element of $\mathcal{H}$. If we denote the $i$-th training observation by $\mathbf{Z}_i=(\mathbf{x}_i, \mu(\delta_{\mathbf{y}_i})) \in \R^p \times \mathcal{H}$, then by \eqref{eq: embedding estimate} we estimate the embedding of the true conditional distribution $\mu(\P(\bold{Y}\mid\bold{X}\myeq\bold{x}))$ by the average of the corresponding estimates per tree: 
\begin{equation*}\label{rewriting}
    \mu(\hat{\P}(\mathbf{Y} \mid \mathbf{X} \myeq \mathbf{x})) = \frac{1}{N} \sum_{j=1}^N T(\bold{x}; \varepsilon_j, \mathcal{Z}_j),
\end{equation*}
where $\mathcal{Z}_k$ is a random subset of $\{\mathbf{Z}_i\}_{i=1}^n$ of size $s_n$ chosen for constructing the $j$-th tree $\mathcal{T}_j$ and $\varepsilon_j$ is a random variable capturing all randomness in growing $\mathcal{T}_j$, such as the choice of the splitting candidates. $T(\bold{x}; \varepsilon, \mathcal{Z})$ denotes the output of a single tree: i.e. the average of the terms $\mu(\delta_{\bold{Y}_i})$ over all data points $\bold{Z}_i$ contained in the leaf $\mathcal{L}(\bold{x})$ of the tree constructed from $\varepsilon$ and $\mathcal{Z}$.

Since one can take the number of trees $N$ to be arbitrarily large, we consider an ``idealized'' version of our estimator, as done in \citet{wager2017estimation}, which we denote as $\hat{\mu}_n(\mathbf{x})$:
\begin{equation}\label{finalestimator}
    \hat{\mu}_n(\mathbf{x}) = \binom{n}{s_n}^{-1}  \sum_{i_1 < i_2 < \ldots < i_{s_n}} \E_{\varepsilon} \,\,T(\mathbf{x}; \varepsilon; \{\bold{Z}_{i_1}, \ldots, \bold{Z}_{i_{s_n}}\}),
\end{equation}
where the sum is taken over all $\binom{n}{s_n}$ possible subsets of $\{\mathbf{Z}_{i}\}_{i=1}^n$.\sloppy We have that $\mu(\hat{\P}(\bold{Y}\mid\bold{X}\myeq\bold{x})) \to \hat{\mu}_n(\mathbf{x})$ as $N\to \infty$, while keeping the other variables constant, and thus we assume for simplicity that those two quantities are the same.

Our main result shows that, under similar assumptions as in \citet{wager2017estimation}, the embedding of our conditional distribution estimator $\hat{\mu}_n(\bold{x}) = \mu(\hat{\P}(\mathbf{Y} \mid \mathbf{X} \myeq \mathbf{x}))$ consistently estimates $\mu(\bold{x}) \coloneqq \mu(\P(\mathbf{Y} \mid \mathbf{X} \myeq \mathbf{x}))$ with respect to the RKHS norm with a certain rate:


\begin{restatable}{theorem}{consistency}\label{thm: consistency}
Suppose that our forest construction satisfies properties \textbf{(P1)}--\textbf{(P5)}. Assume additionally that $k$ is a bounded and continuous kernel and that we have a random design with $\bold{X}_1,\ldots, \bold{X}_n$ independent and identically distributed on $[0,1]^p$ with a density bounded away from $0$ and infinity. If the subsample size $s_n$ is of order $n^\beta$ for some $0  < \beta < 1$, the mapping
$$\mathbf{x} \mapsto \mu(\mathbf{x})=\E[ \mu(\delta_\bold{Y}) \mid \mathbf{X} \myeq \mathbf{x}] \in \mathcal{H},$$
is Lipschitz and $ \sup_{\mathbf{x} \in [0,1]^p} \E[ \|\mu(\delta_\bold{Y})\|_\mathcal{H}^2\mid \mathbf{X} \myeq \mathbf{x}] < \infty$, we obtain the consistency w.r.t. the RKHS norm:
\begin{equation}\label{eq: rate} 
\norm{\hat{\mu}_n(\mathbf{x})- \mu(\mathbf{x})}_\mathcal{H} = \mathcal{O}_{p}\left(n^{-\gamma} \right),
\end{equation}
for $\gamma = \frac{1}{2} \min\left( 1- \beta, \frac{\log((1-\alpha)^{-1})}{\log(\alpha^{-1})} \frac{\pi}{p} \cdot \beta \right)$.
\end{restatable}

\begin{remark}
The rate in \eqref{eq: rate} is analogous to the one from \citet{wager2018estimation}, who used it further to derive the asymptotic normality of the Random Forest estimator in $\R$. Unfortunately, this alone is not enough to establish asymptotic normality of $(\hat{\mu}_n(\mathbf{x})- \mu(\mathbf{x}))/\sigma_n$ as an element of $\mathcal{H}$. To do so, one needs to prove a functional central limit theorem with a Gaussian limiting process in the Hilbert space ${\cal H}$. This then allows to deduce asymptotic normality of smooth real-valued functionals. We will provide the detailed derivations in future work.
\end{remark}

\subsection{Convergence of the Induced Estimates}
The above result shows that DRF estimate $\hat{\P}(\bold{Y}\mid\bold{X}\myeq\bold{x})$ converges fast to the truth $\P(\bold{Y}\mid\bold{X}\myeq\bold{x})$ in the MMD distance, i.e.\ the RKHS distance between the corresponding embeddings. Even though this is interesting on its own, ultimately we want to relate this result to estimation of certain distributional targets $\tau(\bold{x}) = \tau(\P(\bold{Y}\mid\bold{X}\myeq\bold{x})).$

For any $f \in \mathcal{H}$, we have that the DRF estimate of the target $\tau(\bold{x}) = \E[f(\bold{Y})\mid\bold{X}\myeq\bold{x}]$ equals the dot product $\langle f, \hat{\mu}_n(\mathbf{x}) \rangle_\mathcal{H}$ in the RKHS:
\begin{equation*}
    \langle f , \hat{\mu}_n(\mathbf{x}) \rangle_\mathcal{H} = \left\langle f,\, \int_{\R^d} \varphi(\mathbf{y}) d \hat{\P}(\mathbf{y} \mid \mathbf{X} \myeq \mathbf{x})\right\rangle_\mathcal{H} =  \int_{\R^d} f(\mathbf{y}) \, d\hat{\P}(\mathbf{y} \mid \mathbf{X} \myeq \mathbf{x}) = \sum_{i=1}^n w_{\mathbf{x}}(\mathbf{x}_i) f(\mathbf{y}_i),
\end{equation*}
where we recall the weighting function $w_{\mathbf{x}}(\cdot)$ induced by the forest \eqref{eq: weighting function}.
Therefore, the consistency result \eqref{eq: rate} in Theorem \ref{thm: consistency} directly implies that 
\begin{equation} \label{eq: target consistency}
\sum_{i=1}^n w_{\mathbf{x}}(\mathbf{x}_i) f(\mathbf{y}_i) = \langle f, \hat{\mu}_n(\mathbf{x}) \rangle_\mathcal{H}  \stackrel{p}{\to} \langle f, \mu(\mathbf{x}) \rangle_\mathcal{H} = \E[f(\bold{Y})\mid\bold{X}\myeq\bold{x}] \quad \text{ for any $f \in \mathcal{H}$},
\end{equation}
i.e. that the DRF consistently estimates the targets of the form $\tau(\bold{x}) = \E[f(\bold{Y})\mid\bold{X}\myeq\bold{x}]$, for $f \in \mathcal{H}$. From \eqref{eq: rate} we also obtain the rate of convergence when $s_n \asymp n^\beta$:
\begin{equation*}
\left|\sum_{i=1}^n w_{\mathbf{x}}(\mathbf{x}_i) f(\mathbf{y}_i) - \E[f(\bold{Y})\mid\bold{X}\myeq\bold{x}]\right| =  \mathcal{O}_{p}\left( n^{-\gamma} \norm{f}_\mathcal{H}\right), 
\end{equation*}
for $\gamma$ as in Theorem \ref{thm: consistency}.
When $k$ is continuous, it is well known that all elements of $\mathcal{H}$ are continuous, see e.g.\ \citet{hilbertspacebook}. Under certain assumptions on the kernel and its input space, holding for several popular kernels, (e.g.\ the Gaussian kernel) \citep{optimalestimationofprobabilitymeasures}, we can generalize the convergence result \eqref{eq: target consistency} to any bounded and continuous function  $f:\R^d \to \R$, as the convergence of measures $\hat{\P}(\bold{Y}\mid\bold{X}\myeq\bold{x}) \to \P(\bold{Y}\mid\bold{X}\myeq\bold{x})$ in the MMD metric will also imply their weak convergence, i.e. $k$ metrizes weak convergence \citep{optimalestimationofprobabilitymeasures, simon2018kernel, simon2020metrizing}:

\begin{corollary} \label{cor: metrizing weak convergence}
Assume that one of the following two sets of conditions holds:
\begin{itemize}
    \item[(a)] The kernel $k$ is bounded, (jointly) continuous and has 
    \begin{align}\label{characteristiccondapp}
        \int \int k(\mathbf{x},\mathbf{y}) d\mathcal{P}(\mathbf{x}) d\mathcal{P}(\mathbf{y}) > 0 \  \ \forall \mathcal{P} \in \mathcal{M}_b(\R^d)\setminus \{0\}.
    \end{align}
    Moreover, $\mathbf{y} \mapsto k(\mathbf{y}_0, \mathbf{y})$ is vanishing at infinity, for all $\mathbf{y}_0 \in \R^d$.
    \item[(b)] The kernel $k$ is bounded, shift-invariant, (jointly) continuous and $\nu$ in the Bochner representation in \eqref{eq: bochner} is supported on all of $\R^d$. Moreover, $\mathbf{Y}$ takes its values almost surely in a closed and bounded subset of $\R^d$.
\end{itemize}
Then, under the conditions of Theorem \ref{thm: consistency}, we have for any bounded and continuous function $f:\R^d \to \R$ that DRF consistently estimates the target $\tau(\bold{x})=\E[f(\bold{Y})\mid\bold{X}\myeq\bold{x}]$ for any $\bold{x} \in [0,1]^p$:
$$\sum_{i=1}^n w_{\mathbf{x}}(\mathbf{x}_i) f(\mathbf{y}_i) \,\stackrel{p}{\to}\, \E[f(\bold{Y})\mid\bold{X}\myeq\bold{x}].$$
\end{corollary}

Recalling the Portmanteau Lemma on separable metric spaces, see e.g.\ \citet[Chapter 11]{dudley}, this has several other interesting consequences, such as the consistency of CDF and quantile estimates; Let $F_{\mathbf{Y}\mid\bold{X}\myeq\bold{x}}(\cdot)$ be the conditional CDF of $\mathbf{Y}$ and for any index $1\leq i\leq d$, let $F_{Y_i\mid\bold{X}\myeq\bold{x}}(\cdot)$ be the conditional CDF of $Y_i$ and $F_{Y_i\mid\bold{X}\myeq\bold{x}}^{-1}(\cdot)$ its generalized inverse, i.e. the quantile function. Let $\hat{F}_{Y_i\mid\bold{X}\myeq\bold{x}}(\cdot)$ and $\hat{F}_{Y_i\mid\bold{X}\myeq\bold{x}}^{-1}(\cdot)$ be the corresponding DRF estimates via weighting function \eqref{eq: CDF estimate}. Then we have the following result:



\begin{corollary}\label{cor: cdfresult}
Under the conditions of Corollary \ref{cor: metrizing weak convergence}, for any $1 \leq i \leq d$, we have 
\begin{align*}
    \hat{F}_{\mathbf{Y} \mid\bold{X}\myeq\bold{x}}(\mathbf{t}) \, &\stackrel{p}{\to}\, F_{\mathbf{Y}\mid\bold{X}\myeq\bold{x}}(\mathbf{t}) \\
    \hat{F}_{Y_i\mid\bold{X}\myeq\bold{x}}^{-1}(t) \, &\stackrel{p}{\to}\, F_{Y_i\mid\bold{X}\myeq\bold{x}}^{-1}(t),
\end{align*}
for all points of continuity $\mathbf{t} \in \R^d$ and $t \in \R$ of $F_{\mathbf{Y}\mid\bold{X}\myeq\bold{x}}(\cdot)$ and $F_{Y_i\mid\bold{X}\myeq\bold{x}}^{-1}(\cdot)$ respectively.
\end{corollary}

\section{Applications and Numerical Experiments}
\label{sec: applications}
The goal of this section is to demonstrate the versatility and applicability of DRF for many practical problems. We show that DRF can be used not only as an estimator of the multivariate conditional distribution, but also as a two-step method to easily obtain out-of-the box estimators for various, and potentially complex, targets $\tau(\bold{x})$.

Our main focus lies on the more complicated targets which cannot be that straightforwardly approached by conventional methods. However, we also illustrate the usage of DRF for certain applications for which there already exist several well-established methods. Whenever possible in such cases, we compare the performance of DRF with the specialized, task-specific methods to show that, despite its generality, there is at most a very small loss of precision. However, we should point out that for many targets such as, that can not be written in a form of a conditional mean or a conditional quantile, for example, conditional correlation, direct comparison of the accuracy is not possible for real data, since no suitable loss function exists and the ground truth is unknown. Finally, we show that, in addition to directly estimating certain targets, DRF can also be a very useful tool for many different applications, such as causality and fairness.

Detailed descriptions of all competing methods, data sets and the corresponding analyses can be found in Appendix \ref{appendix: simulation details}, and some additional simulations can be found in the Appendix \ref{appendix: additional examples}. 

\subsection{Estimation of Conditional Multivariate Distributions}
\label{sec: benchmarks}
In order to provide good estimates for any target $\tau(\bold{x}) = \tau(\P(\bold{Y}\mid\bold{X}\myeq\bold{x}))$, our method needs to estimate the conditional multivariate distribution $\P(\bold{Y}\mid \bold{X}\myeq\bold{x})$ well. Therefore, we first investigate here the accuracy of the DRF estimate \eqref{eq: distributional estimate} of the full conditional distribution and compare its performance with the performance of several existing methods. 

In addition to a few simple methods such as the $k$-nearest neighbors or the kernel regression, which locally weight the training points, we also consider the CME estimator of \citet{OurapproachtoCME} and several advanced machine learning methods such as the Conditional Generative Adversarial Network (CGAN) \citep{mirza2014conditional, aggarwal2019benchmarking}, Conditional Variational Autoencoder (CVAE) \citep{sohn2015learning} and Masked Autoregressive Flow \citep{papamakarios2017masked}. It is worth mentioning that the focus in the machine learning literature has been more on applications where $d$ is very large (e.g.\ pixels of an image) and $p$ is very small (such as image labels). Even though some methods do not provide the estimated conditional distribution in a form as simple as DRF, one is still able to sample from the estimated distribution and thus perform any subsequent analysis and make fair comparisons between the methods. For the CME estimator we simply set the negative weights to zero and renormalize, such that the weights are nonnegative and sum to one.

\begin{figure}
    \includegraphics[width=1\linewidth]{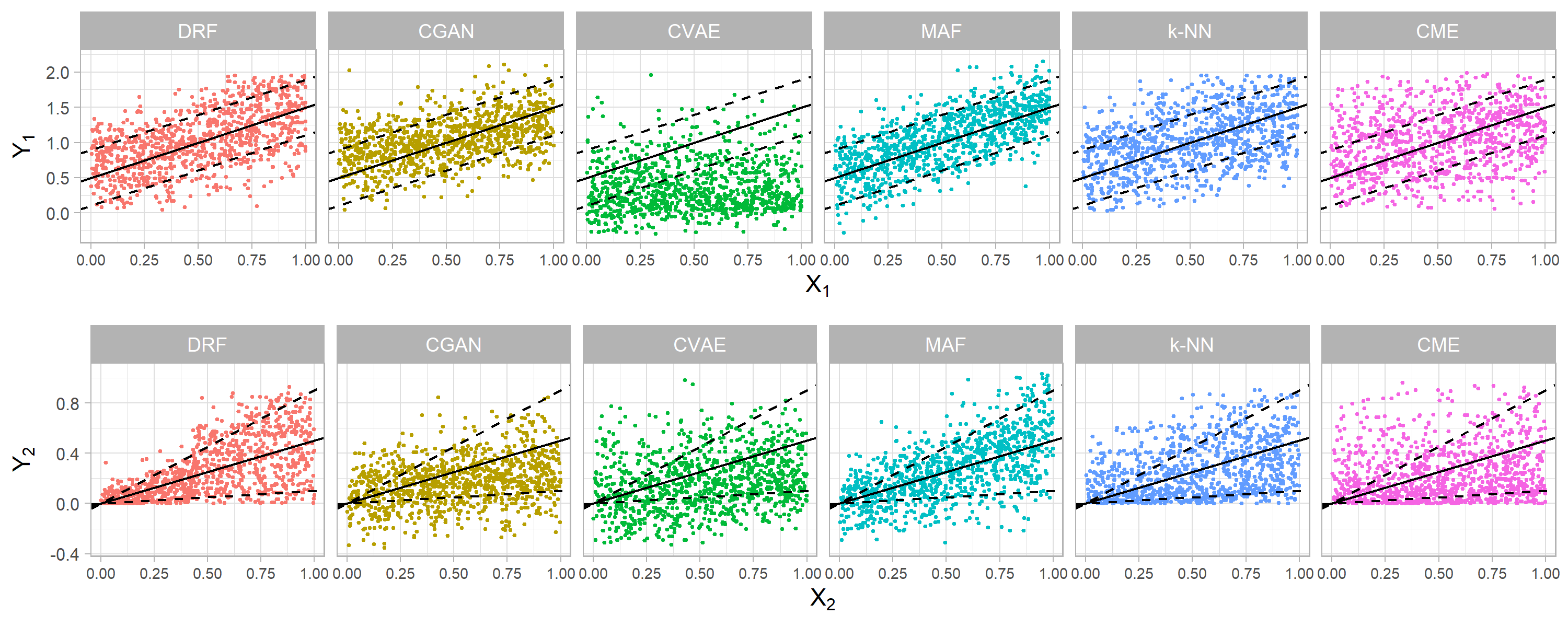}
    \caption{The illustration of the estimated joint conditional distribution obtained by different methods for the toy example \eqref{eq: vignette example}. For $1000$ randomly generated test points $\bold{X}_{\text{test}} \sim U(0, 1)^p$ the top row shows the estimated distribution of the response component $Y_1$, whereas the bottom row shows the estimated distribution of $Y_2$. The $0.1$ and $0.9$ quantiles of the true conditional distribution are indicated by a dashed black line, whereas the conditional mean is shown as a black solid line.}
    \label{fig: distribution illustration}
\end{figure}

We first illustrate the estimated distributions by the above methods on a toy example where $n=1000, p=10, d=2$ and
\begin{equation} \label{eq: vignette example}
Y_1 \indep Y_2\mid \bold{X}\myeq\bold{x}, \quad Y_1 \mid \bold{X}\myeq\bold{x} \sim U(x_1, x_1 + 1), \quad Y_2 \mid \bold{X}\myeq\bold{x} \sim U(0, x_2), \quad \bold{X} \sim U(0,1)^p.
\end{equation}
In the above example $X_1$ affects the mean of $Y_1$, whereas $X_2$ affects the both mean and variance of $Y_2$, and $X_3,\ldots, X_p$ have no impact. The results can be seen in Figure \ref{fig: distribution illustration}. We see that, unlike some other methods, DRF is able to balance the importance of the predictors $X_1$ and $X_2$ and thus to estimate the distributions of $Y_1$ and $Y_2$ well.

One can do a more extensive comparison on a collection of real data sets. We use the benchmark data sets from the multi-target regression literature \citep{mtr} together with some additional ones created from the data sets described throughout this paper. The performance of DRF is compared with the performance of other existing methods for nonparametric estimation of multivariate distributions by using the Negative Log Predictive Density (NLPD) loss, which evaluates the logarithm of the induced multivariate density estimate \citep{quinonero2005evaluating}. As the number of test points grows to infinity, NLPD loss becomes equivalent to the average KL divergence between the estimated and the true conditional distribution and is thus able to capture how well one estimates the whole distribution, instead of only its mean.

In addition to the methods mentioned above, we also include some methods that are intended only for mean prediction, by assuming that the distribution of the response around its mean is homogeneous, i.e. that the conditional distribution $\P\left(\bold{Y} - \E[\bold{Y} \mid \bold{X}] \mid \bold{X}\myeq\bold{x}\right)$ does not depend on $\bold{x}$. This is fitted by regressing each component of $\bold{Y}$ separately on $\bold{X}$ and using the pooled residuals. We consider the standard nonparametric regression methods such as Random Forest \citep{breiman2001random}, XGBoost \citep{Chen:XGB}, and Deep Neural Networks \citep{goodfellow2016deep}.

\begin{table}
\centering
\begin{tabularx}{\textwidth}{lcccccccccccccc}
    \rot[0][3.5em]{}
    & \rot[45][1.8em]{jura}
    & \rot[45][1.8em]{slump}
    & \rot[45][1.8em]{wq}
    & \rot[45][1.8em]{enb} 
    & \rot[45][1.8em]{atp1d}
    & \rot[45][1.8em]{atp7d}
    & \rot[45][1.8em]{scpf} 
    & \rot[45][1.8em]{sf1} 
    & \rot[45][1.8em]{sf2} 
    & \rot[45][1.8em]{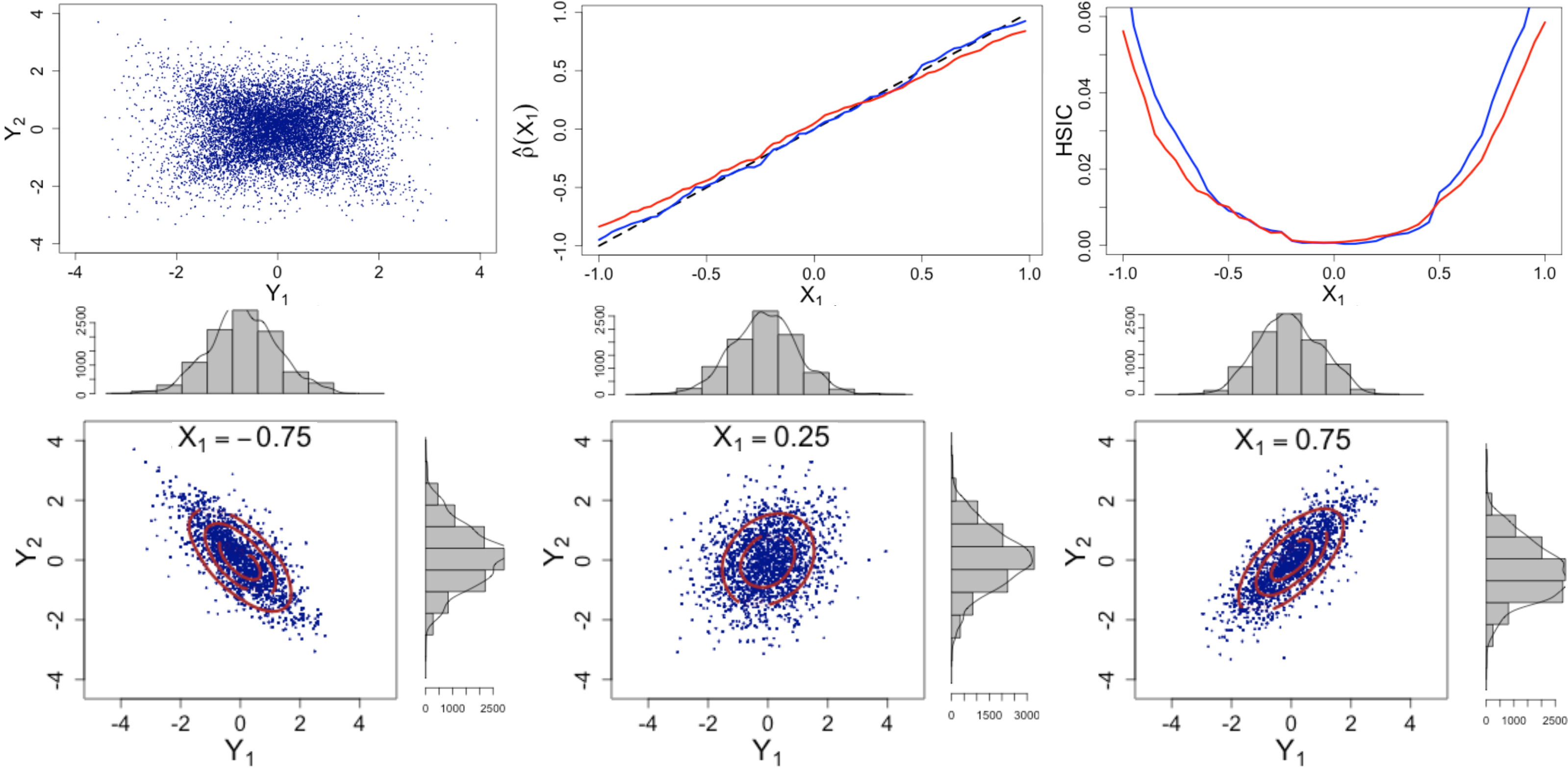} 
    & \rot[45][1.8em]{wage}
    & \rot[45][1.8em]{births1}
    & \rot[45][1.8em]{births2}
    & \rot[45][1.8em]{air} \\
    \bottomrule
    
    \multicolumn{1}{c|}{$n$} & 359 & 103 & 1K & 768 & 337 & 296 & 143 & 323 & 1K & 5K & 10K & 10K & 10K & 10K\\
    \multicolumn{1}{c|}{$p$} & 15 & 7  & 16 & 8 & 370 & 370 & 8 & 21 &22 & 10 & 73 & 23 & 24 & 15 \\
    \multicolumn{1}{c|}{$d$} & 3 & 3& 14 &2 & 6 & 6 & 3 & 3 & 6 & 2 & 2 & 2 &4 & 6\\
     \hline
    \multicolumn{1}{c|}{$\text{DRF}$} & 3.9 & \textbf{4.0} & \textbf{22.5} & 2.1 & 7.3 & \textbf{7.0} & \textbf{2.0} & -24.2 &  \textbf{-24.3} & \textbf{2.8}  & \textbf{2.8} & 2.5 & \textbf{4.2} & \textbf{8.5} \\
    \multicolumn{1}{c|}{$\text{CGAN}$} & 10.8 & 5.3 & 27.3 & 3.5 & 10.4 & 363 & 4.8 & 9.8 & 21.1 & 5.8 & 360 & \textbf{2.4} & {\small $>$1K}  & 11.8 \\ 
    \multicolumn{1}{c|}{$\text{CVAE}$} & 4.8& 37.8 & 36.8 & 2.6  & {\small $>$1K}  &  {\small $>$1K}  & 108.8 & 8.6 & {\small $>$1K} & 2.9 & {\small $>$1K} & {\small $>$1K}  & 49.7 & 9.6 \\ 
    \multicolumn{1}{c|}{$\text{MAF}$} & 4.6 & 4.5 & 23.9& 3.0  & 8.0 & 8.1 & 2.6 & 4.7 & 3.8 & 2.9 & 3.0 & 2.5 & {\small $>$1K} & 8.5 \\ 
    \multicolumn{1}{c|}{$\text{k-NN}$} & 4.5 & 5.0 & 23.4 & 2.4 & 8.8 & 8.6 & 4.1 & -22.4 & -19.7 & 2.9 & \textbf{2.8} & 2.7 & 4.4 & 8.8 \\
    \multicolumn{1}{c|}{$\text{kernel}$} & 4.1 & 4.2 & 23.0 & \textbf{2.0} & \textbf{6.6} & 7.1 & 2.9 & -23.0 & -20.6 & 2.8 & 2.9 & 2.6 & 4.3 & 8.4 \\
    \multicolumn{1}{c|}{$\text{RF}$} & 7.1 & 12.1 & 35.2 & 5.7 & 12.7 & 13.3 & 16.7 & 3.9 & 2.2 & 5.8 & 6.1 & 5.0 & 8.3 & 13.9 \\ 
    \multicolumn{1}{c|}{$\text{XGBoost}$} & 11.4 & 38.3 & 25.9 & 3.0 &  {\small $>$1K} &  {\small {\small $>$1K}} &  {\small $>$1K} & 0.3 & 1.6 & 3.5 & 2.9 &  {\small $>$1K} &  {\small $>$1K} & 12.8 \\ 
    \multicolumn{1}{c|}{$\text{DNN}$} & 4.0 & 4.2 & 23.3 & 2.6 & 8.6 & 8.7 & 2.6 & 2.3 & 2.2 & 2.9 & 3.0 & 2.6 & 5.4 & 8.6 \\ 
     \multicolumn{1}{c|}{$\text{CME}$} & \textbf{3.2} & 4.9 & 23.2 & 2.9 & 8.5 & 8.4 & 2.5 & \textbf{-24.4} & \textbf{-24.3} & 2.8 & 3.5 & 3.8 & 15.2 & 8.8 \\ 
         \hline
\end{tabularx}
\caption{
NLPD loss computed on out-of-sample observations for the estimated conditional distributions obtained by several different methods (corresponding to rows) for many real data sets (corresponding to columns). The best method is indicated in bold.}
\label{benchtable}
\end{table}

The results are shown in Table \ref{benchtable}. We see that DRF performs well for a wide range of sample size and problem dimensionality, especially in problems where $p$ is large and $d$ is moderately big. It does so without the need for any tuning or involved numerical optimization.

\subsection{Estimation of Statistical Functionals}

Because DRF represents the estimated conditional distribution $\hat{\P}(\bold{Y} \mid \bold{X}=\bold{x}) = \sum_{i} w_{\bold{x}}(\bold{x}_i)\cdot \delta_{\bold{y}_i}$ in a convenient form by using weights $w_{\bold{x}}(\bold{x}_i)$, a plug-in estimator $\tau(\hat{\P}(\bold{Y} \mid \bold{X}=\bold{x}))$ of many common real-valued statistical functionals $\tau(\P(\bold{Y} \mid \bold{X}=\bold{x})) \in \R$ can be easily constructed from $w_{\bold{x}}(\cdot)$.

\begin{figure}[h]
    \centering
    \includegraphics[width=1\linewidth]{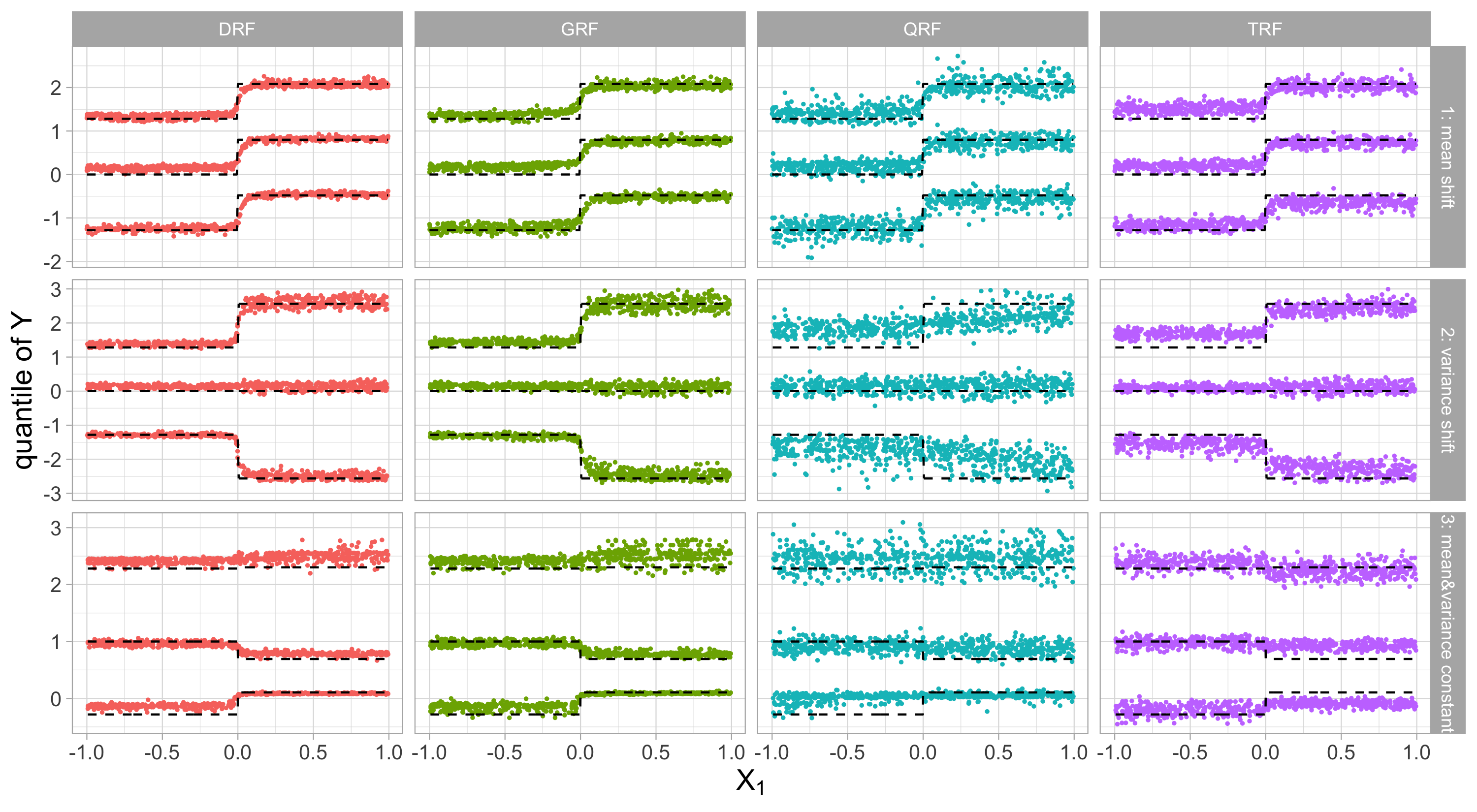}
    \caption{Scatter plot of predictions of the $0.1, 0.5$ and $0.9$ quantiles against $X_1$ for randomly generated $500$ test data points $\bold{X}_{\text{test}} \sim U(-1, 1)^p$. The true values of the quantiles are displayed by black dashed lines. The columns corresponds to different methods DRF (red), GRF (green), QRF (blue), TRF (purple). The rows correspond to different simulation scenarios. The first two are taken from \citet{athey2019generalized}.}
    \label{fig: univariate}
\end{figure}

We first investigate the performance for the classical problem of univariate quantile estimation on simulated data. We consider the following three data generating mechanisms with $p=40, n=2000$ and $\bold{X}_i \iid U(-1,1)^p$:
\begin{itemize}
    \item Scenario 1: $Y \sim N(0.8\cdot \1(X_1 > 0), 1)$ (mean shift based on $X_1$)
    \item Scenario 2: $Y \sim N(0, (1+\1(X_1>0))^2)$ (variance shift based on $X_1$)
    \item Scenario 3: $Y \sim \1(X_1\leq 0) \cdot N(1, 1) + \1(X_1 > 0) \cdot \text{Exp}(1)$ (distribution shift based on $X_1$, constant mean and variance)
\end{itemize}
The first two scenarios correspond exactly to the examples given in \citet{athey2019generalized}. 

In Figure \ref{fig: univariate} we can see the corresponding estimates of the conditional quantiles for DRF, Quantile Regression Forest (QRF) \citep{meinshausen2006quantile}, which uses the same forest construction with CART splitting criterion as the original Random Forest \citep{breiman2001random} but estimates the quantiles from the induced weighting function, Generalized Random Forests (GRF) \citep{athey2019generalized} with a splitting criterion specifically designed for quantile estimation and Transformation Forests (TRF) \citep{hothorn2017transformation}. We see that DRF is performing very well even compared to methods that are specifically tailored to quantile estimation. 

The multivariate setting is however more interesting, as one can use DRF to compute much more interesting statistical functionals $\tau(\bold{x})$. We illustrate this in Figure \ref{fig: functionals} for the air quality data set, described in Section \ref{sec: weighting function}. The left plot shows one value of the estimated multivariate CDF, specifically the estimated probability of the event that the air quality index (AQI) is at most $50$ at a given test site. This corresponds to the "Good" category and means that the amount of every air pollutant is below a certain threshold determined by the EPA. Such probability estimates can be easily obtained by summing the weights of the training points belonging to the event of interest. For both plots in Figures \ref{air_data} and \ref{fig: functionals}, we train the single DRF with the same set of predictor variables and take the three pollutants O$_3$, SO$_2$ and PM$2.5$ as the responses. In this way we still have training data from many different sites.

\begin{figure}
\centering
\includegraphics[width=1\textwidth]{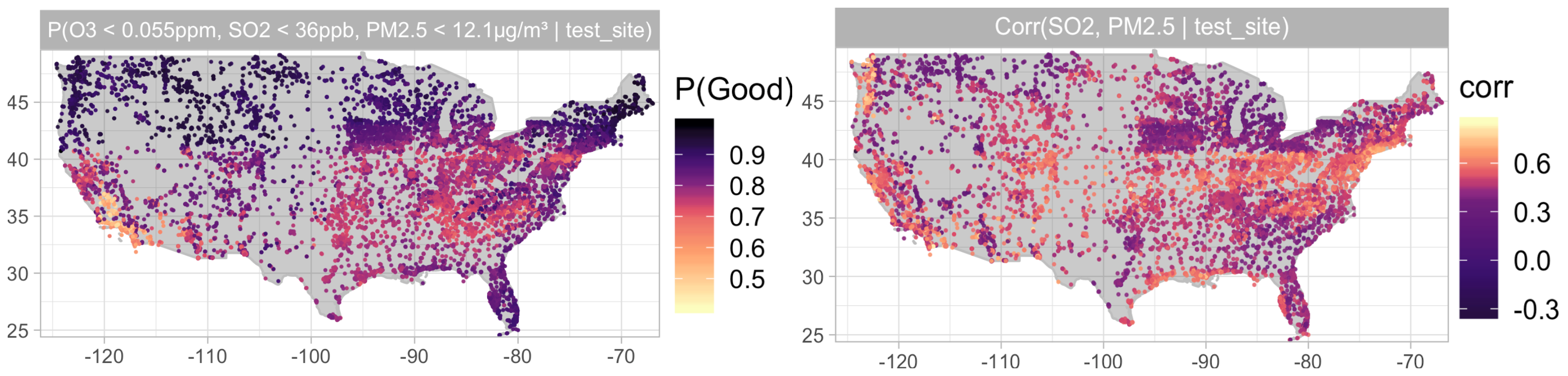}
\caption{Estimates of the probability $\P(\text{AQI} \leq 50 \mid \text{test site})$ (left) and the conditional correlation (right) derived from the DRF estimate of the multivariate conditional distribution.}
\label{fig: functionals}
\end{figure}

In order to investigate the accuracy of the conditional CDF obtained by DRF, we compare the estimated probabilities with estimates of the standard univariate classification forest \citep{breiman2001random} with the response $\1(\text{AQI} \leq 50)$. In the left plot of Figure \ref{comparison}, we can see that the DRF estimates of the $\P(\text{AQI} \leq 50 \mid {\mathbf{X}=\mathbf{x}})$ (also visualized in Figure \ref{fig: functionals}) are quite similar to the estimates of the classification forest predicting the outcome $\1(\text{AQI} \leq 50)$. Furthermore, the cross-entropy loss evaluated on the held-out measurements equals $0.4671$ and $0.4663$ respectively, showing almost no loss of precision. In general, estimating the simple functionals from the weights provided by DRF comes usually at a small to no loss compared to the classical methods specifically designed for this task. 

In addition to the classical functionals $\tau(\bold{x})$ in the form of an expectation $\mathbb{E}(f(\bold{Y})\mid \bold{X}=\bold{x})$ or a quantile $Q_{\alpha}(f(\bold{Y})\mid \bold{X}=\bold{x})$ for some function $f: \mathbb{R}^{d} \rightarrow \mathbb{R}$, which can also be computed by solving the corresponding one-dimensional problems, additional interesting statistical functionals with intrinsically multivariate nature that are not that simple to estimate directly are accessible by DRF, such as, for example, the conditional correlations $\Cor(Y_i,\, Y_j \mid \bold{X}\myeq\bold{x})$. As an illustration, the estimated correlation of the sulfur dioxide ($\text{SO}_2$) and fine particulate matter (PM2.5) is shown in the right plot of Figure \ref{fig: functionals}. The plot reveals also that the local correlation in many big cities is slightly larger than in its surroundings, which can be explained by the fact that the industrial production directly affects the levels of both pollutants.

\begin{figure}
\includegraphics[width=1\textwidth, height=3.5cm]{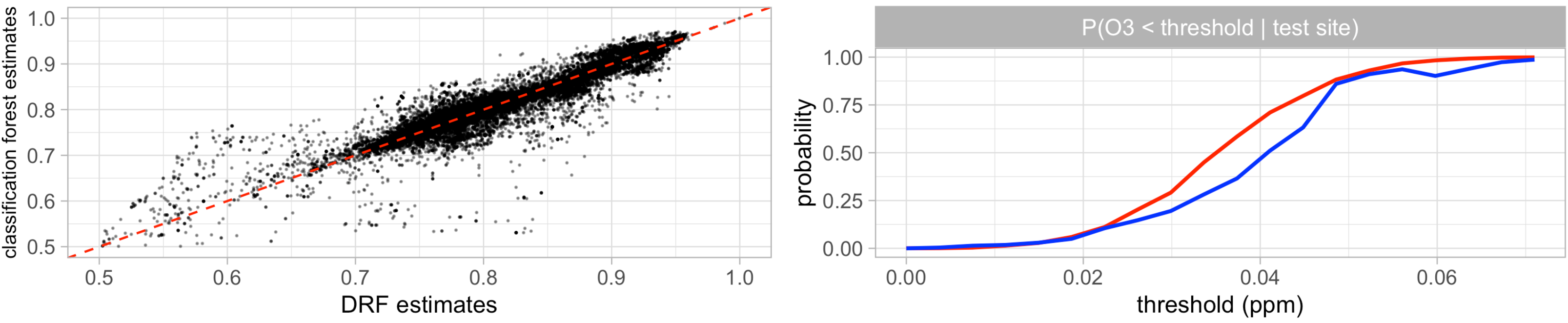}
\caption{Left: Comparison of the CDF estimates obtained by DRF (displayed also in the left plot of Figure \ref{fig: functionals}) and by the classification forest. Right: Example how the CDF estimated by using the classification forest (blue) need not be monotone, whereas the DRF estimates (red) are well-behaved.}
\label{comparison}
\end{figure}

A big advantage of the target-free forest construction of DRF is that all subsequent targets are computed from same the weighting function $w_\bold{x}$ obtained from a single forest fit. First, this is computationally more efficient, since we do not need for every target of interest to fit the method specifically tailored to it. For example, estimating the CDF with classification forests requires fitting one forest for each function value. Secondly and even more importantly, since all statistical functionals are plug-in estimates computed from the same weighting function, the obtained estimates are mathematically well-behaved and mutually compatible. For example, if we estimate $\Cor(Y_i, Y_j \mid \bold{X}\myeq\bold{x})$ by separately estimating the terms $\Cov(Y_i, Y_j \mid \bold{X}\myeq\bold{x})$, $\Var(Y_i \mid \bold{X}\myeq\bold{x})$, and $\Var(Y_j \mid \bold{X}\myeq\bold{x})$, one can not in general guarantee the estimate to be in the range $[-1, 1]$, but this is possible with DRF. Alternatively, the correlation or covariance matrices that are estimated entrywise are guaranteed to be positive semi-definite if one uses DRF. As an additional illustration, Figure \ref{comparison} shows that the estimated (univariate) CDF using the classification forest need not be monotone due to random errors in each predicted value, which can not happen with the DRF estimates.

\subsection{Conditional Copulas and Conditional Independence Testing} 
\label{sec: copulas}

One can use the weighting function not only to estimate certain functionals, but also to obtain more complex objects, such as, for example, the conditional copulas. The well-known Sklar's theorem \citep{sklar1959fonctions} implies that at a point $\textbf{x} \in \mathbb{R}^{p}$, the conditional CDF $\P(\bold{Y} \leq \bold{y} \mid \bold{X}=\bold{x}) = \P(Y_1 \leq y_1, \ldots, Y_d \leq y_d \mid \bold{X}=\bold{x})$ can be represented by a CDF $C_{\bold{x}}$ on $[0,1]^{d}$, the conditional copula at $\bold{x}$, and $d$ conditional marginal CDFs $F_{Y_i\mid\bold{X}\myeq\bold{x}}(y) = \P(Y_{i} \leq y \mid \bold{X}=\bold{x})$ for $1 \leq i \leq d$, as follows:
\begin{equation} \label{eq: copula decomposition}
\P(\bold{Y} \leq \bold{y}\mid  \bold{X}=\bold{x})  =
C_{\bold{x}}\left( F_{Y_1\mid\bold{X}\myeq\bold{x}}(y_1), \ldots, F_{Y_d\mid\bold{X}\myeq\bold{x}}(y_d) \right).
\end{equation}

Copulas capture the dependence of the components $Y_i$ by the joint distribution of the corresponding quantile levels of the marginal distributions: $F_{Y_i\mid\bold{X}\myeq\bold{x}}(Y_i) \in [0,1]$. Decomposing the full multivariate distribution to marginal distributions and the copula is a very useful technique used in many fields such as risk analysis or finance \citep{cherubini2004copula}. Using DRF enables us to estimate copulas conditionally, either by fitting certain parametric model or nonparametrically, directly from the weights.

\begin{figure}
\centering
\includegraphics[width=1\textwidth]{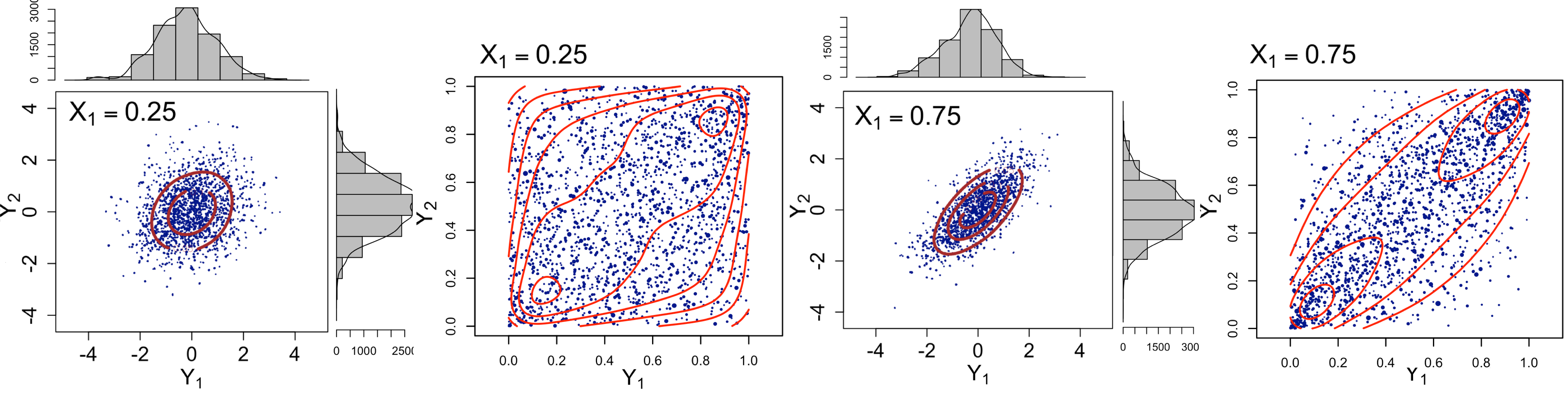}
\caption{Estimated conditional joint distribution of $(Y_1, Y_2)$ and conditional copulas obtained by DRF at different test points $\bold{x}$, where $x_1$ equals $0.25$ and $0.75$ respectively. The red lines are the contours of the true multivariate density function.} 
\label{fig: gaussian_copula}
\end{figure}

To illustrate this, consider an example where the $5$-dimensional $\bold{Y}$ is generated from the equicorrelated Gaussian copula $\bold{Y}= (Y_{1},\ldots,Y_{5}) \mid \bold{X}=\bold{x} \sim C^{\text{Gauss}}_{\rho(\bold{x})}$ conditionally on the covariates $\bold{X}$ with distribution $\mathbf{X}_i \iid U(0,1)^{p}$, where $p=30$ and $n\myeq 5000$. All $Y_i$ have a $N(0,1)$ distribution marginally, but their conditional correlation for $i \neq j$ is given by $\text{Cor}(Y_i, Y_j) = \rho(\bold{x}) = x_1$. Figure \ref{fig: gaussian_copula} shows that DRF estimates the full conditional distribution at different test points $\bold{x}$ quite accurately and thus we can obtain a good nonparametric estimate of the conditional copula as follows. First, for each component $Y_i$, we compute the corresponding marginal CDF estimate $\hat{F}_{Y_i\mid\bold{X}\myeq\bold{x}}(\cdot)$ from the weights. Second, we map each response $\bold{y}_i \to \bold{u}_i \coloneqq \left(\hat{F}_{Y_1\mid\bold{X}\myeq\bold{x}}\left((\bold{y}_i)_1\right), \ldots, \hat{F}_{Y_d\mid\bold{X}\myeq\bold{x}}\left((\bold{y}_i)_d\right)\right)$. The copula estimate is finally obtained from the weighted distribution $\sum_{i=1}^n w_{\bold{x}}(\bold{x}_i) \delta_{\bold{u}_i}$, from which we sample the points in Figure \ref{fig: gaussian_copula} in order to visualize the copula.


If we want to instead estimate the copula parametrically, we need to find the choice of parameters for a given model family which best matches the estimated conditional distribution, e.g.\ by weighted maximum likelihood estimation (MLE). For the above example, the correlation parameter of the Gaussian copula can be estimated by computing the weighted correlation with weights $\{w_\bold{x}(\bold{x}_i)\}_{i=1}^n$. 
The left plot in Figure \ref{fig: conditional-independence} shows the resulting estimates of the conditional correlation $\text{Cor}\left(Y_1, Y_2 \mid \bold{X}=\bold{x}\right)$ obtained from $\text{DRF}_{\text{MMD}}$, which uses the MMD splitting criterion \eqref{eq: MMD splitcrit} described in Section \ref{sec: MMD splitcrit}, and $\text{DRF}_{\text{CART}}$, which aggregates the marginal CART criteria \citep{kocev2007ensembles, segal2011multivariate}. We see that $\text{DRF}_{\text{MMD}}$ is able to detect the distributional heterogeneity and provide good estimates of the conditional correlation. On the other hand, $\text{DRF}_{\text{CART}}$ cannot detect the change in distribution of $\bold{Y}$ caused by $X_1$ that well. The distributional heterogeneity can not only occur in marginal distribution of the responses (a case extensively studied in the literature), but also in their interdependence structure described by the conditional copula $C_{\bold{x}}$, as one can see from decomposition \eqref{eq: copula decomposition}. Since $\text{DRF}_\text{MMD}$ relies on a distributional metric for its splitting criterion, it is capable of detecting any change in distribution \citep{gretton2007kernel}, whereas aggregating marginal CART criteria for $Y_1, \ldots, Y_d$ in $\text{DRF}_\text{CART}$ only captures the changes in the marginal means.

\begin{figure}
\centering
\includegraphics[width=1\textwidth]{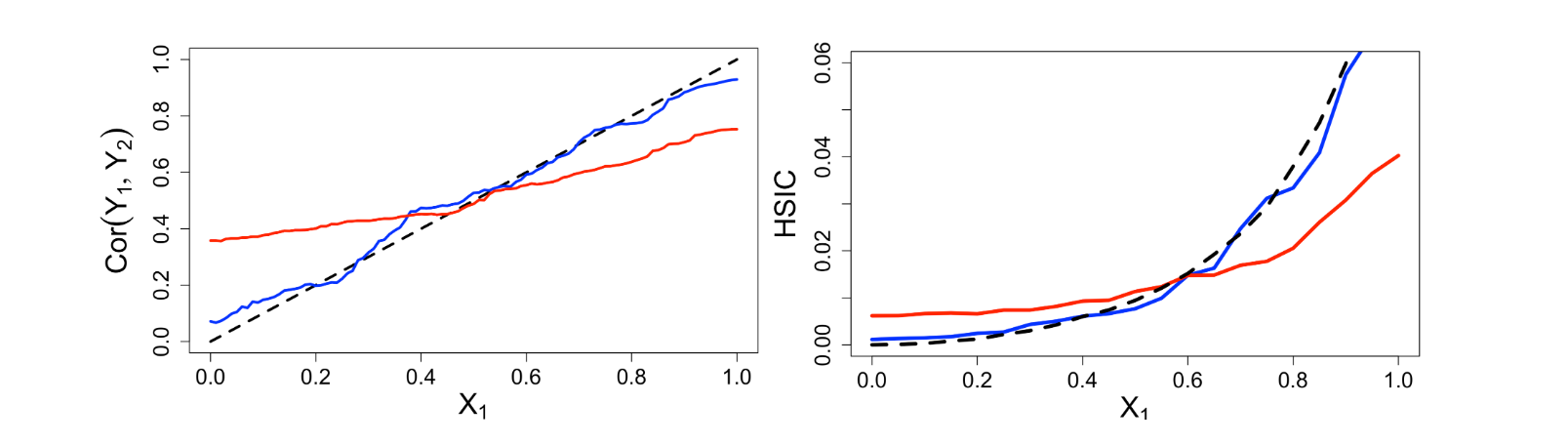}
\caption{Estimated conditional correlation of $Y_1$ and $Y_2$ (left) and estimated conditional dependence quantified by HSIC statistic (right), obtained by $\text{DRF}_{\text{MMD}}$ (blue) and $\text{DRF}_{\text{CART}}$ (red) respectively. For every test point, we set $X_j=0.5, j \neq 1$. Black dashed curve indicates the population values.}
\label{fig: conditional-independence}
\end{figure}

This is further illustrated for a related application of conditional independence testing, where we compute some dependence measure from the obtained weights. For example, we can test the independence $Y_1 \indep Y_2$ conditionally on the event $\bold{X}=\bold{x}$ by using the Hilbert Schmidt Independence Criterion (HSIC) \citep{HSIC}, which measures the difference between the joint distribution and the product of the marginal distributions. The right plot of Figure \ref{fig: conditional-independence} shows that the $\text{DRF}_{\text{MMD}}$ estimates are quite close to the population value of the HSIC, unlike the ones obtained by $\text{DRF}_\text{CART}$.

\subsection{Heterogeneous Regression and Causal Effect Estimation}
\label{sec: causality}

In this and the following section, we illustrate that, in addition to direct estimation of certain targets, DRF can also be a useful tool for complex statistical problems and applications, such as causality.

Suppose we would like to investigate the relationship between some (univariate) quantity of interest $Y$ and certain predictors $\bold{W}$ from heterogeneous data, where the change in distribution of $(\bold{W}, Y)$ can be explained by some other covariates $\bold{X}$. Very often in causality applications, $\bold{W}$ is a (multivariate) treatment variable, $Y$ is the outcome, which is commonly, but not necessarily, binary, and $\bold{X}$ is a set of observed confounding variables for which we need to adjust if we are interested in the causal effect of $\bold{W}$ on $Y$. This is illustrated by the following causal graph:
\begin{center}
\begin{tikzpicture}[
        > = stealth, 
        shorten > = 1pt, 
        auto,
        node distance = 3cm, 
        semithick 
    ]

    \tikzstyle{every state}=[
        draw = black,
        thick,
        fill = white,
        minimum size = 4mm
    ]

    \node[circle, draw=black] (W) at (0,0) {$\bold{W}$};
    \node[circle, draw=black] (X) at (2,1.3)  {$\boldsymbol{X}$};
    \node[circle, draw=black] (Y) at (4, 0) {$Y$};
    \path[->,line width=1.4pt] (W) edge node {} (Y);
    \path[->] (X) edge node {} (W);
    \path[->] (X) edge node {} (Y);
\end{tikzpicture}
\end{center}

The problem of nonparametric confounding adjustment is hard; not only can the marginal distributions of $Y$ and $\bold{W}$ be affected by $\bold{X}$, thus inducing spurious associations due to confounding, but the way how $\bold{W}$ affects $Y$ can itself depend on $\bold{X}$, i.e.\ the treatment effect might be heterogeneous. The total causal effect can be computed by using the adjustment formula \citep{pearl2009causality}:
\begin{align}
\E[Y\mid do(\bold{W}\myeq\bold{w})] &= \int \E[Y\mid do(\bold{W}\myeq\bold{w}), \bold{X}\myeq\bold{x}]\,\P(\bold{X}\myeq\bold{x}\mid do(\bold{W}\myeq\bold{w}))d\bold{x} \nonumber\\
&=\int \E[Y\mid \bold{W}\myeq\bold{w}, \bold{X}\myeq\bold{x}]\,\P(\bold{X}\myeq\bold{x})d\bold{x}. \label{eq: causal effect}
\end{align}
In general, implementing do-calculus for finite samples and potentially non-discrete data might not be straightforward and comes with certain difficulties. In this case, the standard approach would be to estimate the conditional mean $\E[Y\mid \bold{W}\myeq\bold{w}, \bold{X}\myeq\bold{x}]$ nonparametrically by regressing $Y$ on $(\bold{X}, \bold{W})$ with some method of choice and to average out the estimates over different $\bold{x}$ sampled from the observed distribution of $\bold{X}$. Using DRF for this approach is not necessary, but has an advantage that one can easily estimate the full interventional distribution $\P(Y\mid do(\bold{W}\myeq\bold{w}))$ and not only the interventional mean $\E[Y\mid do(\bold{W}\myeq\bold{w})]$.

Another way of computing the causal effect, which allows to add more structure to the problem, is explained in the following: We use DRF to first fit the forest with the multivariate response $(\bold{W}, Y)$ and the predictors $\bold{X}$. In this way, one can for any point of interest $\bold{x}$ obtain the joint distribution of $(\bold{W}, Y)$ conditionally on the event $\bold{X}\myeq\bold{x}$ and then the weights $\{w_\bold{x}(\bold{x}_i)\}_{i=1}^n$ can be used as an input for some regression method for regressing $Y$ on $\bold{W}$ in the second step. This conditional regression fit might be of an independent interest, but it can also be used for estimating the causal effect $\E[Y \mid do(\bold{W}\myeq\bold{w})]$ from \eqref{eq: causal effect}, by averaging the estimates $\E[Y \mid \bold{W}\myeq\bold{w}, \bold{X}\myeq\bold{x}]$ over $\bold{x}$, where $\bold{x}$ is sampled from the empirical observation of $\bold{X}$. 
In this way one can efficiently exploit and incorporate any prior knowledge of the relationship between $\bold{W}$ and $Y$, such as, for example, monotonicity, smoothness or that it satisfies a certain parametric regression model, without imposing any assumptions on the effect of $\bold{X}$ on $(\bold{W}, Y)$. Furthermore, one might be able to better extrapolate to the regions of space where $\P(\bold{W}\myeq\bold{w}, \bold{X}\myeq\bold{x})$ is small, compared to the standard approach which computes $\E[Y \mid \bold{W}\myeq\bold{w}, \bold{X}\myeq\bold{x}]$ directly, by regressing $Y$ on $(\bold{W}, \bold{X})$. Extrapolation is crucial for causal applications, since for computing $\E[Y\mid do(\bold{W}\myeq\bold{w})]$ we are interested in what would happen with $Y$ when our treatment variable $\bold{W}$ is set to be $\bold{w}$, regardless of the value achieved by $\bold{X}$. However, it can easily happen that for this specific combination of $\bold{X}$ and $\bold{W}$ there are very few observed data points, thus making the estimation of the causal effect hard \citep{pearl2009causality}.

\begin{figure}[h]
    \centering
    \includegraphics[width=1\linewidth]{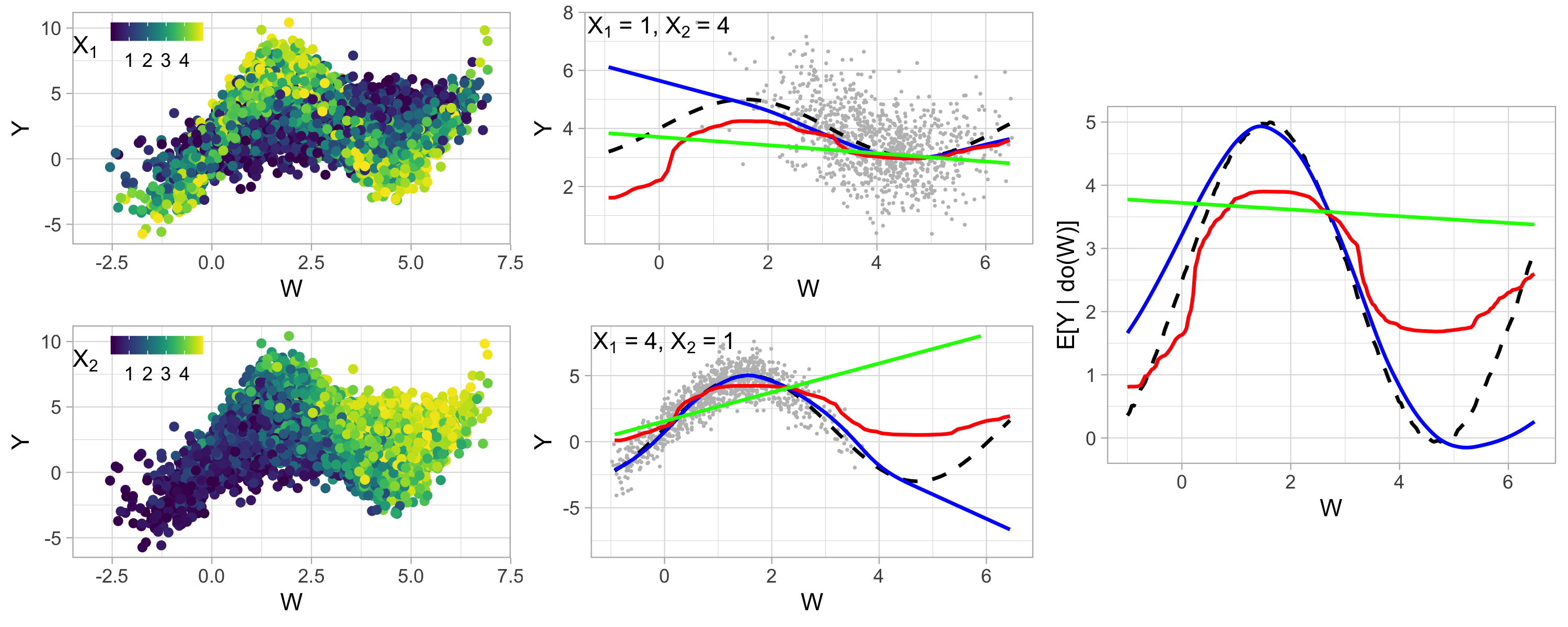}
    \caption{Left: Visualization of heterogeneous synthetic example \eqref{eq: causal example}. Middle: Gray points depict joint distribution of $(W, Y)$ conditionally on $\bold{X}\myeq\bold{x}$, for some choices of $\bold{x}$ indicated in the top left corner. Black curve indicates the true conditional mean $\E[Y\mid W\myeq w, \bold{X}\myeq\bold{x}]$, the blue curve represents the estimate obtained by DRF with response $(W, Y)$ and predictors $\bold{X}$ in combination with smoothing splines regression, the red curve represents the estimate obtained by standard Random Forest, whereas the green line shows the estimate of the Causal Forest \citep{athey2019generalized} which makes the linearity assumption and is thus misspecified. Right: The corresponding estimates for all the methods of the causal effect $\E[Y\mid do(W\myeq w)]$ computed from \eqref{eq: causal effect}. The true causal effect is denoted by a black dashed curve.}
    \label{fig: causal_effect}
\end{figure}

As an illustration, we consider the following synthetic data example, with continuous outcome $Y$, continuous univariate treatment $W$, $n=5000$ and $p=20$:
\begin{equation} \label{eq: causal example}
\bold{X} \sim U(0,5)^p,\quad W \mid \bold{X} \sim N(X_2, 1),\quad Y \mid \bold{X}, W \sim N(X_2 + X_1 \sin(W), 1).
\end{equation}
A visualization of the data can be seen on the left side of Figure \ref{fig: causal_effect}; treatment $W$ affects $Y$ nonlinearly, $X_2$ is a confounding variable that affects the marginal distributions of $Y$ and $W$ and $X_1$ makes the treatment effect heterogeneous. The middle part of Figure \ref{fig: causal_effect} shows the conditional regression fits, i.e.\ the estimates of $\E[Y\mid W\myeq w, \bold{X}\myeq\bold{x}]$ as $w$ varies and $\bold{x}$ is fixed. In general, the conditional regression fit is related to the concept of the conditional average treatment effect (CATE) as it quantifies the effect of $\bold{W}$ on $Y$ for the subpopulation for which $\bold{X}=\bold{x}$.
We see that combination of DRF with response $(Y, W)$ and predictors $\bold{X}$ with the smoothing splines regression of $Y$ on $W$ (blue curve) is more accurate than the estimates obtained by standard Random Forest \citep{breiman2001random} with response $Y$ and predictors $(W, \bold{X})$ (red curve). Furthermore, we see that the former approach can extrapolate better to regions with small number of data points, which enables us to better estimate the causal effect $\E[Y \mid do(W \myeq w)]$ from \eqref{eq: causal effect}, by averaging the corresponding estimates of $\E[Y\mid W\myeq w, \bold{X}\myeq\bold{x}]$ over observed $\bold{x}$, as shown in the right plot of Figure \ref{fig: causal_effect}. 

There exist many successful methods in the literature for estimating the causal effects and the (conditional) average treatment effects for a wide range of settings \citep{abadie2006large, chernozhukov2018double, wager2018estimation, kunzel2019metalearners}. However, some methods are not designed for the most general case and make certain modeling assumptions or are designed specifically for the (very common) case where the treatment variable is univariate or even binary. Due to its versatility, DRF can easily be used when the underlying assumptions of conventional methods are violated, when some additional structure is given in the problem or for the general, nonparametric, settings \citep{imbens2004nonparametric, ernest2015marginal, kennedy2017nonparametric}. Appendix \ref{appendix: additional examples} contains additional comparisons with some existing methods for causal effect estimation.

\subsubsection{Births data}

We further illustrate the applicability of DRF for causality-related problems on the natality data obtained from the Centers for Disease Control and Prevention (CDC) website, where we have information about all recorded births in the USA in 2018. We investigate the relationship between the pregnancy length and the birthweight, an important indicator of baby's health. Not only is this relationship complex, but it also depends on many different factors, such as parents' race, baby's gender, birth multiplicity (single, twins, triplets...) etc. In the left two plots of Figure \ref{birthweight} one can see the estimated joint distribution of birthweight and pregnancy length conditionally on many different covariates, as indicated in the plot. The black curves denote the subsequent regression fit, based on smoothing splines. In addition to the estimate of the mean, indicated by the solid curve, we also include the estimates of the conditional $0.1$- and $0.9$-quantiles, indicated by dashed curves, which is very useful in practice for determining whether a baby is large or small for its gestational age. Notice how DRF assigns less importance to the mother's race when the point of interest is a twin (middle plot), as in this case more weight is given to twin births, regardless of the race of the parents.

\begin{figure*}[h]
\centering
\includegraphics[width=0.9\textwidth]{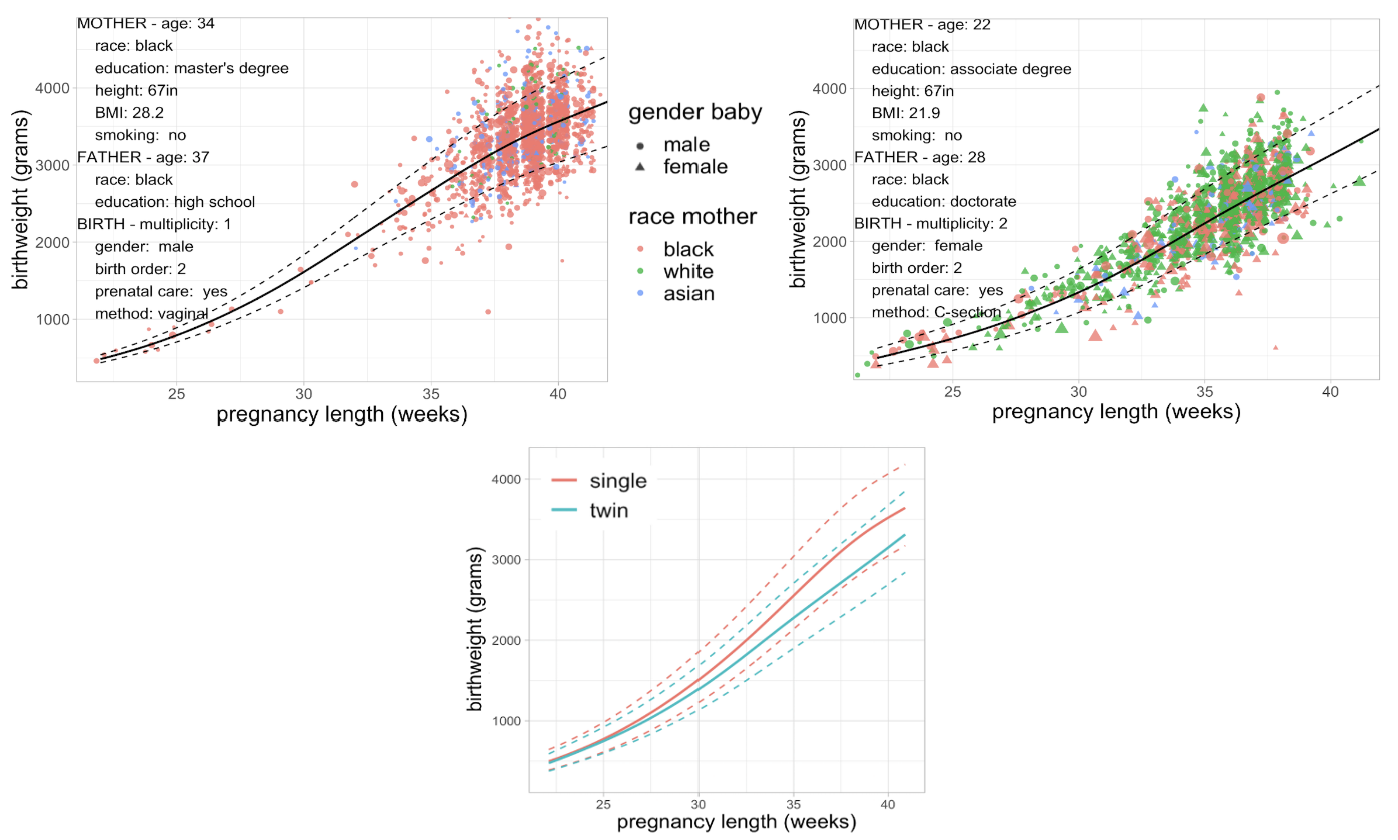}
\caption{Above: estimated relationship of pregnancy length and birthweight, conditionally on the criteria indicated in the upper left corner. Below: estimated interventional effect of twin birth on the birthweight for a fixed pregnancy length. In all plots the solid curves denote the estimated conditional mean and the dashed denote the estimated $0.1$ and $0.9$ quantiles.}
\label{birthweight}
\end{figure*}

Suppose now we would like to understand how a twin birth $T$ causally affects the birthweight $B$, but ignoring the obvious indirect effect due to shorter pregnancy length $L$. For example, sharing of resources between the babies might have some effect on their birthweight. We additionally need to be careful to adjust for other confounding variables $\bold{X}$, such as, for example, the parents' race, which can affect $B, T$ and $L$. We assume that this is represented by the following causal graph:
\begin{center}
\begin{tikzpicture}[
        > = stealth, 
        shorten > = 1pt, 
        auto,
        node distance = 3cm, 
        semithick 
    ]

    \tikzstyle{every state}=[
        draw = black,
        thick,
        fill = white,
        minimum size = 4mm
    ]

    \node[circle, draw=black] (T) at (0,0) {$T$};
    \node[circle, draw=black] (X) at (2,1)  {$\boldsymbol{X}$};
    \node[circle, draw=black] (L) at (2,-1) {$L$};
    \node[circle, draw=black] (B) at (4, 0) {$B$};
    \path[->] (T) edge node {} (L);
    \path[->,line width=1.4pt] (T) edge node {} (B);
    \path[dashed,->] (X) edge node {} (T);
    \path[dashed,->] (X) edge node {} (L);
    \path[dashed,->] (X) edge node {} (B);
    \path[->] (L) edge node {} (B);
\end{tikzpicture}
\end{center}
In order to answer the above question, we investigate the causal quantity $\P(B \mid do(T\myeq t, L\myeq l))$. Even though one cannot make such do-intervention in practice, this quantity describes the total causal effect if the birth multiplicity and the length of the pregnancy could be manipulated and thus for a fixed pregnancy length $l$, we can see the difference in birthweight due to $T$. We compute this quantity as above, by using DRF with subsequent regression fits, which has the advantage of better extrapolating to regions with small probability, such as long twin pregnancies (see the middle plot of Figure \ref{birthweight}). In the right plot of Figure \ref{birthweight} we show the mean and quantiles of the estimated interventional distribution and we see that, as one might expect, a twin birth causes smaller birthweight on average, with the difference increasing with the length of the pregnancy.

\subsection{Fairness}

Being able to compute different causal quantities with DRF could prove useful in a range of applications, including fairness \citep{kusner2017counterfactual}. We investigate the data on approximately $1$ million full-time employees from the 2018 American Community Survey by the US Census Bureau from which we have extracted the salary information and all covariates that might be relevant for salaries. In the bottom left plot of Figure \ref{fig: wage} one can see the distribution of hourly salary of men and women (on the logarithmic scale). The overall salary was scaled with working hours to account for working part-time and for the fact that certain jobs have different working hours. We can see that men are paid more in general, especially for the very high salaries. The difference between the median hourly salaries, a commonly used statistic in practice, amounts $17\%$ for this data set.

We would like to answer whether the observed gender pay gap in the data is indeed unfair, i.e.\ only due to the gender, or whether it can at least in part be explained by some other factors, such as age, job type, number of children, geography, race, attained education level and many others. Hypothetically, it could be, for example, that women have a preference for jobs that are paid less, thus causing the gender pay gap.

In order to answer this question, we assume that the data is obtained from the following causal graph, where $G$ denotes the gender, $W$ the hourly wage and all other factors are denoted by $\bold{X}$:
\begin{center}
\begin{tikzpicture}[
        > = stealth, 
        shorten > = 1pt, 
        auto,
        node distance = 3cm, 
        semithick 
    ]

    \tikzstyle{every state}=[
        draw = black,
        thick,
        fill = white,
        minimum size = 4mm
    ]

    \node[circle, draw=black] (G) at (0,0) {$G$};
    \node[circle, draw=black] (X) at (2,1)  {$\boldsymbol{X}$};
    \node[circle, draw=black] (W) at (4, 0) {$W$};
    \path[->,line width=1.4pt] (G) edge node {} (W);
    \path[->] (G) edge node {} (X);
    \path[->] (X) edge node {} (W);
\end{tikzpicture}
\end{center}
i.e.\ $G$ is a source node and $W$ is a sink node in the graph.
In order to determine the direct effect of the gender on wage that is not mediated by other factors, we would like to compute the distribution of the nested counterfactual $W(\text{male},\, \bold{X}(\text{female}))$, which is interpreted as the women's wage had they been treated in same way as men by their employers for determining the salary, but without changing their propensities for other characteristics, such as the choice of occupation \citep{chernozhukov2013inference}. Therefore, it can be obtained from the observed distribution as follows:
\begin{align}
\P\left(W(\text{male},\, \bold{X}(\text{female}))\right) &= \int \P\left(W( G\myeq\text{male},\, \bold{X}\myeq\bold{x})\right)\P(\bold{X}\myeq\bold{x}\mid G\myeq\text{female})d\bold{x} \nonumber\\
&= \int \P\left(W\mid G\myeq\text{male},\, \bold{X}\myeq\bold{x}\right)\P(\bold{X}\myeq\bold{x}\mid G\myeq\text{female})d\bold{x},
\label{eq: nested_counterfactual}
\end{align}
Put in the language of the fairness literature, it quantifies the unfairness when all variables $\bold{X}$ are assumed to be resolving \citep{kilbertus2017avoiding}, meaning that any difference in salaries directly due to factors $\bold{X}$ is not viewed as gender discrimination. For example, one does not consider unfair if people with low education level get lower salaries, even if the gender distribution in this group is not balanced.

\begin{figure}[h]
\includegraphics[width=1\textwidth]{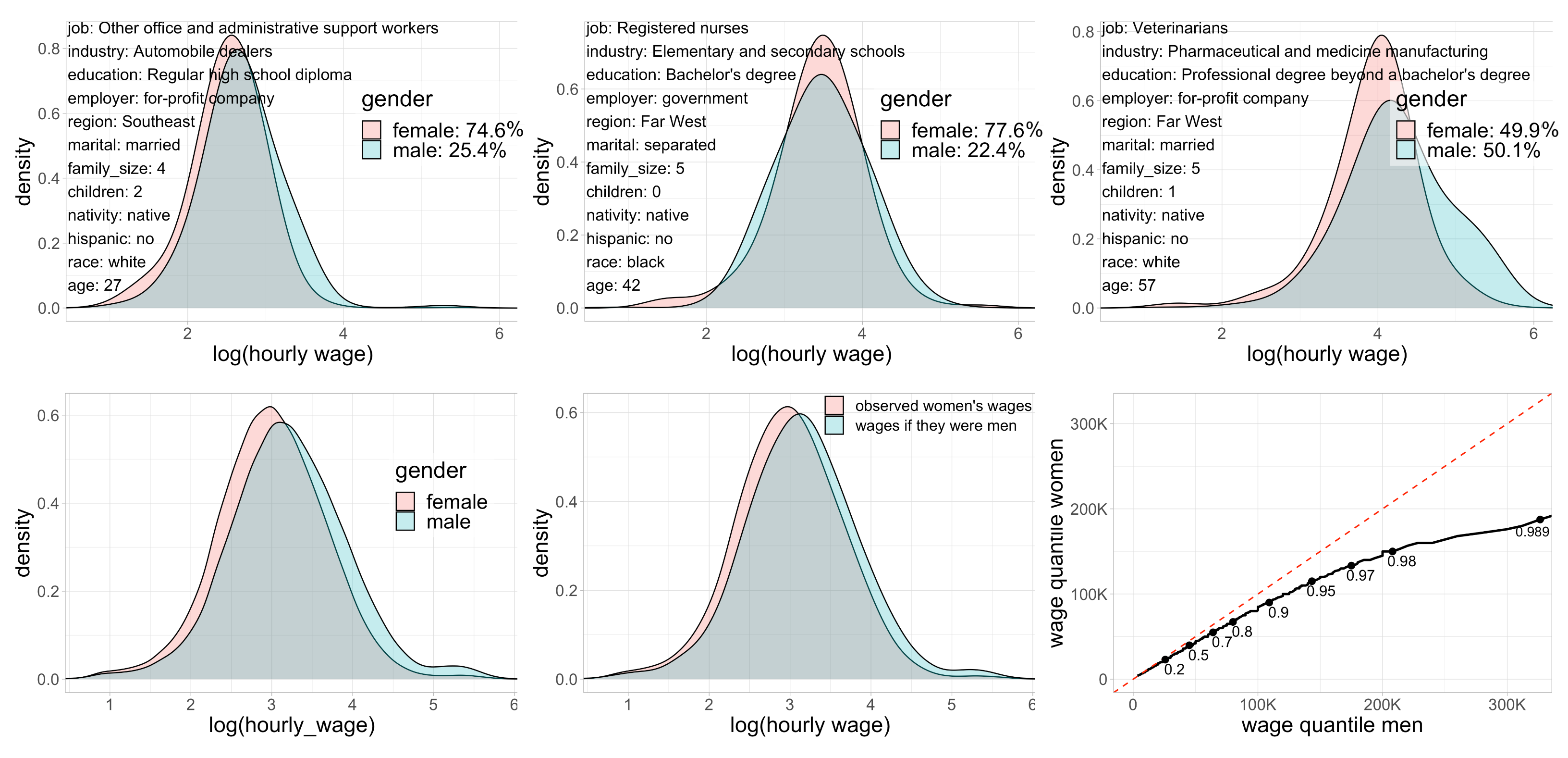}
\caption{Top row: Estimated joint distribution of wage and gender for some fixed values of other covariates $\bold{X}$ indicated in the top left part of each plot. Bottom row: observed overall distribution of salaries (left), estimated counterfactual distribution $\P\left(W(\text{male},\, \bold{X}(\text{female}))\right)$ of women's salaries (middle) and the quantile comparison of the counterfactual distribution of women's salaries and the observed distribution of men's salaries (right).}
\label{fig: wage}
\end{figure}

There are several ways how one can compute the distribution of $W(\text{male},\, \bold{X}(\text{female}))$ from \eqref{eq: nested_counterfactual} with DRF. The most straightforward option is to take $W$ as the response and $(G, \bold{X})$ as predictors in order to compute the conditional distribution $\P\left(W\mid G\myeq\text{male},\, \bold{X}\myeq\bold{x}\right)$. 
However, with this approach it could happen that for predicting $\P\left(W\mid G\myeq\text{male},\, \bold{X}\myeq\bold{x}\right)$ we also assign weight to training data points for which $G\myeq\text{female}$. This happens if in some trees we did not split on variable $G$, which is likely, for example, if $\P(G=\text{male} \mid \bold{X}\myeq\bold{x})$ is low. Using salaries of both genders to estimate the distribution of men's salaries might be an issue if our goal is to objectively compare how women and men are paid.

Another approach is to take $(W, G)$ as a multivariate response and $\bold{X}$ as the predictors for DRF and thus obtain joint distribution of $(W, G)$ conditionally on the event $\bold{X}\myeq\bold{x}$. In this way we can also quantify the gender discrimination of a single individual with characteristics $\bold{x}$ by comparing his/her salary to the corresponding quantile of the salary distribution of people of the opposite gender with the same characteristics $\bold{x}$ \citep{plevcko2019fair}. This is interesting because the distribution of salaries, and thus also the gender discrimination, can be quite different depending on other factors such as the industry sector or job type, as illustrated for a few choices of $\bold{x}$ in the top row of Figure \ref{fig: wage}.

Finally, by averaging the DRF estimates of $\P\left(W\mid\bold{X}\myeq\bold{x},\, G\myeq\text{male}\right)$, conveniently represented via the weights, over different $\bold{x}$ sampled from the distribution $\P(\bold{X}\mid G\myeq\text{female})$, we can compute the distribution of the nested counterfactual $W(\text{male},\, \bold{X}(\text{female}))$ \citep{chernozhukov2013inference}. In the middle panel in the bottom row of Figure \ref{fig: wage} a noticeable difference in the means, also called natural direct effect in the causality literature \citep{pearl2009causality}, is still visible between the observed distribution of women's salaries and the hypothetical distribution of their salaries had they been treated as men, despite adjusting for indirect effects of the gender via covariates $\bold{X}$. By further matching the quantiles of the counterfactual distribution $\P\left(W(\text{male},\, \bold{X}(\text{female}))\right)$ with the corresponding quantiles of the observed distribution of men's salaries in the bottom right panel of Figure \ref{fig: wage}, we can also see that the adjusted gender pay gap even increases for larger salaries. Median hourly wage for women is still $11\%$ lower than the median wage for the hypothetical population of men with exactly the same characteristics $\bold{X}$ as women, indicating that only a minor proportion of the actually observed hourly wage difference of $17\%$ can be explained by other demographic factors. 

\section{Conclusion}
We have shown that DRF is a flexible, general and powerful tool, which exploits the well-known properties of the Random Forest as an adaptive nearest neighbor method via the induced weighting function. Not only does it estimate multivariate conditional distributions well, but it constructs the forest in a model- and target-free way and is thus an easy to use out-of-the-box algorithm for many, potentially complex, learning problems in a wide range of applications, including also causality and fairness, with competitive performance even for problems with existing tailored methods.

\newpage

\appendix

\section{Implementation Details}\label{appendix: impdetails}

Here we present the implementation of the Distributional Random Forests (DRF) in detail. The code is available in the R-package \texttt{drf} and the Python package \texttt{drf}.
The implementation is based on the implementations of the R-packages \texttt{grf} \citep{athey2019generalized} and \texttt{ranger} \citep{wright2015ranger}. The largest difference is in the splitting criterion itself and the provided user interface. Algorithm \ref{pseudocode} gives the pseudocode for the forest construction and computation of the weighting function $w_\bold{x}(\cdot)$.

\begin{algorithm}[htp]
\caption{Pseudocode for Distributional Random Forest} \label{pseudocode}
\begin{algorithmic}[1]
\Procedure{BuildForest}{set of samples $\mathcal{S} = \{(\bold{x}_i, \bold{y}_i)\}_{i=1}^n$, number of trees $N$}
    \For{$i=1,\ldots, N$}
        \State $\mathcal{S}_\text{subsample}$ = \textsc{Subsample}($\mathcal{S}$)
        \State $\mathcal{S}_\text{build}, \mathcal{S}_\text{populate} \gets$ \textsc{SplitSamples}($\mathcal{S}_\text{subsample}$) \Comment{Honesty principle, see above}
        \State $\mathcal{T}_i \gets$  \textsc{CreateNewTree}($\mathcal{S}_\text{build}$) \Comment{Samples $\mathcal{S}_\text{build}$ used for building the tree}
        \State \textsc{BuildTree}(\textsc{RootNode}($\mathcal{T}_i$)) \Comment{Start recursion from the root node}
        \State \textsc{PopulateLeaves}($\mathcal{T}_i, \mathcal{S}_\text{populate})$ \Comment{Samples $\mathcal{S}_\text{populate}$ used for computing $w_\bold{x}(\cdot)$}
    \EndFor
    \State \textbf{return} $\mathcal{F} = \{\mathcal{T}_1, \ldots, \mathcal{T}_N\}$
\EndProcedure

\item[]

\Procedure{BuildTree}{current node $\mathcal{N}$} \Comment{Recursively constructs the trees} 
    \If{\textsc{StoppingCriterion}($\mathcal{N}$)} \Comment{E.g.\ if only a few samples left}
        \State \textbf{return}
    \EndIf
    \State $\mathcal{S} \gets$ \textsc{GetSamples}($\mathcal{N}$)
    \State $\mathcal{I} \gets$ \textsc{GetSplitVariables}() \Comment{Random set of candidate variables}
    \State $\mathcal{C}$ $\gets$  \textsc{InitializeSplits}() \Comment{Here we store info about candidate splits}
    \For{idx $\in \mathcal{I}$, level $l$} \Comment{$l$ iterates over all values of variable $X_{\text{idx}}$}
        \State $\mathcal{S}_L, \mathcal{S}_R \gets$ \textsc{ChildSamples}($\mathcal{S}, \text{idx}, l$) \Comment{Split samples based on $(\bold{x}_i)_\text{idx} \leq l$} 
        \State test statistic $v$ = \textsc{SplittingCriterion($\mathcal{S}_L, \mathcal{S}_R$)} \Comment{Two-sample test of choice} 
        \State \textsc{AddNewSplitCandidate}($\mathcal{C}$, $v$, $\mathcal{S}_L, \mathcal{S}_R, \text{idx}, l$) 
    \EndFor
    \State $\mathcal{S}_L, \mathcal{S}_R, \text{idx}, l \gets$ \textsc{FindBestSplit}($\mathcal{C}$) 
    \State $\mathcal{N}_L \gets$  \textsc{CreateNode}($\mathcal{S}_L$) \Comment{Create new node with set of samples $\mathcal{S}_L$} 
    \State $\mathcal{N}_R \gets$ \textsc{CreateNode}($\mathcal{S}_R$) \Comment{Create new node with set of samples $\mathcal{S}_R$}
    \State \textsc{BuildTree}($\mathcal{N}_L$),\quad \textsc{BuildTree}($\mathcal{N}_R$) \Comment{Proceed building recursively}
    \State \textsc{Children}($\mathcal{N}) \gets \mathcal{N}_L, \mathcal{N}_R$ 
    \State \textsc{Split}($\mathcal{N}) \gets \text{idx}, l$ \Comment{Store the split} 
    \State \textbf{return}
\EndProcedure

\item[]

\Procedure{GetWeights}{forest $\mathcal{F}$, test point $\bold{x}$} \Comment{Computes the weighting function}
    \State vector of weights $w$ = \textsc{Zeros}($n$) \Comment{$n$ is the training set size}
    \For{$i = 1,\ldots,|\mathcal{F}|$}
        \State $\mathcal{L}$ = \textsc{GetLeafSamples}($\mathcal{T}_i, \bold{x}$) \Comment{indices of training samples in same leaf as $\bold{x}$}
        \For{$\text{idx} \in \mathcal{L}$}
            \State $w[\text{idx}] = w[\text{idx}]$ + $1/(|\mathcal{L}|\cdot|\mathcal{F}|)$
        \EndFor
    \EndFor
    \State \textbf{return} $w$
\EndProcedure
\end{algorithmic}
\end{algorithm}


\begin{itemize}
\item Every tree is constructed based on a random subset of size $s$ (taken to be $50\%$ of the size of the training set by default) of the training data set, similar to \citet{wager2018estimation}. This differs from the original Random Forest algorithm \citep{breiman2001random}, where the bootstrap subsampling is done by drawing from the original sample with replacement.

\item The principle of honesty \citep{biau2012analysis, denil2014narrowing, wager2018estimation} is used for building the trees (line 4), whereby for each tree one first performs the splitting based on one random set of data points $\mathcal{S}_\text{build}$, and then populates the leaves with a disjoint random set $\mathcal{S}_\text{populate}$ of data points for determining the weighting function $w_\bold{x}(\cdot)$. This prevents overfitting, since we do not assign weight to the data points which we used to built the tree.

\item We borrow the method for selecting the number of candidate splitting variables from the \texttt{grf} package \citep{athey2019generalized}. This number is randomly generated as $\min (\max (\text{Poisson}(\text{mtry}), 1) , p)$, where mtry is a tuning parameter. This differs from the original Random Forests algorithm, where the number of splitting candidates is fixed to be mtry.

\item The number of trees built is $N=2000$ by default.

\item The factor variables in both the responses and the predictors are encoded by using the one-hot encoding, where we add an additional indicator variable for each level $l$ of some factor variable $X_k$. This implies that in the building step, if we split on this indicator variable, we divide the current set of data points in the sets where $X_k = l$ and $X_k \neq l$. This works well if the number of levels is not too big, since otherwise one makes very uneven splits and the dimensionality of the problem increases significantly. Handling of categorical problems is a general challenge for the forest based methods and is an area of active research \citep{johannemann2019sufficient}. We will leave improving on this approach for the future development.

\item We try to enforce splits where each child has at least a fixed percentage (chosen to be $10\%$ as the default value) of the current number of data points. In this way we achieve balanced splits and reduce the computational time. However, we cannot enforce this if we are trying to split on the variable $X_i$ with only a few unique values, e.g.\ indicator variable for a level of some factor variable.

\item All components of the response $Y$ are scaled for the building step (but not when we populate the leaves). This ensures that each component of the response contributes equally to the kernel values, and consequently to the MMD two-sample test statistic. Plain usage of the MMD two-sample test would scale the components of $Y$ at each node. However, this approach favors always splitting on the same variables, even though their effect will diminish significantly after having split several times.

\item By default, in step 20 of the Algorithm \ref{pseudocode}, we use the MMD-based splitting criterion given by
$$\frac{1}{B}\sum_{k=1}^B \frac{|\mathcal{S}_L||\mathcal{S}_L|}{(|\mathcal{S}_L|+|\mathcal{S}_R|)^2} \left\lvert \frac{1}{|\mathcal{S}_L|}\sum_{(\bold{x}_i, \bold{y}_i) \in \mathcal{S}_L} \varphi_{\boldsymbol{\omega}_k}(\bold{y}_i) -  \frac{1}{|\mathcal{S}_R|}\sum_{(\bold{x}_i, \bold{y}_i) \in \mathcal{S}_R} \varphi_{\boldsymbol{\omega}_k}(\bold{y}_i) \right\rvert^2.$$
The Gaussian kernel $k(\bold{x}, \bold{y}) = \tfrac{1}{(\sqrt{2\pi}\sigma)^d}e^{\tfrac{-\norm*{x-y}^2_2}{2\sigma^2}}$ is used as the default choice, with the bandwidth $\sigma$ chosen as the median pairwise distance between all training responses $\{\bold{y}_i\}_{i=1}^n$, commonly referred to as the 'median heuristic' \citep{gretton2012optimal}. However the algorithm can be used with any choice of kernel, or in fact with any two-sample test.


\item The number $B$ of random Fourier features is fixed and taken to be $20$ by default. The performance of the trees empirically shows stability for large range of $B$. Smaller values of $B$ help making the trees more independent which could improve the performance. One could even use an adaptive strategy of choosing $B$, possibly increasing $B$ as the depth of the tree increases, but we decided to keep $B$ fixed for simplicity.

\item We compute variable importance similarly as for the original Random Forest algorithm \citep{breiman2001random, wright2015ranger}, by sequentially permuting each variable and investigating the decrease in the performance. However, since we target the full conditional distribution of the multivariate response, as our performance measure we use for every test point $(\bold{x}, \bold{y})$ the MMD distance between the estimated joint distribution $\hat{\P}(\bold{Y}\mid \bold{X}\myeq\bold{x})$, described by the DRF weights, and the point mass $\delta_{\bold{y}}$.

\end{itemize}

\section{Derivations and Proofs}
\label{appendix: proofs}


In this section we present proofs and further details for the results in Sections \ref{sec: MMD splitcrit} and \ref{sec: theory}, in the order in which they appear. Sections \ref{MMDintegral} and \ref{CART_equivalence} derive the splitting criterion presented in Equation \eqref{eq: MMD splitcrit} in the main text, with Section \ref{CART_equivalence} showing that the CART criterion can be written in an analogous way. Section \ref{consistencyproofsec} provides background and proves to the statements in Section \ref{sec: theory}.

\subsection{Expressing MMD test statistic as an integral in the feature space} \label{MMDintegral}
The biased MMD two-sample statistic is given as
\begin{align*}
&\mathcal{D}_{\text{MMD}}\left(\{\bold{u}_i\}_{i=1}^m, \{\bold{v}_i\}_{i=1}^n\right) \\
&=\frac{1}{m^2}\sum_{i,j} k(\bold{u}_i, \bold{u}_j) + \frac{1}{n^2}\sum_{i,j} k(\bold{v}_i, \bold{v}_j)
 - \frac{2}{mn}\sum_{i}\sum_{j} k(\bold{u}_i, \bold{v}_j) \\
& = \frac{1}{m^2}\sum_{i,j} k(\bold{u}_i, \bold{u}_j) + \frac{1}{n^2}\sum_{i,j} k(\bold{v}_i, \bold{v}_j)
 - \frac{1}{mn}\sum_{i}\sum_{j} k(\bold{u}_i, \bold{v}_j) - \frac{1}{mn}\sum_{i}\sum_{j} k(\bold{v}_j, \bold{u}_i).    
\end{align*}

Assume that the kernel $k$ is bounded and shift-invariant, then by Bochner's theorem there exist a measure $\nu$ such that $k$ can be written as
$k(\bold{x}, \bold{y}) = \int_{\R^d} e^{i\boldsymbol{\omega}^T(\bold{x}-\bold{y})}d\nu(\boldsymbol{\omega}).$ 

Let us write $\varphi_{\boldsymbol{\omega}}^U = \frac{1}{m}\sum_i e^{i\boldsymbol{\omega}^T\bold{u}_i}$ and $\varphi_{\boldsymbol{\omega}}^V = \frac{1}{n}\sum_i e^{i\boldsymbol{\omega}^T\bold{v}_i}$. We can now write $\mathcal{D}_\text{MMD}$ as
\begin{align*}
\mathcal{D}_{\text{MMD}}\left(\{\bold{u}_i\}_{i=1}^m, \{\bold{v}_i\}_{i=1}^n\right)
&= \int_{\R^d} \left(\varphi_{\boldsymbol{\omega}}^U \overline{\varphi_{\boldsymbol{\omega}}^U} + \varphi_{\boldsymbol{\omega}}^V \overline{\varphi_{\boldsymbol{\omega}}^V} - \varphi_{\boldsymbol{\omega}}^U \overline{\varphi_{\boldsymbol{\omega}}^V} - \varphi_{\boldsymbol{\omega}}^V \overline{\varphi_{\boldsymbol{\omega}}^U}\right)d\nu(\boldsymbol{\omega})\\
&= \int_{\R^d} \left\lvert\varphi_{\boldsymbol{\omega}}^U - \varphi_{\boldsymbol{\omega}}^V\right\rvert^2d\nu(\boldsymbol{\omega}) \\
&= \int_{\R^d}\left\lvert \frac{1}{m}\sum_{i=1}^m \varphi_{\boldsymbol{\omega}}(\bold{u}_i) -  \frac{1}{n}\sum_{i=1}^n \varphi_{\boldsymbol{\omega}}(\bold{v}_i) \right\rvert^2 d\nu(\boldsymbol{\omega}), 
\end{align*}
where $\varphi_{\boldsymbol{\omega}}(\bold{y}) = e^{i\boldsymbol{\omega}^T\bold{y}} \in \mathbb{C}$ are the corresponding Fourier features, which is what we wanted to show.

\subsection{Approximate kernel and its MMD}
When the kernel $k$ is bounded and shift invariant, we have seen that it can be written as $k(\bold{x}, \bold{y}) = \int_{\R^d} e^{i\boldsymbol{\omega}^T(\bold{x}-\bold{y})}d\nu(\boldsymbol{\omega}).$ This integral can be approximated by sampling from $\nu$: Let $\boldsymbol{\omega}_1, \ldots, \boldsymbol{\omega}_B \sim \nu$ be a random sample from the measure $\nu$. Then we can write 
$$k(\bold{x}, \bold{y}) = \int_{\R^d} e^{i\boldsymbol{\omega}^T(\bold{x}-\bold{y})}d\nu(\boldsymbol{\omega}) \approx \frac{1}{B}\sum_{b=1}^B e^{i\boldsymbol{\omega_b}^T(\bold{x}-\bold{y})} = \frac{1}{B} \langle \bold{\widetilde{\varphi}}(\bold{u}), \bold{\widetilde{\varphi}}(\bold{v})\rangle_{\mathbb{C}^B} \vcentcolon= \tilde{k}(\bold{u}, \bold{v}),$$
where $\bold{\widetilde{\varphi}}(\bold{u}) = (\varphi_{\boldsymbol{\omega}_1}(\bold{u}), 
\ldots, \varphi_{\boldsymbol{\omega}_B}(\bold{u}))^T$ is a random complex vector consisting of the Fourier features $\varphi_{\boldsymbol{\omega}}(\bold{u}) = e^{i\boldsymbol{\omega}^T\bold{u}} \in \mathbb{C}.$
The kernel $\tilde{k}$ is analogous to the kernel $k$, but where the measure $\nu$ is replaced by the empirical measure $\tilde{\nu} = \tfrac{1}{B}\sum_{b=1}^B \delta_{\boldsymbol{\omega_b}}$: 
$$\tilde{k}(\bold{x}, \bold{y}) = \int_{\R^d} e^{i\boldsymbol{\omega}^T(\bold{x}-\bold{y})}d\tilde{\nu}(\boldsymbol{\omega}).$$

Analogously as in the section \ref{MMDintegral}, we can now write the MMD for the kernel $\tilde{k}$ as:
\begin{align*}
    \mathcal{D}_{\text{MMD}(\tilde{k})} &= \int_{\R^d}\left\lvert \frac{1}{m}\sum_{i=1}^m \varphi_{\boldsymbol{\omega}}(\bold{u}_i) -  \frac{1}{n}\sum_{i=1}^n \varphi_{\boldsymbol{\omega}}(\bold{v}_i) \right\rvert^2 d\tilde{\nu}(\boldsymbol{\omega}) \\
    &= \frac{1}{B}\sum_{b=1}^B \left\lvert \frac{1}{m}\sum_{i=1}^m \varphi_{\boldsymbol{\omega}_b}(\bold{u}_i) -  \frac{1}{n}\sum_{i=1}^n \varphi_{\boldsymbol{\omega}_b}(\bold{v}_i) \right\rvert^2,
\end{align*}
which can also additionally be interpreted as the approximation of $\mathcal{D}_{\text{MMD}}$. Therefore, our splitting criterion is obtained as the MMD of the random approximate kernel $\tilde{k}$:
$$\frac{1}{B}\sum_{b=1}^B \frac{n_Ln_R}{n_P^2} \left\lvert \frac{1}{n_L}\sum_{\bold{x}_i \in C_L} \varphi_{\boldsymbol{\omega}_b}(\bold{y}_i) -  \frac{1}{n_R}\sum_{\bold{x}_i \in C_R} \varphi_{\boldsymbol{\omega}_b}(\bold{y}_i) \right\rvert^2.$$
The scaling factor $\frac{n_Ln_R}{n_P^2}$ occurs naturally and penalizes the increased variance of the sample MMD statistic when $n_L$ or $n_R$ are small: it appears when we rewrite the CART criterion in the related form, see section \ref{CART_equivalence}.

This representation of the MMD is the key why we use the approximate kernel $\tilde{k}$ instead of $k$. This splitting criterion can be computed in $\O(Bn_P)$ complexity, by updating the sums $\sum_{\bold{x}_i \in C_L} \varphi_{\boldsymbol{\omega}_k}(\bold{y}_i)$ and $\sum_{\bold{x}_i \in C_R} \varphi_{\boldsymbol{\omega}_k}(\bold{y}_i)$ in $\O(1)$ computations, whereas this is not possible for $\mathcal{D}_{\text{MMD}}$.

\subsection{CART criterion rewritten} \label{CART_equivalence}

Standard CART criterion used in Random Forests \citep{breiman2001random} is the following: we repeatedly choose to split the parent node $P$ of size $n_P$ in two children $C_L$ and $C_R$, of sizes $n_L$ and $n_R$ respectively, such that the expression
\begin{equation} \label{CART}
\frac{1}{n_P}\left(\sum_{i \in C_L} (Y_i - \overline{Y}_L)^2 + \sum_{i \in C_R} (Y_i - \overline{Y}_R)^2 \right)
\end{equation}
is minimized, where $\overline{Y}_L = \tfrac{1}{n_L}\sum_{i\in C_L} Y_i$ and $Y_R$ is defined similarly.

We now have $\overline{Y} = \tfrac{1}{n_P}\sum_{i \in P} Y_i = \frac{n_L}{n_P}\overline{Y}_L + \frac{n_R}{n_P}\overline{Y}_R$, which gives $\overline{Y} - \overline{Y}_L = \tfrac{n_R}{n_P}(\overline{Y}_R - \overline{Y}_L)$, so we can write 
\begin{gather*}
\sum_{i \in C_L} (Y_i - \overline{Y}_L)^2 = \sum_{i \in C_L} (Y_i - \overline{Y} + \overline{Y} - \overline{Y}_L)^2
= \sum_{i \in C_L} (Y_i - \overline{Y} + \frac{n_R}{n_P}(\overline{Y}_R - \overline{Y}_L))^2 \\ = \sum_{i \in C_L} (Y_i - \overline{Y})^2 + 2\frac{n_R}{n_P}(\overline{Y}_R - \overline{Y}_L)\sum_{i \in C_L}(Y_i - \overline{Y}) + \frac{n_Ln_R^2}{n_P^2}(\overline{Y}_R - \overline{Y}_L)^2\\
= \sum_{i \in C_L} (Y_i - \overline{Y})^2 + 2\frac{n_R}{n_P}(\overline{Y}_R - \overline{Y}_L)\cdot n_L(\overline{Y}_L - \overline{Y}) + \frac{n_Ln_R^2}{n_P^2}(\overline{Y}_R - \overline{Y}_L)^2 \\
= \sum_{i \in C_L} (Y_i - \overline{Y})^2 + 2\frac{n_Rn_L}{n_P}(\overline{Y}_R - \overline{Y}_L)\cdot \frac{n_R}{n_P}(\overline{Y}_L - \overline{Y}_R) + \frac{n_Ln_R^2}{n_P^2}(\overline{Y}_R - \overline{Y}_L)^2 \\
= \sum_{i \in C_L} (Y_i - \overline{Y})^2 - \frac{n_Ln_R^2}{n_P^2}(\overline{Y}_R - \overline{Y}_L)^2.
\end{gather*}

Similarly we obtain
$$\sum_{i \in C_R} (Y_i - \overline{Y}_R)^2 = \sum_{i \in C_R} (Y_i - \overline{Y} + \overline{Y} - \overline{Y}_R)^2
= \sum_{i \in C_R} (Y_i - \overline{Y})^2 - \frac{n_L^2n_R}{n_P^2}(\overline{Y}_R - \overline{Y}_L)^2,$$
which gives us that the CART criterion \eqref{CART} can be written as 
\begin{gather*}
\frac{1}{n_P}\left(\sum_{i \in C_L} (Y_i - \overline{Y})^2 - \frac{n_Ln_R^2}{n_P^2}(\overline{Y}_R - \overline{Y}_L)^2 + \sum_{i \in C_R} (Y_i - \overline{Y})^2 - \frac{n_L^2n_R}{n_P^2}(\overline{Y}_R - \overline{Y}_L)^2 \right)\\
= \frac{1}{n_P}\sum_{i\in P} (Y_i - \overline{Y})^2 - \frac{n_Ln_R}{n_P^2}(\overline{Y}_R - \overline{Y}_L)^2, 
\end{gather*}
since $n_L + n_R = n_P$. Since the first term depends only on the parent node and not on the chosen split, we conclude that minimizing the CART criterion \eqref{CART} is equivalent to maximizing the following expression
\begin{equation}
\frac{n_Ln_R}{n_P^2}(\overline{Y}_L - \overline{Y}_R)^2.
\end{equation}

This equivalent criterion can be interpreted as comparing the difference in the means of the resulting child nodes, i.e. we will choose the split such that the means in the child nodes are as heterogeneous as possible. The scaling factor $\frac{n_Ln_R}{n_P^2}$ appears naturally, penalizing uneven splits due to the increased variance of $\overline{Y}_L$ or $\overline{Y}_R$.



\subsection{Proofs for Section \ref{sec: theory}}\label{consistencyproofsec}

\paragraph{Preliminaries.}
We first set notation and define basic probabilistic concepts on the separable Hilbert space $(\mathcal{H}, \langle, \cdot, \rangle_{\mathcal{H}})$. We thereby mostly refer to \citet{hilbertspacebook} and \citet{pisier_2016}. The initial results derived here parallel some of the results derived in \citet{OurapproachtoCME}, but where derived independently. Let $\left(\Omega, \mathcal{A}, \mathbb{P}\right)$ be the underlying probability space.
Let $\left(\mathcal{H}, \langle\cdot,\cdot\rangle_{\mathcal{H}}\right)$ be the Hilbert space induced by the kernel $k$ and $\mu: \mathcal{M}_{b}(\R^d) \to \mathcal{H}$
be the embedding function of, $\mu(P) \in \mathcal{H}$ for all bounded signed Borel measures $P$ on $\R^d$. Throughout we assume that $k$ is \emph{bounded} and \emph{continuous} in its two arguments. Boundedness of $k$ ensures that $\mu$ is indeed defined on all of $\mathcal{M}_{b}(\R^d)$, while continuity of $k: \R^d \times \R^d \to \R$ ensures $\mathcal{H}$ is \emph{separable}. Thus measurability issues can be avoided, in particular, a map $\xi: (\Omega, \mathcal{A}) \to (\mathcal{H},\mathcal{B}(\mathcal{H}))$ is measurable iff $\langle \xi,f\rangle_{\mathcal{H}}$ is measurable for all $f \in \mathcal{H}$. Moreover, a quick check reveals that $\mu(P)$ is linear on $\mathcal{M}_{b}(\R^d)$. If $\E[\| \xi \|_{\mathcal{H}}] < \infty$, we define
\[
\E[\xi]:= \int_{\mathcal{H}} \xi d \P \in \mathcal{H},
\]
where the integral is meant in a Bochner sense. Separability and $\E[\| \xi \|_{\mathcal{H}}] < \infty$ mean this integral is well-defined and moreover
\[
F(\E[\xi])=\E[F(\xi)],
\]
for any continuous linear function $F:\mathcal{H} \to \R$.\footnote{Here and later $F(\xi)$ is meant to mean $F(\xi(\omega))$ for all $\omega \in \Omega$. } In particular, $\E[\langle \xi, f  \rangle_{\mathcal{H}}]= \langle \E[\xi], f \rangle_{\mathcal{H}}$ for all $f \in \mathcal{H}$. Define moreover for $q \geq 1$, and $\xi, \xi_1,\xi_2 \in \mathcal{L}^2(\Omega, \mathcal{A}, \mathcal{H})$,
\begin{align*}
    \mathcal{L}^q(\Omega, \mathcal{A}, \mathcal{H}) &= \{\xi:(\Omega, \mathcal{F}) \to (\mathcal{H},\mathcal{B}(\mathcal{H})) \text{ measurable, with } \E[\|\xi\|^q] < \infty]  \}\\
    \mathbb{L}^q(\Omega, \mathcal{A}, \mathcal{H}) &= \text{Set of equivalence classes in $\mathcal{L}^q(\Omega, \mathcal{A}, \mathcal{H})$}\\
    \Var(\xi)&:=\E[\|\xi - \E[\xi]\|^2]=\E[\|\xi\|^2] - \|\E[\xi]\|^2, \ \ \xi \in \mathcal{L}^2(\Omega, \mathcal{A}, \mathcal{H})\\
    \Cov(\xi_1,\xi_2)&= \E[ \langle \xi_1 - \E[\xi_1],\xi_2 - \E[\xi_2] \rangle_{\mathcal{H}}]=\E[\langle \xi_1 ,\xi_2  \rangle_{\mathcal{H}} ] - \langle \E[\xi_1], \E[\xi_2] \rangle_{\mathcal{H}}.
\end{align*}
It is well-known, that $(\mathbb{L}^q, \| \cdot \|_{\mathbb{L}^q(\mathcal{H})})$ is a Banach space, with 
\[
\| \xi \|_{\mathbb{L}^q(\mathcal{H})}= \E[\| \xi \|_{\mathcal{H}}^q]^{1/q}.
\]
We can then also define \emph{conditional} expectation. For a sub $\sigma-$ algebra $\mathcal{F} \subset \mathcal{A}$, $\xi \in \mathcal{L}^1(\Omega, \mathcal{A}, \mathcal{H}) $, $\E[\xi\mid \mathcal{F}]$ is the (a.s.) unique element such that
\begin{itemize}
    \item[(C1)]  $\E[\xi\mid \mathcal{F}] : (\Omega, \mathcal{F}) \to (\mathcal{H},\mathcal{B}(\mathcal{H}))$ is measurable and $\E[\xi\mid \mathcal{F}]\in   \mathbb{L}^1(\Omega, \mathcal{F}, \mathcal{H})$,
    \item[(C2)] $\E[ \xi \1_{F} ] = \E[ \E[\xi\mid \mathcal{F}]\1_{F} ]$ for all $F \in \mathcal{F}$.
\end{itemize} 
See e.g.\ \citet{UMEGAKI197049} or \citet[Chapter 1]{pisier_2016}.
(C2) in particular means that $\E[ \E[\xi\mid \mathcal{F}] ] = \E[ \E[\xi\mid \mathcal{F}]\1_{\Omega} ] = \E[\xi] $, since $\Omega \in \F$ for any $\sigma$-algebra. It can also be shown that $F(\E[\xi\mid \mathcal{F}])=\E[F(\xi)\mid\mathcal{F}]$ for all linear and continuous $F: \mathcal{H} \to \R$ and that $\|\E[\xi\mid \mathcal{F}]  \|_{\mathcal{H}}\leq \E[ \|\xi \|_{\mathcal{H}} \mid \mathcal{F}]$ \citep[Chapter 1]{pisier_2016}. Moreover, 
\begin{itemize}
    \item[(C3)] For $\xi \in \mathbb{L}^2(\Omega, \mathcal{A}, \mathcal{H})$, $\E[\xi \mid \mathcal{F}]$ is the orthogonal projection into $\mathbb{L}^2(\Omega, \mathcal{F}, \mathcal{H})$,
\end{itemize}
again we refer to \citep{UMEGAKI197049}. We note that, as with conditional expectation on $\R$, $\E[\xi\mid \mathcal{F}]$ is only defined uniquely a.s. As such all (in)equalitie statements hold only a.s. However, we will often not explicitly write this going forward.


We then define $\E[\xi\mid\mathbf{X}]=\E[\xi\mid \sigma(\mathbf{X})]$. The following Proposition shows that this notion is well-defined and some further properties of Hilbert space-valued conditional expectation, in addition to (C1) -- (C3):

\begin{proposition}\label{condexp1}
Let $\left(\mathcal{H}_1, \langle\cdot,\cdot\rangle_1\right)$, $\left(\mathcal{H}_2, \langle\cdot,\cdot\rangle_2\right)$ be two separable Hilbert spaces, $\mathbf{X}, \mathbf{X}_1, \mathbf{X}_2  \in \mathcal{L}^1(\Omega, \mathcal{A}, \mathcal{H}_1) $ and $\xi_1, \xi_2, \xi \in \mathcal{L}^1(\Omega, \mathcal{A}, \mathcal{H}_2)$.\footnote{We again note that all equalities technically only hold a.s.}
\begin{itemize}
    \item[(C4)] There exists a measurable function $h: (\mathcal{H}_1, \mathcal{B}(\mathcal{H}_1)) \to (\mathcal{H}_2, \mathcal{B}(\mathcal{H}_2)) $, such that $\E[\xi\mid\sigma(\mathbf{X})]=h(\mathbf{X})=\E[\xi\mid\mathbf{X}]$,
    \item[(C5)] If $\xi_1 \in \mathcal{L}^2(\Omega, \mathcal{A}, \mathcal{H}_1)$, $\xi_2 \in \mathcal{L}^2(\Omega, \sigma(\mathbf{X}), \mathcal{H}_1)$, then $\E[ \langle \xi_1, \xi_2 \rangle_{\mathcal{H}_1} \mid \mathbf{X} ] =  \langle \E[\xi_1 \mid \mathbf{X} ], \xi_2 \rangle_{\mathcal{H}_1}$,
        \item[(C6)] If $\mathbf{X}_2$ and $(\xi, \mathbf{X}_1)$ are independent, then $\E[\xi\mid \mathbf{X}_1,\mathbf{X}_2 ] =\E[\xi \mid \mathbf{X}_1 ]$, 
        \item[(C7)] $\E[ \E[\xi \mid \mathbf{X}_1, \mathbf{X}_2 ] \mid \mathbf{X}_1]= \E[ \E[\xi \mid \mathbf{X}_1 ] \mid \mathbf{X}_1, \mathbf{X}_2] = \E[\xi \mid \mathbf{X}_1]$.
\end{itemize} 
\end{proposition}

\begin{proof}
We will use the following fact in the proof: Under the assumption of separability, all relevant notions of measurability are the same, see e.g.\ \citet[Chapter 2, 7]{hilbertspacebook}. In particular, $\xi \in  \mathbb{L}^q(\Omega, \mathcal{F}, \mathcal{H}_2)$, for any $q \geq 1$, means that $\xi: (\Omega, \mathcal{F}) \to (\mathcal{H}_2, \mathcal{B}(\mathcal{H}_2))$ is measurable, which in turns means there exists a sequence of simple functions 
\begin{align}\label{simplefunc}
    f_n=\sum_{k=1}^{m_n} g_k \1_{A_{k}},
\end{align}
with $g_k \in \mathcal{H}_2$ and $A_{k} \in \mathcal{F}$ for all $k$ and such that $f_n \to \xi$ a.s. on $\mathcal{H}_2$ and even $\|f_n -\xi \|_{\mathbb{L}^q(\mathcal{H})} \to 0$, see e.g.\ \citet[Proposition 1.2]{pisier_2016}.


For (C4), we note that by (C1), $\E[\xi\mid\sigma(\mathbf{X})] \in  \mathbb{L}^1(\Omega, \sigma(\mathbf{X}), \mathcal{H}_2)$ and thus there exists a sequence of functions $f_n: (\Omega, \sigma(\mathbf{X})) \to (\mathcal{H}_2, \mathcal{B}(\mathcal{H}_2))$ of the form \eqref{simplefunc}, such that $f_n \to \E[\xi\mid\sigma(\mathbf{X})]$ a.s. on $\mathcal{H}_2$. Since $A_k \in \sigma(\mathbf{X})$, $A_k = \{\omega: \mathbf{X}(\omega) \in B_k\}$ for some $B_k \in \mathcal{B}(\mathcal{H}_1)$, we may transform $f_n$ from a function on $\Omega$ to a function on $\mathcal{H}_1$ into $\mathcal{H}_2$:
\[
f_n(\omega)=\sum_{k=1}^{m_n} g_k \1_{A_{k}}(\omega)= \sum_{k=1}^{m_n} g_k \1_{B_{k}}(\mathbf{X}(\omega)):=h_n(\mathbf{X}(\omega)).
\]
This defines a sequence of measurable functions $h_n: (\mathcal{H}_1, \mathcal{B}(\mathcal{H}_1)) \to (\mathcal{H}_2, \mathcal{B}(\mathcal{H}_2))$ with $h(\mathbf{X})=\lim_{n} h_n(\mathbf{X})=\E[\xi\mid\sigma(\mathbf{X})]$ a.s., proving the result.

We first show (C5) for simple functions and then extend this to $\mathcal{L}^2(\Omega, \sigma(\mathbf{X}), \mathcal{H}_1)$, using the fact at the beginning of the proof. Let thus $f_n$ be of the form \eqref{simplefunc}. Then for all $F \in \sigma(\mathbf{X})$,
\begin{align*}
    \E[ \langle \E[f_n \mid \mathbf{X}], \xi_2 \rangle_{\mathcal{H}} \1_F] &= \sum_{k=1}^{m_n} \E[   \langle \E[\1_{A_k} \mid \mathbf{X} ] g_k, \xi_2 \rangle_{\mathcal{H}} \1_F  ]\\
    &=\sum_{k=1}^{m_n} \E[ \E[\1_{A_k} \mid \mathbf{X} ]  \langle  g_k, \xi_2 \rangle_{\mathcal{H}} \1_F  ]\\
      &=\sum_{k=1}^{m_n} \E[ \E[\1_{A_k}\langle  g_k, \xi_2 \rangle_{\mathcal{H}} \1_F \mid \mathbf{X} ]    ]\\
       &= \E[\langle f_n, \xi_2 \rangle_{\mathcal{H}} \1_F    ],
\end{align*}
from the properties of real-valued conditional expectation. As additionally $\langle \E[f_n \mid \mathbf{X}], \xi_2 \rangle_{\mathcal{H}}$ is clearly $\sigma(\mathbf{X})$ measurable, (C1) and (C2) are met for this candidate. Since conditional expectation is (a.s.) uniquely defined by (C1) and (C2), (C5) holds true for the special case of simple functions. For general $\xi_1 \in \mathcal{L}^2(\Omega, \mathcal{A}, \mathcal{H}_1)$, let $f_n$ have $\|f_n -\xi_1 \|_{\mathbb{L}^2(\mathcal{H})} \to 0$. The goal is to show that
\begin{align}\label{convergencegoals}
    |  \E[ \langle \E[f_n \mid \mathbf{X}], \xi_2 \rangle_{\mathcal{H}} \1_F] -\E[ \langle \E[\xi_1 \mid \mathbf{X}], \xi_2 \rangle_{\mathcal{H}} \1_F]  | &\to 0, \\
    |  \E[\langle f_n, \xi_2 \rangle_{\mathcal{H}} \1_F    ] - \E[\langle \xi_1, \xi_2 \rangle_{\mathcal{H}} \1_F    ] | &\to 0.
\end{align}
We can bound both terms by the same quantity, using Cauchy–Schwarz:
\begin{align*}
    |  \E[ \langle \E[f_n \mid \mathbf{X}], \xi_2 \rangle_{\mathcal{H}} \1_F] -\E[ \langle \E[\xi_1 \mid \mathbf{X}], \xi_2 \rangle_{\mathcal{H}} \1_F]  |& = |  \E[ \langle \E[f_n  -\xi_1 \mid \mathbf{X}], \xi_2 \rangle_{\mathcal{H}} \1_F]  |\\
    &\leq   \E[  \E[\|f_n  -\xi_1 \|_{\mathcal{H}} \1_F \|\xi_2 \|_{\mathcal{H}} \mid \mathbf{X}]   ]  \\
    &= \E[ \|f_n  -\xi_1 \|_{\mathcal{H}} \1_F \|\xi_2 \|_{\mathcal{H}} ],  
\end{align*}
as the random variable $\1_F \|\xi_2 \|_{\mathcal{H}}$ is $\sigma(\mathbf{X})$ measurable by assumption and
\begin{align*}
|  \E[\langle f_n, \xi_2 \rangle_{\mathcal{H}} \1_F    ] - \E[\langle \xi_1, \xi_2 \rangle_{\mathcal{H}} \1_F    ] | &=|  \E[\langle f_n - \xi_1, \xi_2 \rangle_{\mathcal{H}} \1_F    ] |\\
&\leq  \E[\| f_n - \xi_1 \|_{\mathcal{H}} \|\xi_2\|_{\mathcal{H}} \1_F    ].
\end{align*}
The result thus follows from the Hölder inequality, 
\begin{align*}
    \E[\| f_n - \xi_1 \|_{\mathcal{H}} \|\xi_2\|_{\mathcal{H}} \1_F    ] \leq 
    \|f_n -\xi_1 \|_{\mathbb{L}^2(\mathcal{H})}  \cdot \|\xi_2\|_{\mathbb{L}^2(\mathcal{H})}  \to 0.
\end{align*}


Finally (C6) and (C7) is easily proven with the technique of ``scalarization'' \citep[Chapter 1]{pisier_2016}. We do the full argument for (C6), (C7) can be shown analogously. That is, we show that for any $F: \H_{2} \to \R$ linear and continuous,
\begin{align}\label{scalarization}
    F(\E[\xi\mid \mathbf{X}_1,\mathbf{X}_2 ](\omega)) = F(\E[\xi\mid \mathbf{X}_1 ](\omega)).  
\end{align}
for almost all $\omega$ and some representation of $\E[\xi\mid \mathbf{X}_1,\mathbf{X}_2 ]$. This then immediately implies the result. Now, using the property of real-valued conditional expectations
\begin{align*}
    F(\E[\xi\mid \mathbf{X}_1,\mathbf{X}_2 ])&=  \E[F(\xi)\mid \mathbf{X}_1,\mathbf{X}_2 ]=\E[F(\xi)\mid \mathbf{X}_1 ],
\end{align*}
since $F(\xi)$ is real-valued and independent of $\mathbf{X}_2$. Since $\E[F(\xi)\mid \mathbf{X}_1 ] = F(\E[\xi\mid \mathbf{X}_1 ] )$, we obtain \eqref{scalarization}.
\end{proof}

(C4) in particular allows to see $\E[\xi \mid \sigma(\mathbf{X})]$ as a function in $\mathbf{X}$ and thus justifies the notation $\E[\xi \mid \mathbf{X}]$ and all the subsequent derivations. We may also define \emph{conditional independence} through conditional expectation: With the notation of Proposition \ref{condexp1}, $\xi$ and $\mathbf{X}_1$ are conditionally independent given $\mathbf{X}_2$, if $\E[f(\xi) \mid \mathbf{X}_1, \mathbf{X}_2]=\E[f(\xi) \mid \mathbf{X}_1]$ for all $f:(\mathcal{H}_2, \mathcal{B}(\mathcal{H}_2)) \to (\R, \mathcal{B}(\R))$ bounded and measurable, see e.g., \citet[Proposition 2.3]{conditionalindep}. This leads to two further important properties:

\begin{proposition}\label{condexp2}
Let $\left(\mathcal{H}_1, \langle\cdot,\cdot\rangle_1\right)$, $\left(\mathcal{H}_2, \langle\cdot,\cdot\rangle_2\right)$ be two separable Hilbert spaces, $\mathbf{X}, \mathbf{X}_1, \mathbf{X}_2  \in \mathcal{L}^1(\Omega, \mathcal{A}, \mathcal{H}_1) $ and $\xi_1, \xi_2, \xi \in \mathcal{L}^1(\Omega, \mathcal{A}, \mathcal{H}_2)$.
\begin{itemize}
    \item[(C8)] If $\xi$ and $\mathbf{X}_2$ are conditionally independent given $\mathbf{X}_1$, then $\E[\xi \mid \mathbf{X}_1,\mathbf{X}_2 ] =\E[\xi \mid \mathbf{X}_1 ] $,
    \item[(C9)] If $\xi_1$, $\xi_2$ are conditionally independent given $\mathbf{X}$, $ \E[ \langle \xi_1, \xi_2 \rangle \mid \mathbf{X}  ]=\langle \E[ \xi_1 \mid \mathbf{X} ], \E[ \xi_2 \mid \mathbf{X} ] \rangle .$
\end{itemize} 
\end{proposition}

\begin{proof}
We prove (C8) again using the scalarization trick: For any $F: \mathcal{H}_2 \to \R$ continuous and linear, it holds that 
\[
F_n(f) := F(f) \1\{ |F(f)| \leq n \} \ \ \forall f \in \mathcal{H}_2,
\]
is a bounded and measurable function. Thus by assumption, 
\begin{align*}
    \E[F_n(\xi)\mid \mathbf{X}_1,\mathbf{X}_2 ]  = \E[F_n(\xi)\mid \mathbf{X}_1 ].
\end{align*}
We now show that $\E[F_n(\xi)\mid \mathbf{X}_1,\mathbf{X}_2 ] \to \E[F(\xi)\mid \mathbf{X}_1,\mathbf{X}_2 ]$ and $\E[F_n(\xi)\mid \mathbf{X}_1 ] \to \E[F(\xi)\mid \mathbf{X}_1 ]$ a.s. Let $\mathcal{F}$ stand for either $\sigma(\mathbf{X}_1,\mathbf{X}_2)$ or $\sigma(\mathbf{X}_1)$. Then, as $F(f)= F_n(f) + F(f) \1\{ |F(f)| > n \}$,
\begin{align*}
    | \E[F_n(\xi) \mid \mathcal{F}] - \E[F(\xi) \mid \mathcal{F} ]| &\leq  \E[| F_n(\xi) - F(\xi)| \mid \mathcal{F}]\\
    &=\E[| F(f)| \1\{ |F(f)| > n \} \mid \mathcal{F}].
\end{align*}
Now, since for all $n$, $| F(f) | \1\{ |F(f)| > n \} \geq 0$ and $\E[| F(f) | \1\{ |F(f)| > n \} \mid \mathcal{F}] \leq \E[| F(f)| \mid \mathcal{F}] < \infty$ a.s. an application of Fatou's Lemma for (real-valued) conditional expectation (see e.g., \citet[Problem 10.7]{dudley}) to $|F(f)| -| F(f)| \1\{ |F(f)| > n \}$ implies
\begin{align*}
    \limsup_{n} | \E[F_n(\xi) \mid \mathcal{F}] - \E[F(\xi) \mid \mathcal{F} ]| \leq \E[\limsup_{n} | F(f)| \1\{ |F(f)| > n \} \mid \mathcal{F}]=0.
\end{align*}
Thus, we have shown that for all $F: \mathcal{H}_2 \to \R$ continuous and linear, $F(\E[\xi \mid \mathbf{X}_1,\mathbf{X}_2 ])=F(\E[\xi \mid \mathbf{X}_1 ])$, proving the claim.

Finally, combining (C5), (C7) and (C8), we obtain for $\xi_1$, $\xi_2$ conditionally independent given $\mathbf{X}$ 
\begin{align*}
    \E[ \langle \xi_1, \xi_2 \rangle \mid \mathbf{X}  ] &=  \E[ \E[ \langle \xi_1, \xi_2 \rangle \mid \mathbf{X}, \xi_2 ] \mid \mathbf{X} ]\\
    &=  \E[  \langle \E[ \xi_1 \mid \mathbf{X}, \xi_2 ], \xi_2 \rangle  \mid \mathbf{X} ]\\
    &=  \E[  \langle \E[ \xi_1 \mid \mathbf{X} ], \xi_2 \rangle  \mid \mathbf{X} ]\\
    &=   \langle \E[ \xi_1 \mid \mathbf{X} ], \E[ \xi_2\mid \mathbf{X} ] \rangle .
\end{align*}
\end{proof}

Let for $\mathbf{x}\in \mathcal{X}$, $P_{\mathbf{x}}$ be the conditional distribution of $Y$ given $\mathbf{x}$ on $\R^d$ and similarly with $P_{\mathbf{X}}$ (i.e. the regular conditional probability measure). Note that $P_{\mathbf{x}} \in \mathcal{H}$, while $P_{\mathbf{X}}$ is a random element mapping into $\mathcal{H}$.

As in \citet{wager2018estimation}, we define for $f$, $g$ two functions, with $\liminf_{s \to \infty} g(s) > 0$, $f(s)= \mathcal{O}(g(s))$ if 
\[
\limsup_{s \to \infty} \frac{|f(s)|}{g(s)} \leq C,
\]
for some $C > 0$. If $C=1$, then we write $f(s) \precsim g(s)$. For a sequence of random variables $X_n: \Omega \to \R$, and $a_n \in (0,+\infty)$, $n \in \N$, we write as usual $X_n=\O_p(a_n)$, if 
\[
\lim_{M \to \infty} \sup_{n} \P(a_n^{-1} |X_n| > M) = 0,
\]
i.e. if $X_n$ is bounded in probability. We write $X_n=o_p(a_n)$, if $a_n^{-1} X_n$ converges in probability to zero. Similarly, for $(S,d)$ a separable metric space, $\mathbf{X}_n: (\Omega, \mathcal{A}) \to (S, \mathcal{B}(S))$, $n \in \N$ and $\mathbf{X}: (\Omega, \mathcal{A}) \to (S, \mathcal{B}(S))$ measurable, we write $\mathbf{X}_n \stackrel{p}{\to} \mathbf{X}$, if $d(\mathbf{X}_n, \mathbf{X})=o_{p}(1)$.

Finally let $\mathbf{X} \in \mathcal{L}^2(\Omega, \mathcal{A}, \mathcal{H}_1)$, $\xi \in \mathcal{L}^2(\Omega, \mathcal{A}, \mathcal{H}_2)$ and assume that $A \subset \Omega$ depends on $\mathbf{X}$, $A=A(\mathbf{X})$. Thus for $\mathbf{X}$ fixed to a certain value, $A$ is a fixed set. If $ \P(A \mid \mathbf{X}) > 0$ almost everywhere, we define
\[
\E[\xi \mid  A ]=\E[\xi \mid  \mathbf{X}, A ] := \frac{\E[\xi \1_A \mid \mathbf{X} ]}{\P( A \mid \mathbf{X})} \in \mathcal{L}^2(\Omega, \sigma(\mathbf{X}), \mathcal{H}_2).
\]
It then holds by construction that
\begin{align}
    \E[\xi \1_A \mid \mathbf{X} ] = \E[\xi \mid  \mathbf{X}, A ] \cdot \P( A \mid \mathbf{X}).
\end{align}

Let again $\mu(\mathbf{x}):=\mu(P_{\mathbf{x}})$ be the embedding of the true conditional distribution into $\mathcal{H}$. We first state 3 preliminary results:

\begin{lemma}\label{equivalence} It holds that
$\E[\mu(\delta_{\mathbf{Y}}) \mid\mathbf{X} \myeq \mathbf{x}]= \mu(P_{\mathbf{x}}).$
\end{lemma}

\begin{proof}
We first note that $P_{\mathbf{x}}$ exists and is a probability measure on $\R^d$. Since $k$ is bounded $\mu(P_{\mathbf{x}}) \in \mathcal{H}$ exists and is uniquely defined by the relation
\[
\langle f, \mu(P_{\mathbf{x}}) \rangle_{\mathcal{H}} = \E[f(\mathbf{Y}) \mid\mathbf{X} \myeq \mathbf{x}] \ \ \ \forall f \in \mathcal{H}.
\]
On the other hand, by the above $\E[\mu(\delta_{\mathbf{Y}}) \mid\mathbf{X}] \in   \mathbb{L}^2(\Omega, \mathcal{A}, \mathcal{H}) $ exists and for all $f \in \mathcal{H}$,
\[
\langle \E[\mu(\delta_{\mathbf{Y}}) \mid\mathbf{X}], f \rangle_{\mathcal{H}} = \E[\langle \mu(\delta_{\mathbf{Y}}), f \rangle_{\mathcal{H}} \mid\mathbf{X}]=\E[f(\mathbf{Y}) \mid \mathbf{X}],
\]
since $F: H \to \R$, $F(g)=\langle g, f \rangle_{\mathcal{H}}$ defines a continuous linear function. In particular, for all $f \in \mathcal{H}$,
\[
\langle \E[\mu(\delta_{\mathbf{Y}}) \mid \mathbf{X} \myeq \mathbf{x}], f \rangle_{\mathcal{H}}= \E[f(\mathbf{Y}) \mid \mathbf{X} \myeq \mathbf{x}],
\]
or $\E[\mu(\delta_{\mathbf{Y}}) \mid \mathbf{X} \myeq \mathbf{x}]=\mu(P_{\mathbf{x}})$.
\end{proof}

Since $\mu(\delta_{\mathbf{Y}})=k(\mathbf{Y}, \cdot)$, Lemma \ref{equivalence} in fact corresponds to Lemma 3.2 in \citet{OurapproachtoCME}.

For a more compact notation in the following Lemma, let $N=\{1,\ldots,n\}$ and let for $A \subset N$ and $k \leq |A|$, $C_{k}(A)$ be the set of all subsets of size $k$ drawn from $A$ without replacement, with $C_0:=\emptyset$.

\begin{lemma}[H-Decomposition of a Hilbert-space valued Kernel]\label{Hdecomposition}

Let $\left(\mathcal{H}_1, \langle\cdot,\cdot\rangle_1\right)$, $\left(\mathcal{H}_2, \langle\cdot,\cdot\rangle_2\right)$ be two separable Hilbert spaces, $\mathbf{X}_1, \ldots, \mathbf{X}_n$ be i.i.d. copies of a random element $\mathbf{X}: (\Omega, \A) \to (\mathcal{H}_1, \mathcal{B}(\mathcal{H}_1))$. Write $\mathcal{X}_n=(\mathbf{X}_1, \ldots, \mathbf{X}_n)$ and let $T:(\mathcal{H}_1^n, \mathcal{B}(\mathcal{H}_1^n)) \to  (\mathcal{H}_2, \mathcal{B}(\mathcal{H}_2))$ measurable with 
$\E[\| T(\mathcal{X}_n) \|^2_{\mathcal{H}_2} ] < \infty$. If $T$ is symmetric, there exists functions $T_j$, $j=1,\ldots,n$, such that
\begin{equation}\label{ANOVA}
    T(\mathcal{X}_n) = \E[T(\mathbf{X})] + \sum_{i=1}^n T_1(\mathbf{X}_i)  + \sum_{i_1 < i_2} T_2(\mathbf{X}_{i_1}, \mathbf{X}_{i_2}) + \ldots T_n(\mathcal{X}_n),
\end{equation}
and it holds that
\begin{align}\label{vardecomp}
    \Var(T(\mathcal{X}_n))=\sum_{i=1}^n \binom{n}{i} \Var(T_i(\mathbf{X}_1, \ldots, \mathbf{X}_i) ),
\end{align}
and
\[
T_{1}(\mathbf{X}_i)=\E[T(\mathcal{X}_n)\mid\mathbf{X}_i] - \E[T(\mathcal{X}_n)].
\]

\end{lemma}

\begin{proof}
Composition \eqref{ANOVA} was proven in a range of different ways for real-valued $T$, see e.g.\ \citet{originalhoeffding}, \citet{ANOVAproof2}, \citet{efron1981}, or \citet{ANOVAproof3}. We consider and slightly extend the elegant proof of \citet{ANOVAproof2} to also prove \eqref{vardecomp}. See also \citet{ANOVAproof3}. Let
\begin{align*}
    T_{1}(\mathbf{X}_i)&=\E[T(\mathcal{X}_n)\mid \mathbf{X}_i] - \E[T(\mathcal{X}_n)]\\
    T_{2}(\mathbf{X}_i, \mathbf{X}_j)&= \E[T(\mathcal{X}_n)\mid \mathbf{X}_i, \mathbf{X}_j] - \E[T(\mathcal{X}_n)\mid \mathbf{X}_i] - \E[T(\mathcal{X}_n)\mid \mathbf{X}_j]  + \E[T(\mathcal{X}_n)]\\
        & \vdots \\
    T_{\ell}(\pi_{NA_{\ell}}(\mathcal{X}_n)) &=  \sum_{k=0}^{\ell-1}  (-1)^{\ell-k} \sum_{B \in C_{k}(A_{\ell})} \E[ T(\mathcal{X}_n) \mid \pi_{NB}(\mathcal{X}_n)]. 
\end{align*}
We note that $T_{\ell}$ does not depend on the exact indices in $A_{\ell}$ thanks to the assumed symmetry of $T$. Since $\E[T(\mathcal{X}_n)\mid \mathbf{X}_1,\ldots,\mathbf{X}_n]=T(\mathcal{X}_n)$ is part of $T_n(\mathcal{X}_n)$, this leads to a telescoping sum, already proving \eqref{ANOVA}.  

Adapting the approach of \citet{ANOVAproof2}, consider now $\mathcal{L}^2(\Omega, \sigma(\mathcal{X}_n), \mathcal{H})$ and let $Q_i$ be the projection operator into $\mathcal{L}^2(\Omega, \sigma(\mathbf{X}_1, \ldots, \mathbf{X}_{i-1}, \mathbf{X}_{i+1}, \ldots,  \mathbf{X}_n ), \mathcal{H}).$
That is 
$$(Q_i T)(\mathbf{X}_1, \ldots, \mathbf{X}_{i-1}, \mathbf{X}_{i+1}, \ldots,  \mathbf{X}_n ) = \E[T(\mathcal{X}_n)\mid \mathbf{X}_1, \ldots, \mathbf{X}_{i-1}, \mathbf{X}_{i+1}, \ldots,  \mathbf{X}_n],$$
by (C3). Now as in \citet{ANOVAproof3}, 
\begin{itemize}
    \item[(I)] the $Q_i$ commute,
    \item[(II)] For $A_{\ell}=\{i_1, \ldots, i_{\ell} \} \subset N$, 
    $$(Q_{i_1}\cdots Q_{i_{\ell}} T ) (\pi_{NA_{\ell}^c}(\mathcal{X}_n)) = \E[ T(\mathcal{X}_n) \mid \pi_{NA_{\ell}^c}(\mathcal{X}_n)].$$ 
    In particular, 
    $$(Q_{1}\cdots Q_{\ell} T ) (\mathbf{X}_{\ell+1}, \ldots, \mathbf{X}_{n}) = \E[ T(\mathcal{X}_n)\mid \mathbf{X}_{\ell+1}, \ldots, \mathbf{X}_{n}].$$
\end{itemize}
Moreover, it holds that
\begin{align*}
T_{\ell}(\mathbf{X}_{i_1}, \ldots, \mathbf{X}_{i_{\ell}} )  = ([I-Q_{i_1}][I-Q_{i_2}]\ldots [I-Q_{i_{\ell}}] Q_{i_{\ell} + 1} \cdots Q_{i_n}T)(\mathbf{X}_{i_1}, \ldots, \mathbf{X}_{i_{\ell}}).
\end{align*}
Expanding the identity
\begin{align*}
   T(\mathcal{X}_n) = (I^n  T)(\mathcal{X}_n)= ([(I-Q_1) + Q_1] [(I-Q_2) + Q_2] \cdots [(I-Q_n) + Q_n]  T )(\mathcal{X}_n),
\end{align*}
as in \citet{ANOVAproof3} and using $Q_i (1-Q_i)=0$, proves \eqref{ANOVA}.
Furthermore the above implies that for any subset $A_l \subset N$ that intersects with $A_{\ell}=\{i_1, \ldots, i_{\ell}\}$, we must have $\E[T_{\ell}(\mathbf{X}_{i_1}, \ldots, \mathbf{X}_{i_{\ell}} ) \mid \pi_{NA_{l}}(\mathcal{X}_n)]=0$. Indeed assume $A_l \cap A_{\ell}=\{i_2, \ldots, i_{\ell}\}$, then since the elements of $\mathcal{X}_n$ are independent, by (C6),
\[
\E[T_{\ell}(\pi_{NA_{\ell}}(\mathcal{X}_n) )| \pi_{NA_{l}}(\mathcal{X}_n)] = \E[T_{\ell}(\pi_{NA_{\ell}}(\mathcal{X}_n) ) \mid \pi_{N (A_{l} \cap A_{\ell} )}(\mathcal{X}_n)],
\]
i.e. all elements outside the intersection are irrelevant. Moreover,
\begin{align*}
  &\E[T_{\ell}(\pi_{NA_{\ell}}(\mathcal{X}_n) ) \mid \pi_{N (A_{l} \cap A_{\ell} )}(\mathcal{X}_n)]  \\
  &=(Q_{i_1} [I-Q_{i_1}][I-Q_{i_2}]\ldots [I-Q_{i_{\ell}}] Q_{i_{\ell} + 1} \cdots Q_{i_n}T) (\pi_{N (A_{l} \cap A_{\ell} )}(\mathcal{X}_n))
\end{align*}
and clearly this projection can only be $0$. The same argument can be made for any other intersection set $A_l \cap A_{\ell}$, even if it is the empty set.

Thus combining the above with (C5), for $A_{\ell} \neq A_{l}$, it holds that
\begin{align*}
    \E[ \langle T_{\ell}(\pi_{NA_{\ell}}(\mathcal{X}_n)), T_{l}(\pi_{NA_{l}}(\mathcal{X}_n)) \rangle_{\mathcal{H}} ] &= \E[ \E[ \langle T_{\ell}(\pi_{NA_{\ell}}(\mathcal{X}_n)), T_{l}(\pi_{NA_{l}}(\mathcal{X}_n))\rangle_{\mathcal{H}} \mid \pi_{NA_{l}}(\mathcal{X}_n)]  ] \\
    &= \E[  \langle \E[ T_{\ell}(\pi_{NA_{\ell}}(\mathcal{X}_n)) \mid \pi_{NA_{l}}(\mathcal{X}_n)], T_{l}(\pi_{NA_{l}}(\mathcal{X}_n))\rangle_{\mathcal{H}}   ]\\
    &=0.
\end{align*}
In other words, the covariance between any two elements in the decomposition of $ T(\mathcal{X}_n) - \E[T(\mathbf{X})] $ in \eqref{ANOVA} is uncorrelated. Finally then
\begin{align*}
    \Var(T(\mathbf{X}_1,\ldots, \mathbf{X}_n))&=\E[ \langle T(\mathcal{X}_n) - \E[T(\mathcal{X}_n)] , T(\mathcal{X}_n) - \E[T(\mathcal{X}_n)] \rangle  ] \nonumber \\
    &=\sum_{i=1}^n \binom{n}{i} \Var(T_i(\mathbf{X}_1, \ldots, \mathbf{X}_i) ).
\end{align*}

\end{proof}

\noindent
In order to proceed, we first prove Theorem \ref{thm: MMD_CART_equivalence}, which is somewhat separate from the remainder of this section:

\MMDCART*

\begin{proof}
The first part of the Theorem is shown analogously to the proof in Section \ref{CART_equivalence} of the appendix, but where we replace the standard dot product in $\R$ with the inner product $\langle\ \cdot , \cdot \rangle_{\mathcal{H}}$ associated with $(\mathcal{H}, k)$ and use the induced RKHS norm $\norm*{\cdot}_\mathcal{H}$.
Also, since $k$ is bounded, the embedding $\mu(\mathcal{D})$ into RKHS $\mathcal{H}$ exists for any distribution $\mathcal{D}$, so everything is well-defined.

For the second statement of the Theorem 1, note that $n_P \sim \mbox{Binomial}(\pi,n)$, where $\pi :=\P(\mathbf{X} \in P) > 0$. Let $\hat{P}_{\mathbf{x}} = \hat{\P}(\bold{Y}\mid \bold{X}=\bold{x})$ be a fixed conditional distribution estimator and recall that $P_{\mathbf{x}} = \P(\bold{Y}\mid \bold{X}=\bold{x})$. We now write
\begin{equation*}
    \sum_{\bold{x}_i \in P} \norm*{\mu(\delta_{\boldsymbol{y}_i}) - \mu(\hat{P}_{\mathbf{x}_i})}_\mathcal{H}^2  = \sum_{i=1}^n \norm*{\mu(\delta_{\boldsymbol{y}_i}) - \mu(\hat{P}_{\mathbf{x}_i})}_\mathcal{H}^2 \1\{\bold{x}_i \in P\}.
\end{equation*}
Then it holds that
\begin{gather*}
    \E \left[ \norm*{\mu(\delta_{\boldsymbol{Y}_i}) - \mu(\hat{P}_{\mathbf{X}_i})}_\mathcal{H}^2 \1\{\bold{X}_i \in P\} \right] 
= \E \left[ \E \left[ \norm*{\mu(\delta_{\boldsymbol{Y}})  - \mu(\hat{P}_{\mathbf{X}})}_\mathcal{H}^2\mid  \mathbf{X} \right] \1{\{\bold{X} \in P\}}\right],
\end{gather*}
and
\begin{align*}
    \E\left[\norm*{\mu(\delta_{\boldsymbol{Y}})  - \mu(\hat{P}_{\mathbf{X}})}_\mathcal{H}^2\mid \mathbf{X} \right] 
    &= \E\left[\E\norm*{\mu(\delta_{\boldsymbol{Y}}) - \mu(P_{\mathbf{X}})}_\mathcal{H}^2 + \E \norm*{\mu(P_{\mathbf{X}}) - \mu(\hat{P}_{\mathbf{X}})}_\mathcal{H}^2 \mid  \mathbf{X} \right]\\
    &\qquad + 2\E\left[\langle \mu(\delta_{\boldsymbol{Y}}) - \mu(P_{\mathbf{X}}) ,\, \mu(P_{\mathbf{X}}) - \mu(\hat{P}_{\mathbf{X}})\rangle_{\mathcal{H}} \mid  \mathbf{X} \right].
\end{align*}
It follows with Lemma \ref{equivalence} and (C5) that,
\begin{align*}
  &\E\left[\langle \mu(\delta_{\boldsymbol{Y}}) - \mu(P_{\mathbf{X}}) ,\, \mu(P_{\mathbf{X}}) - \mu(\hat{P}_{\mathbf{X}})\rangle_{\mathcal{H}} \mid  \mathbf{X}  \right] =    \\
  &\langle \E\left[\mu(\delta_{\boldsymbol{Y}}) - \mu(P_{\mathbf{X}}) \mid  \mathbf{X}  \right] ,\, \mu(P_{\mathbf{X}}) - \mu(\hat{P}_{\mathbf{X}})\rangle_{\mathcal{H}}  =0.
\end{align*}

Combining the three equations, this means
\begin{align}\label{expectationrewriting}
    &\E \left[ \norm*{\mu(\delta_{\boldsymbol{Y}_i}) - \mu(\hat{P}_{\mathbf{X}_i})}_\mathcal{H}^2 \1\{\bold{X}_i \in P\} \right] \nonumber \\
    &= \E\left[\norm*{\mu(\delta_{\boldsymbol{Y}}) - \mu(P_{\mathbf{X}})}_\mathcal{H}^2 \1\{\bold{X} \in P\} +  \norm*{\mu(P_{\mathbf{X}}) - \mu(\hat{P}_{\mathbf{X}})}_\mathcal{H}^2 \1\{\bold{X} \in P\} \right] \nonumber \\
    &=\pi \left( V_P + \E\left[\norm*{\mu(P_{\mathbf{X}}) - \mu(\hat{P}_{\mathbf{X}})}_\mathcal{H}^2 \mid  \bold{X} \in P  \right]  \right),
\end{align}
using that $\E[g(\mathbf{X})\mid  \bold{X} \in P ]= \E[g(\mathbf{X})\1{\{\bold{X} \in P\}}]/\P( \bold{X} \in P)$. We now show that the difference between $\frac{1}{n_P}\sum_{\bold{x}_i \in P} \norm*{\mu(\delta_{\bold{y}_i}) - \mu(\hat{\P}(\bold{Y}\mid\bold{X}\myeq\bold{x}_i))}_\mathcal{H}^2$ and the expectation on the left of Equation \eqref{expectationrewriting} is $\mathcal{O}_p(n^{-1/2})$, using standard CLT arguments. Define $K=\sup_{z,z'} \mid k(z,z')\mid  < \infty$, as we have assumed that $k$ is bounded. For any two distributions $\mathcal{D}_1$, $\mathcal{D}_2$ we now obtain
\begin{align*}
    \norm*{\mu(\mathcal{D}_1) - \mu(\mathcal{D}_2)}_\mathcal{H}^2 = \E[k(\mathbf{Z}_1,\mathbf{Z}_1')] - 2\E[k(\mathbf{Z}_1,\mathbf{Z}_2)] + \E[k(\mathbf{Z}_2,\mathbf{Z}_2')] \leq  4 K,
\end{align*}
where $\mathbf{Z}_1,\mathbf{Z}_1' \sim \mathcal{D}_1$ and $\mathbf{Z}_2,\mathbf{Z}_2' \sim \mathcal{D}_2$ are independent random variables. Thus,
\begin{align*}
    \E\left[ \norm*{\mu(\delta_{\boldsymbol{Y}}) - \mu(\hat{P}_{\mathbf{X}})}_\mathcal{H}^2 \right] \leq 4K, \hspace{ 0.5 cm} \E\left[ \norm*{\mu(\delta_{\boldsymbol{Y}}) - \mu(\hat{P}_{\mathbf{X}})}_\mathcal{H}^4 \right] \leq 16K^2,
\end{align*}
implying that both first and second moments of the random variable $\norm*{\mu(\delta_{\boldsymbol{Y}}) - \mu(\hat{P}_{\mathbf{X}})}_\mathcal{H}^2 \1{\{\bold{X} \in P\}}$ are finite. Moreover, since $\norm*{\mu(\delta_{\boldsymbol{y}_i}) - \mu(\hat{P}_{\mathbf{x}_i})}_\mathcal{H}^2 \1{\{\bold{x}_i \in P\}}$ for $i=1,\ldots,n$ are its i.i.d. realizations, it follows directly from the CLT that:
\begin{align*}
    &\sqrt{n}  \left( \frac{1}{n} \sum_{i=1}^{n} \norm*{\mu(\delta_{\boldsymbol{y}_i}) - \mu(\hat{P}_{\mathbf{x}_i})}_\mathcal{H}^2 \1{\{\bold{x}_i \in P\}} - \E \left[ \norm*{\mu(\delta_{\boldsymbol{Y}_i}) - \mu(\hat{P}_{\mathbf{X}_i})}_\mathcal{H}^2 \1\{\bold{X}_i \in P\} \right] \right)   \\
    &=   \O_{p}\left(1\right).
\end{align*}

By multiplying the above equation with $n/n_P=(1/\pi+o_{p}(1))=\O_p(1)$, it also holds that
\begin{align}\label{Osqrtnresult1}
    \sqrt{n}  \left( \frac{1}{n_P} \sum_{\bold{x}_i \in P} \norm*{\mu(\delta_{\boldsymbol{y}_i}) - \mu(\hat{P}_{\mathbf{x}_i})}_\mathcal{H}^2 - \frac{n }{n_P}  \E \left[ \norm*{\mu(\delta_{\boldsymbol{Y}_i}) - \mu(\hat{P}_{\mathbf{X}_i})}_\mathcal{H}^2 \1\{\bold{X}_i \in P\} \right] \right)   =   \O_{p}\left(1\right).
\end{align}
Thus,
\begin{align}\label{nonameeq}
     &\sqrt{n}  \left( \frac{1}{n_P} \sum_{\bold{x}_i \in P} \norm*{\mu(\delta_{\boldsymbol{y}_i}) - \mu(\hat{P}_{\mathbf{x}_i})}_\mathcal{H}^2 -  \frac{1}{\pi}  \E \left[ \norm*{\mu(\delta_{\boldsymbol{Y}_i}) - \mu(\hat{P}_{\mathbf{X}_i})}_\mathcal{H}^2 \1\{\bold{X}_i \in P\} \right] \right) \nonumber \\ 
    &= \sqrt{n}  \left( \frac{1}{n_P} \sum_{\bold{x}_i \in P} \norm*{\mu(\delta_{\boldsymbol{y}_i}) - \mu(\hat{P}_{\mathbf{x}_i})}_\mathcal{H}^2 - \frac{n}{n_P} \E \left[ \norm*{\mu(\delta_{\boldsymbol{Y}_i}) - \mu(\hat{P}_{\mathbf{X}_i})}_\mathcal{H}^2 \1\{\bold{X}_i \in P\} \right] \right) \nonumber \\ 
     &\qquad - \sqrt{n}\left( 1- \frac{n \pi}{n_P}  \right) \frac{1}{\pi} \E \left[ \norm*{\mu(\delta_{\boldsymbol{Y}_i}) - \mu(\hat{P}_{\mathbf{X}_i})}_\mathcal{H}^2 \1\{\bold{X}_i \in P\} \right].
\end{align}
Now both terms in \eqref{nonameeq} are $\O_{p}(1)$: For the first term this follows from \eqref{Osqrtnresult1}. For the second term it holds, since $\E\left[ \norm*{\mu(\delta_{\boldsymbol{Y}}) - \mu(\hat{P}_{\mathbf{X}})}_\mathcal{H}^2 \right] \leq 4K$ and 
\[  
\sqrt{n}\left(1 - \frac{n \pi}{n_P}  \right) = \O_p(1),
\]
which in turn is true by another application of the CLT on random variables $\1\{\mathbf{X}_i\in P\}$ and the fact that $n/n_P=\O_p(1)$:
\[
\sqrt{n}\left(1 - \frac{n \pi}{n_P} \right) = \sqrt{n}\frac{n_P - n \pi}{n_P} = \frac{n}{n_P} \frac{n_P - n \pi }{\sqrt{n}}=\O_p(1).
\]
Combining the fact that the expression in \eqref{nonameeq} is $\O_p(1)$ with \eqref{expectationrewriting} gives the result.
\end{proof}

Let $\hat{\mu}_n(\mathbf{x})$ be defined as in \eqref{finalestimator}:
\begin{align}\label{finalestimator2}
    \hat{\mu}_n(\mathbf{x}) = \binom{n}{s_n}^{-1}  \sum_{i_1 < i_2 < \ldots < i_{s_n}} \E_{\varepsilon} \left[ T(\mathbf{x}, \varepsilon; \mathbf{Z}_{i_1}, \ldots, \mathbf{Z}_{i_{s_n}}) \right],
\end{align}
where the sum is taken over all $\binom{n}{s_n}$ possible subsamples $\mathbf{Z}_{i_1}, \ldots, \mathbf{Z}_{i_{s_n}}$ of $\mathbf{Z}_{1}, \ldots \mathbf{Z}_{n}$ and $s_n \to \infty$ with $n$. Moreover, a single tree is given as
\begin{align}\label{tree}
    T(\mathbf{x}; \epsilon_k, \mathbf{Z}_{k_1}, \ldots \mathbf{Z}_{k_{s_n}} )=\sum_{j=1}^{s_n} \frac{\1(\mathbf{X}_{k_j} \in \mathcal{L}_{k}(\mathbf{x}))}{|\mathcal{L}_{k}(\mathbf{x})|} \mu(\delta_{\mathbf{Y}_{k_j}}). 
\end{align}
Given these preliminary results, the proofs we use are for the most part analogous to the ones in \citet{wager2018estimation}. Thus the proofs are given mostly for completeness and sometimes omitted altogether. We introduce the following additional notation, similar to Section \ref{sec: theory}: Let $\mathcal{Z}_s=\left(\mathbf{Z}_{1}, \ldots,\mathbf{Z}_{s}  \right)$ collect $s$ i.i.d. copies of $\mathbf{Z}$ and define for $j=1,\ldots, s_n$,
\begin{align*}
    \Var(T)&=\Var(T(\mathbf{x}, \varepsilon; \mathcal{Z}_{s_n}))\\
    \Var(T_j)&=\Var(\E[T(\mathbf{x}, \varepsilon; \mathcal{Z}_{s_n}) | \mathbf{Z}_1, \ldots, \mathbf{Z}_j ]).
\end{align*}
The index $s$ will take the role of $s_n$ or $n$, depending on the situation. We note that, due to i.i.d. sampling, it doesn't matter for variance or expectation what kind of subset $\mathbf{Z}_{i_1}, \ldots,\mathbf{Z}_{i_s}$ we are considering. In particular, we might just take $\mathcal{Z}_s$ each time.

Before going on to the main proofs, we repeat here the assumed properties of the trees for completeness:

\begin{itemize}
    \item[\textbf{(P1)}] (\textit{Data sampling}) The bootstrap sampling with replacement, usually used in forest-based methods, is replaced by a subsampling step, where for each tree we choose a random subset of size $s_n$ out of $n$ training data points. We consider $s_n$ going to infinity with $n$, with the rate specified below. 
    \item[\textbf{(P2)}] (\textit{Honesty}) An observation $\mathbf{Z}=(\mathbf{X}, \mu(\delta_{\mathbf{Y}}))$ is either used to place the splits in a tree or to estimate the response, but never both.
    \item[\textbf{(P3)}] (\textit{$\alpha$-regularity}) Each split leaves at least a fraction $\alpha \leq 0.2$ of the available training sample on each side. Moreover, the trees are grown until every leaf contains between $\kappa$ and $2\kappa - 1$ observations, for some fixed tuning parameter $\kappa \in \N$. 
    \item[\textbf{(P4)}] (\textit{Symmetry}) The (randomized) output of a tree does not depend on the ordering of the training samples.
    \item[\textbf{(P5)}] (\textit{Random-split}) At every split point, the probability that the split occurs along the feature $X_j$ is bounded below by $\pi/p$, for some $\pi > 0$ and for all $j=1,\ldots,p$.
\end{itemize}

Note that assumption \textbf{(P2)} differentiates from assumption \textbf{(P2)} in the main text. We refer to it as 
\begin{itemize}
    \item[\textbf{(P2')}] (\textit{Double Sampling}) The data used for constructing each tree is split into two parts; the first is used for determining the splits and the second for populating the leaves and thus for estimating the response.
\end{itemize}

$\textbf{(P2)}$ will allow us to assume that all $s_n$ observations of the tree are used to estimate the response, as in \eqref{tree} and the trees are built with some auxiliary data. This is done for simplicity of exposition, the results can be extended to hold in case of \textbf{(P2')} as well. In fact, since the (random) division into the two data sets can be seen as part of $\varepsilon_k$, the adaptation simply involves changing $s_n$ to $s_n/2$.

Now, we may directly apply Lemma \eqref{Hdecomposition} to the U-statistics $\hat{\mu}_n(\mathbf{x})$:

\begin{lemma}\label{variancebound}
Let $\hat{\mu}_n(\mathbf{x})$ be as in \eqref{finalestimator2} and assume $T$ satisfies \textbf{(P4)} and
\begin{align*}
    \Var(T) < \infty.
\end{align*}
Then 
\begin{align*}
   \Var(\hat{\mu}_n(\mathbf{x}) )&\leq \frac{s_n^2}{n} \Var(T_1 ) + \frac{s_n^2}{n^2} \Var(T) \\
   &\leq \left( \frac{s_n}{n} + \frac{s_n^2}{n^2} \right) \Var(T).
\end{align*}
\end{lemma}

\begin{proof}
Using the composition in \eqref{ANOVA} on $T(\mathbf{Z}_{i_1}, \ldots, \mathbf{Z}_{i_{s}}):=\E_{\varepsilon} \left[ T(\mathbf{x}, \varepsilon; \mathbf{Z}_{i_1}, \ldots, \mathbf{Z}_{i_{s}}) \right]$, we have for $A \subset N=\{1,\ldots, n \}$, $|A|=s_n$,
\begin{align}
     T(\pi_{NA}(\mathcal{Z}_{n})) = \E[T(\pi_{NA}(\mathcal{Z}_{n}))] + \sum_{\ell=1}^s \sum_{ B \in C_{\ell}(A)  } T_{\ell}(\pi_{NB}(\mathcal{Z}_{n})) ,
\end{align}
Moreover, it holds by symmetry and i.i.d. sampling, that for $A_1=A_2 \subset N$, $|A_1|=|A_2|=\ell$, $T_{\ell}(\pi_{A_1}(\mathcal{Z}_{n})) = T_{\ell}(\pi_{A_2}(\mathcal{Z}_{n}))$. Thus we obtain,
\begin{align*}
\hat{\mu}_n(\mathbf{x})&= \E[T(\mathcal{Z}_{s_n}) ] + \binom{n}{s_n}^{-1}\Big( \binom{n-1}{s_n-1} \sum_{i=1}^n T_{1}(\mathbf{Z}_i) +  \binom{n-2}{s_n-2} \sum_{i_1 < i_2} T_{2}(\mathbf{Z}_{i_1}, \mathbf{Z}_{i_2})\\
&+ \ldots + \sum_{i_1 < i_2 < \ldots < i_{s_n}} T_s(\mathbf{Z}_{i_1}, \ldots, \mathbf{Z}_{i_{s_n}}) \Big).
\end{align*}
Now
\begin{align*}
    \binom{n}{s_n}^{-1}\binom{n-j}{s_n-j}
    &=\frac{s_n!}{n!} \frac{(n-j)!}{(s_n-j)!}
    \\&= \frac{s_n \cdot (s_n-1) \cdots (s_n-j+1) }{n \cdot (n-1) \cdots (n-j+1)}\\
    &=\frac{(s_n)_{j}}{(n)_{j}},
\end{align*}
where $(s_n)_{j}=s_n(s_n-1)\cdot (s_n-(j-1))=s_n!/(s_n-j)!$. In particular
\[
\binom{n}{s_n}^{-1} \binom{n-1}{s_n-1}  = \frac{s_n}{n}.
\]
Consequently,
\begin{align*}
    \hat{\mu}_n(\mathbf{x})  &= \frac{s_n}{n} \sum_{i=1}^n T_{1}(\mathbf{Z}_{i}) +  \frac{(s_n)_{2}}{(n)_{2}} \sum_{i_1 < i_2} T_{2}(\mathbf{Z}_{i_1}, \mathbf{Z}_{i_2}) + \frac{(s_n)_{3}}{(n)_{3}} \sum_{i_1 < i_2 < i_3} T_{3}(\mathbf{Z}_{i_1}, \mathbf{Z}_{i_2}, \mathbf{Z}_{i_3}) +\\
    & \ldots +  \frac{(s_n)_{s_n}}{(n)_{s_n}} \sum_{i_1 < i_2 < \ldots < i_{s_n}} T_{s_n}(\mathbf{Z}_{i_1}, \ldots, \mathbf{Z}_{i_{s_n}}),
\end{align*}
with covariances between terms equal to 0, as in Lemma \ref{Hdecomposition}. Thus
\begin{align*}
    \Var(\hat{\mu}_n(\mathbf{x}) ) &= \frac{s_n^2}{n^2} n \Var(T_1) +  \sum_{i=2}^{s_n} \left( \frac{(s_n)_{i}}{(n)_{i}}\right)^2 \binom{n}{i} \Var(T_i)\\
    &=\frac{s_n^2}{n} \Var(T_1) + \sum_{i=2}^{s_n} \left(\frac{(s_n)_{i}}{(n)_{i}}\right) \binom{s_n}{i}\Var(T_i)\\
    &\leq \frac{s_n^2}{n} \Var(T_1) + \frac{(s_n)_2}{(n)_2} \sum_{i=2}^{s_n}  \binom{s_n}{i}\Var(T_i) \\
    & \leq \frac{s_n^2}{n} \Var(T_1) + \frac{s_n^2}{n^2} \Var(T),
\end{align*}
where the last step followed from \eqref{vardecomp}. This proves the first inequality. On the other hand, we have from Lemma \ref{Hdecomposition} that
\begin{align*}
    \Var(T)=\sum_{i=1}^{s_n} \binom{s_n}{i} \Var(T_i) \geq s_n \Var(T_1 ),
\end{align*}
leading to the second inequality.
\end{proof}


\begin{lemma} \label{lemma2}[ Lemma 2 from \citet{wager2018estimation}]
Let $T$ is a tree satisfying \textbf{(P3)}, \textbf{(P5)}  trained on $\mathcal{Z}_s=(\xi_1,\mathbf{X}_1), \ldots (\xi_s,\mathbf{X}_s)$ and let $L(x, \mathcal{Z}_s)$ be the leaf containing $\mathbf{x}$. Suppose that $\mathbf{X}_1,\ldots, \mathbf{X}_s$ are i.i.d. on $[0,1]^p$ independently with a density $f$ bounded away from $0$ and infinity. Then,
\begin{align}
    \P \left( \mbox{diam} (L(\mathbf{x}, \mathcal{Z}_s)) \geq \sqrt{d} \left( \frac{s}{2k-1}\right) ^{-0.51 \frac{\log((1-\alpha)^{-1})}{\log(\alpha^{-1})} \frac{\pi}{p}}   \right) \leq d \left( \frac{s}{2k-1}\right) ^{-1/2 \frac{\log((1-\alpha)^{-1})}{\log(\alpha^{-1})} \frac{\pi}{p}}.
\end{align}
\end{lemma}

\begin{lemma}\label{helperlemma}
Let $T$ be a tree satisfying \textbf{(P1)}, \textbf{(P2)} and $L(\mathbf{x}, \mathcal{Z}_s)$ be the leaf containing $\mathbf{x}$. Then
\begin{align}\label{star1star}
    \E[T(\mathcal{Z}_s)] = \E[ \E[\xi_1 \mid \mathbf{X}_1 \in L(\mathbf{x}, \mathcal{Z}_s) ] ],
\end{align}
and 
\begin{align}\label{star2star}
    \Var(T(\mathcal{Z}_s) ) \leq  \sup_{\mathbf{x} \in [0,1]^p} \E[ \| \xi_1 \|_{\mathcal{H}}^2 \mid \mathbf{X}=\mathbf{x}].
\end{align}
\end{lemma}

\begin{proof}
We want to prove
\begin{align}
    \E[T(\mathcal{Z}_s)]=\E[\E[ T(\mathcal{Z}_s)  \mid L(\mathbf{x}, \mathcal{Z}_s) ]]    = \E[\E[\xi_1 \mid \mathbf{X}_1 \in L(\mathbf{x}, \mathcal{Z}_s),  L(\mathbf{x}, \mathcal{Z}_s)]].
\end{align}
Let for the following $N_{\mathbf{x}}=\sum_{i=1}^s \1\{\mathbf{X}_i \in L(\mathbf{x}, \mathcal{Z}_s)\}$. Then due to i.i.d. sampling: 
\begin{align*}
     \E[\E[ T(\mathcal{Z}_s)  \mid L(\mathbf{x}, \mathcal{Z}_s) ]] =  \E[\E[ \sum_{i=1}^s S_i \xi_i \mid L(\mathbf{x}, \mathcal{Z}_s)]]= s \E[S_1 \xi_1\mid L(\mathbf{x}, \mathcal{Z}_s)].
\end{align*}
The last expression can be broken into
\begin{align}\label{result00}
    s \E[ \E[S_1 \xi_1\mid L(\mathbf{x}, \mathcal{Z}_s)] ]&=s \E \left[ \E[ \E[ S_1 \xi_1 \mid N_{\mathbf{x}}, L(\mathbf{x}, \mathcal{Z}_s)  ]\mid L(\mathbf{x}, \mathcal{Z}_s)] \right] \nonumber \\
    &= \E \left[\E\left[  \frac{s}{N_{\mathbf{x}}} \E[ \1\{\mathbf{X}_1 \in L(\mathbf{x}, \mathcal{Z}_s)\} \xi_1 \mid N_{\mathbf{x}}, L(\mathbf{x}, \mathcal{Z}_s)  ] \mid L(\mathbf{x}, \mathcal{Z}_s) \right] \right] \nonumber \\
     &= \E \Big[ \E\Big[ \frac{s}{N_{\mathbf{x}}} \E[ \xi_1 \mid N_{\mathbf{x}}, L(\mathbf{x}, \mathcal{Z}_s), \mathbf{X}_1 \in L(\mathbf{x}, \mathcal{Z}_s)  ] \nonumber \\
     & \P\left(\mathbf{X}_1 \in L(\mathbf{x}, \mathcal{Z}_s) \mid N_{\mathbf{x}}, L(\mathbf{x}, \mathcal{Z}_s) \right) \mid L(\mathbf{x}, \mathcal{Z}_s)\Big] \Big].
\end{align}
Now, by honesty, given the knowledge that $\mathbf{X}_1 \in L(\mathbf{x}, \mathcal{Z}_s) $, $\xi_1$ is independent of $N_{\mathbf{x}}$, thus:
\begin{align*}
   & s \E[S_1 \xi_1\mid L(\mathbf{x}, \mathcal{Z}_s)] \\
    &=\E \left[ \E[ \xi_1 \mid  L(\mathbf{x}, \mathcal{Z}_s), \mathbf{X}_1 \in L(\mathbf{x}, \mathcal{Z}_s)  ]  \E\left[  \frac{s}{N_{\mathbf{x}}}\P\left(\mathbf{X}_1 \in L(\mathbf{x}, \mathcal{Z}_s) \mid N_{\mathbf{x}}, L(\mathbf{x}, \mathcal{Z}_s) \right) \mid L(\mathbf{x}, \mathcal{Z}_s)\right] \right]\\
     &=\E \left[ \E[ \xi_1 \mid  L(\mathbf{x}, \mathcal{Z}_s), \mathbf{X}_1 \in L(\mathbf{x}, \mathcal{Z}_s)  ]  s \E\left[  S_1 \mid L(\mathbf{x}, \mathcal{Z}_s)\right] \right].
\end{align*}

Now it holds by i.i.d. sampling that,
\begin{align*}
    s \E\left[  S_1 \mid L(\mathbf{x}, \mathcal{Z}_s)\right] = \sum_{i=1}^s \E\left[  S_i \mid L(\mathbf{x}, \mathcal{Z}_s)\right] =  \E\left[ \sum_{i=1}^s  S_i \mid L(\mathbf{x}, \mathcal{Z}_s)\right] =1,
\end{align*}
as $\sum_{i=1}^s  S_i=1$ by definition.


For \eqref{star2star}, we write
\begin{align*}
     &\Var(T(\mathcal{Z}_s)  )  \leq \E \left[ \left\| \sum_{i=1}^{s} S_i \xi_i \right \|_{\mathcal{H}}^2 \right]=\E \left[ \frac{1}{N_{\mathbf{x}}^2}\sum_{i=1}^{s} \1\{\mathbf{X}_i \in L(\mathbf{x}, \mathcal{Z}_s)\} \| \xi_i\|_{\mathcal{H}}^2 \right] + \\
     &\E \left[ \frac{1}{N_{\mathbf{x}}^2}  \sum_{i=1}^s \sum_{j \neq i}  \1\{\mathbf{X}_j \in L(\mathbf{x}, \mathcal{Z}_s)\} \1\{\mathbf{X}_i \in L(\mathbf{x}, \mathcal{Z}_s)\} \langle \xi_i, \xi_j \rangle_{\mathcal{H}} \right].
\end{align*}
We focus on the second term. For the first, the bound follows by analogous arguments. Similar as before,
\begin{align}\label{part1}
    &\E \left[ \frac{1}{N_{\mathbf{x}}^2} \sum_{i=1}^s \sum_{j \neq i}  \1\{\mathbf{X}_j \in L(\mathbf{x}, \mathcal{Z}_s)\} \1\{\mathbf{X}_i \in L(\mathbf{x}, \mathcal{Z}_s)\} \langle \xi_i, \xi_j \rangle_{\mathcal{H}} \right]  \nonumber \\
    &= s(s-1)\E \left[ \frac{1}{N_{\mathbf{x}}^2}  \1\{\mathbf{X}_1 \in L(\mathbf{x}, \mathcal{Z}_s)\} \1\{\mathbf{X}_2 \in L(\mathbf{x}, \mathcal{Z}_s)\} \langle \xi_1, \xi_2 \rangle_{\mathcal{H}} \right] \nonumber \\
    &= s(s-1)\E \left[ \frac{1}{N_{\mathbf{x}}^2} \E \left[   \1\{\mathbf{X}_1 \in L(\mathbf{x}, \mathcal{Z}_s)\} \1\{\mathbf{X}_2 \in L(\mathbf{x}, \mathcal{Z}_s)\} \langle \xi_1, \xi_2 \rangle_{\mathcal{H}} \mid N_{\mathbf{x}} \right] \right].
\end{align}
By the same argument as above
\begin{align*}
   &\P( \mathbf{X}_1 \in L(\mathbf{x}, \mathcal{Z}_s), \mathbf{X}_2 \in L(\mathbf{x}, \mathcal{Z}_s) \mid N_{\mathbf{x}},L(\mathbf{x}, \mathcal{Z}_s) )  \nonumber\\ 
   &=\frac{1}{ s(s-1)}   \E \left[ \sum_{i=1}^s \sum_{j \neq i} \1\{\mathbf{X}_i \in L(\mathbf{x}, \mathcal{Z}_s)\} \1\{\mathbf{X}_j \in L(\mathbf{x}, \mathcal{Z}_s)\}  \mid N_{\mathbf{x}}, L(\mathbf{x},\mathcal{Z}_s) \right]\nonumber\\
   &=\frac{1}{s(s-1)}   \E \left[ N_{\mathbf{x}}(N_{\mathbf{x}}-1) \mid N_{\mathbf{x}}, L(\mathbf{x},\mathcal{Z}_s) \right]\nonumber \\
   &= \frac{N_{\mathbf{x}}(N_{\mathbf{x}}-1)}{s(s-1)},
\end{align*}
and
\begin{align}\label{part2}
    &\E \left[ \1\{\mathbf{X}_1 \in L(\mathbf{x}, \mathcal{Z}_s)\} \1\{\mathbf{X}_2 \in L(\mathbf{x}, \mathcal{Z}_s)\} \langle \xi_1, \xi_2 \rangle_{\mathcal{H}} \mid N_{\mathbf{x}} \right]  \nonumber \\
    &=\E \left[ \E \left[ \1\{\mathbf{X}_1 \in L(\mathbf{x}, \mathcal{Z}_s)\} \1\{\mathbf{X}_2 \in L(\mathbf{x}, \mathcal{Z}_s)\} \langle \xi_1, \xi_2 \rangle_{\mathcal{H}} \mid N_{\mathbf{x}}, L(\mathbf{x},\mathcal{Z}_s) \right] \mid N_{\mathbf{x}}  \right] \nonumber \\
    &=\E \big[\E \left[ \langle \xi_1, \xi_2 \rangle_{\mathcal{H}} \mid N_{\mathbf{x}}, A,\mathbf{X}_1 \in L(\mathbf{x}, \mathcal{Z}_s), \mathbf{X}_2 \in L(\mathbf{x}, \mathcal{Z}_s)  \right]\nonumber\\
    &\P( \mathbf{X}_1 \in L(\mathbf{x}, \mathcal{Z}_s), \mathbf{X}_2 \in L(\mathbf{x}, \mathcal{Z}_s) \mid N_{\mathbf{x}},L(\mathbf{x}, \mathcal{Z}_s) )\mid N_{\mathbf{x}}  \big] \nonumber \\
    & =\E \left[\E \left[ \langle \xi_1, \xi_2 \rangle_{\mathcal{H}} \mid \mathbf{X}_1 \in L(\mathbf{x}, \mathcal{Z}_s), \mathbf{X}_2 \in L(\mathbf{x}, \mathcal{Z}_s), L(\mathbf{x}, \mathcal{Z}_s)  \right] \mid N_{\mathbf{x}} \right] \frac{N_{\mathbf{x}}(N_{\mathbf{x}}-1)}{s(s-1)}.
\end{align}
Thus combining \eqref{part1} and \eqref{part2},
\begin{align*}
    &\E \left[ \frac{1}{N_{\mathbf{x}}^2} \sum_{i=1}^s \sum_{j \neq i}  \1\{\mathbf{X}_j \in L(\mathbf{x}, \mathcal{Z}_s)\} \1\{\mathbf{X}_i \in L(\mathbf{x}, \mathcal{Z}_s)\} \langle \xi_i, \xi_j \rangle_{\mathcal{H}} \right]  \\
    &= \E \Big[\frac{N_{\mathbf{x}}(N_{\mathbf{x}}-1)}{N_{\mathbf{x}}^2}\E \left[\E \left[ \langle \xi_1, \xi_2 \rangle_{\mathcal{H}}  \mid \mathbf{X}_1 \in L(\mathbf{x}, \mathcal{Z}_s), \mathbf{X}_2 \in L(\mathbf{x}, \mathcal{Z}_s), L(\mathbf{x}, \mathcal{Z}_s)  \right] \mid N_{\mathbf{x}} \right]  \Big]  \\
%
    &\leq \E \left[\E \left[ \langle \xi_1, \xi_2 \rangle_{\mathcal{H}}  \mid \mathbf{X}_1 \in L(\mathbf{x}, \mathcal{Z}_s), \mathbf{X}_2 \in L(\mathbf{x}, \mathcal{Z}_s), L(\mathbf{x}, \mathcal{Z}_s)  \right] \right]\\
    &=  \E \left[ \langle \E[ \xi_1 \mid \mathbf{X}_1 \in L(\mathbf{x}, \mathcal{Z}_s)],  \E[\xi_2 \mid \mathbf{X}_2 \in L(\mathbf{x}, \mathcal{Z}_s)] \rangle_{\mathcal{H}}  \right],
\end{align*}
where in the last step we used independence of $(\xi_1, \1\{ \mathbf{X}_1 \in L(\mathbf{x}, \mathcal{Z}_s)\}), (\xi_2, \1\{ \mathbf{X}_2 \in L(\mathbf{x}, \mathcal{Z}_s)\})$ conditionally on $L(\mathbf{x}, \mathcal{Z}_s)$ and (C9). Finally,
\begin{align*}
    \E \left[ \langle \E[ \xi_1 \mid \mathbf{X}_1 \in L(\mathbf{x}, \mathcal{Z}_s)], \E[\xi_2 \mid \mathbf{X}_2 \in L(\mathbf{x}, \mathcal{Z}_s) ]\rangle_{\mathcal{H}}  \right]& \leq   \sup_{\mathbf{x} \in [0,1]^p} \| \E[\xi_1 \mid \mathbf{X}=\mathbf{x}] \|_{\mathcal{H}}^2 \\
    & \leq   \sup_{\mathbf{x} \in [0,1]^p} \E[\|\xi_1\|_{\mathcal{H}}^2 \mid \mathbf{X}=\mathbf{x}],
\end{align*}
proving the claim.

\end{proof}

\begin{corollary}\label{bias}
Under the conditions of Lemma \ref{lemma2}, assume 
\begin{align*}
    \mathbf{x} &\mapsto \mu(\mathbf{x})=\E[ \xi \mid \mathbf{X} = \mathbf{x}] \in \mathcal{H},
\end{align*}
is Lipschitz and that the trees $T$ in the forest satisfy \textbf{(P2)} and \textbf{(P3)}. Then
\begin{equation}\label{biasbound}
    \| \E[\hat{\mu}_n(\mathbf{x})] - \mu(\mathbf{x}) \|_{\mathcal{H}} = \mathcal{O}\left( s^{-1/2 \frac{\log((1-\alpha)^{-1})}{\log(\alpha^{-1})} \frac{\pi}{p}}\right),
\end{equation}
and
\begin{equation}\label{propabiltiyconv1}
\| \E[\xi \mid \mathbf{X} \in L(\mathbf{x}, \mathcal{Z}_s)]\|_{\mathcal{H}} \stackrel{p}{\to} \| \E[\xi \mid \mathbf{X}=\mathbf{x}] \|_{\mathcal{H}}.
\end{equation}
If moreover, 
\begin{align*}
    \mathbf{x} &\mapsto \E[\|\xi\|_{\mathcal{H}}^2 \mid \mathbf{X}=\mathbf{x}] \in \R,
\end{align*}
is Lipschitz, then: 
\begin{align}\label{propabiltiyconv2}
    \E[ \|\xi\|_{\mathcal{H}}^2 \mid \mathbf{X} \in L(\mathbf{x}, \mathcal{Z}_s)]  &\stackrel{p}{\to} \E[\|\xi\|_{\mathcal{H}}^2 \mid \mathbf{X}=\mathbf{x}].
\end{align}
\end{corollary}

\begin{proof}
By \eqref{star1star}, it holds as in \citet{wager2018estimation}
\begin{align*}
  \|  \E[T(x, \mathbf{Z})] - \E[\xi\mid \mathbf{X}=\mathbf{x}] \|_{\mathcal{H}} =\| \E[ \E[\xi \mid \mathbf{X} \in L(\mathbf{x}, \mathcal{Z}_s)] - \E[\xi\mid \mathbf{X}=\mathbf{x}]     ] \|_{\mathcal{H}}.
\end{align*}
Let 
\[
s_1^*= \sqrt{d} \left(\frac{s}{2k-1}\right)^{-0.51  \frac{\log((1-\alpha)^{-1})}{\log(\alpha^{-1})}  \frac{\pi}{p} }, s_2^*= d \left(\frac{s}{2k-1}\right)^{-0.5  \frac{\log((1-\alpha)^{-1})}{\log(\alpha^{-1})}  \frac{\pi}{p} }
\]
Then it follows from Lemma \ref{lemma2} that
\[
\P(L(\mathbf{x},\mathcal{Z}_s) \geq s_1^*) \leq s_2^*,
\]
while by Lipschitz continuity: $ \E[ \| \E[\xi \mid \mathbf{X} \in L(\mathbf{x}, \mathcal{Z}_s)] - \E[\xi\mid \mathbf{X}=\mathbf{x}] \|_{\mathcal{H}} ]  \leq L \E  [\text{diam}(L(\mathbf{x}, \mathcal{Z}_s))]$, where $L$ is the Lipschitz constant.
Thus,
\begin{align*}
    &\|\E[ \E[\xi \mid \mathbf{X} \in L(\mathbf{x}, \mathcal{Z}_s)] - \E[\xi\mid \mathbf{X}=\mathbf{x}]     ] \|_{\mathcal{H}}\\
    %
    &\leq \| \E[ \left( \E[\xi \mid \mathbf{X} \in L(\mathbf{x}, \mathcal{Z}_s)] - \E[\xi\mid \mathbf{X}=\mathbf{x}] \right)  \1\{\text{diam}(L(\mathbf{x}, \mathcal{Z}_s)) \geq s_1^*\}    ]\|_{\mathcal{H}} \\
    & + \| \E[ \left(\E[\xi \mid \mathbf{X} \in L(\mathbf{x}, \mathcal{Z}_s)] - \E[\xi\mid \mathbf{X}=\mathbf{x}] \right) \1\{\text{diam}(L(\mathbf{x}, \mathcal{Z}_s)) < s_1^*\}     ] \|_{\mathcal{H}}\\
      &\leq  \E[ \|  \E[\xi \mid \mathbf{X} \in L(\mathbf{x}, \mathcal{Z}_s)] - \E[\xi\mid \mathbf{X}=\mathbf{x}]  \|_{\mathcal{H}}   \1\{\text{diam}(L(\mathbf{x}, \mathcal{Z}_s)) \geq s_1^*\}    ]\\
    & +  \E[\| \E[\xi \mid \mathbf{X} \in L(\mathbf{x}, \mathcal{Z}_s)] - \E[\xi\mid \mathbf{X}=\mathbf{x}] \|_{\mathcal{H}} \1\{\text{diam}(L(\mathbf{x}, \mathcal{Z}_s)) < s_1^*\}     ] \\
    & \leq \left( \sup_{\mathbf{x}_1,\mathbf{x}_2 \in [0,1]^p} \| \E[\xi \mid \mathbf{X}=\mathbf{x}_1] - \E[\xi\mid \mathbf{X}=\mathbf{x}_2] \|_{\mathcal{H}} \right) \P(\text{diam}(L(\mathbf{x}, \mathcal{Z}_s)) \geq s_1^*) + L s_1^*\\
    & \leq \left( \sup_{\mathbf{x}_1,\mathbf{x}_2 \in [0,1]^p} \| \E[\xi \mid \mathbf{X}=\mathbf{x}_1] - \E[\xi\mid \mathbf{X}=\mathbf{x}_2] \|_{\mathcal{H}}  \right) s_2^* + L s_1^*\\
    & \precsim  \left( \sup_{\mathbf{x}_1,\mathbf{x}_2 \in [0,1]^p} \| \E[\xi \mid \mathbf{X}=\mathbf{x}_1] - \E[\xi\mid \mathbf{X}=\mathbf{x}_2] \|_{\mathcal{H}}  \right) s_2^*,
\end{align*}
since $s_1^*/s_2^* \to 0$. Due to the Lipschitz condition
\[
\sup_{\mathbf{x}_1,\mathbf{x}_2 \in [0,1]^p} \| \E[\xi \mid \mathbf{X}=\mathbf{x}_1] - \E[\xi\mid \mathbf{X}=\mathbf{x}_2] \|_{\mathcal{H}} \leq L \sup_{\mathbf{x}_1,\mathbf{x}_2 \in [0,1]^p} \| \mathbf{x}_1 - \mathbf{x}_2 \|_{\R^d} = \mathcal{O}(1).
\]
Finally by the reverse triangle inequality
\begin{align*}
    \left|  \|\E[ \xi \mid \mathbf{X} \in L(\mathbf{x}, \mathcal{Z}_s) ] \|_{\mathcal{H}} - \|\E[  \xi \mid \mathbf{X}=\mathbf{x} ] \|_{\mathcal{H}} \right| &\leq  \|\E[ \xi \mid \mathbf{X} \in L(\mathbf{x}, \mathcal{Z}_s) ]  - \E[  \xi \mid \mathbf{X}=\mathbf{x} ] \|_{\mathcal{H}} \\
    &\leq L \text{diam}(L(\mathbf{x}, \mathcal{Z}_s)),
\end{align*}
and if $\mathbf{x} \mapsto \E[\|\xi\|_{\mathcal{H}}^2 \mid \mathbf{X}=\mathbf{x}]$ is Lipschitz as well, also
\begin{align*}
   | \E[ \| \xi\|_{\mathcal{H}}^2 \mid \mathbf{X} \in L(\mathbf{x}, \mathcal{Z}_s) ] - \E[ \| \xi\|_{\mathcal{H}}^2 \mid \mathbf{X}=\mathbf{x} ] |  &\leq C \text{diam}(L(\mathbf{x}, \mathcal{Z}_s)).
\end{align*}
Since $\text{diam}(L(\mathbf{x}, \mathcal{Z}_s)) \stackrel{p}{\to} 0$, as $s \to \infty$, \eqref{propabiltiyconv1}, respectively \eqref{propabiltiyconv2} hold true.  

As the expectation of the forest is the same as that of one tree:
\[
    \E[\hat{\mu}_n(\mathbf{x})] =  \E \left[ T(\mathbf{x}, \varepsilon; \mathbf{Z}_{i_1}, \ldots, \mathbf{Z}_{i_{s_n}}) \right],
\]
the result follows.
\end{proof}


This leads us to the proof of of Theorem \ref{thm: consistency} in the main text.


\consistency*

\begin{proof}
We first note that $ \sup_{\mathbf{x} \in [0,1]^p} \E[ \|\mu(\delta_\bold{Y})\|_\mathcal{H}^2\mid \mathbf{X} \myeq \mathbf{x}] < \infty$ together with \eqref{star2star} implies $\Var(T) < \infty$. Thus, from Markov's inequality and Lemma \ref{variancebound},
\begin{align*}
    \P\left( n^{\gamma} ||\hat{\mu}_n(\mathbf{x}) - \E[ \hat{\mu}_n(\mathbf{x})]||_{\mathcal{H}}  > \varepsilon \right)\leq \frac{n^{2\gamma}}{\varepsilon^2} (s/n + s^2/n^2) \Var(T) =\frac{1}{\varepsilon^2} \mathcal{O}(n^{2\gamma+\beta-1}).
\end{align*}
Thus 
\[
n^{\gamma} ||\hat{\mu}_n(\mathbf{x}) - \E[ \hat{\mu}_n(\mathbf{x})]||_{\mathcal{H}} =\mathcal{O}_p(1),
\]
for $\gamma \leq (1- \beta)/2$. In particular, it goes to zero for any $\varepsilon > 0$, if $\gamma < (1- \beta)/2$. Since,
\[
n^{\gamma} \left\| \hat{\mu}_n(\mathbf{x}) -  \mu(\mathbf{x}) \right\|_{\mathcal{H}} \leq  n^{\gamma}\left\| \hat{\mu}_n(\mathbf{x}) -  \E[\hat{\mu}(\mathbf{x})] \right\| + n^{\gamma}\left\| \E[\hat{\mu}(\mathbf{x})] -  \mu(\mathbf{x}) \right\|_{\mathcal{H}}, 
\]
the result follows as soon as the second expression goes to zero. Now from Theorem \ref{bias}, with $C_{\alpha}=\frac{\log((1-\alpha)^{-1})}{\log(\alpha^{-1})}$,
\[
\| \E[\hat{\mu}_n(\mathbf{x})] - \mu(\mathbf{x}) \|_{\mathcal{H}} = \mathcal{O}\left( s_n^{-1/2 C_{\alpha}  \frac{\pi}{p}}\right)= \mathcal{O}\left( n^{-1/2 \beta C_{\alpha} \frac{\pi}{p}}\right).
\]
This goes to zero provided that,
\begin{align*}
    1/2 \beta C_{\alpha} \frac{\pi}{p} > \gamma.
\end{align*}
\end{proof}

To prove Corollaries \ref{cor: metrizing weak convergence} and \ref{cor: cdfresult}, we first need another auxiliary result:

\begin{lemma}\label{asplemma}
Let $(S,d)$ be a separable metric space and $\mathbf{X}_n: (\Omega, \mathcal{A}) \to (S, \mathcal{B}(S))$, $n \in \N$ and $\mathbf{X}: (\Omega, \mathcal{A}) \to (S, \mathcal{B}(S))$ be measurable. Then $\mathbf{X}_n \stackrel{p}{\to} \mathbf{X}$ if and only if for every subsequence $n(k)$ there exists a further subsequence $n(k(l))$ such that
\begin{equation}\label{subsequenceas}
    \mathbf{X}_{n(k(l))} \to \mathbf{X} \text{ a.s.}
\end{equation}
\end{lemma}

\begin{proof}
If $\mathbf{X}_n \stackrel{p}{\to} \mathbf{X}$, then so does any subsequence $\mathbf{X}_{n(k)}$. By well-known results, see e.g.\ \citet[Chapter 2]{vanderVaart}, this implies that there exists a further subsequence $\mathbf{X}_{n(k(l))}$ such that a.s. convergence holds.

As is well known, there exists a metric $\rho$ on $\mathcal{P}(S)$ such that $\rho(\mathbf{X}_n, \mathbf{X}) \to 0$ iff $\mathbf{X}_n \stackrel{p}{\to} X$, see e.g., \citet[Chapter 11]{dudley}. 
Now assume that for any subsequence we can find a further subsequence such that \eqref{subsequenceas} holds, but the overall sequence does not converge in probability. Then we can build a subsequence such that for some $\varepsilon > 0$,
\[
\rho(\mathbf{X}_{n(k)}, \mathbf{X}) \geq \varepsilon
\]
for all elements of that subsequence. Thus any further subsequence will also not convergence in probability and consequently cannot converge a.s. This proves the claim.
\end{proof}

We note that the set $A$ with $P(A)=1$ on which \eqref{subsequenceas} holds is allowed to depend on the subsequence. Corollary \ref{cor: metrizing weak convergence} and \ref{cor: cdfresult} are finally proven jointly in the following Corollary. The proof is motivated by the tools used in \citet{weakconvergenceinprobabilitypaper}.

\begin{corollary}\label{jointcorollaryproof}
Assume that one of the following two sets of conditions holds:
\begin{itemize}
    \item[(a)] The kernel $k$ is bounded, (jointly) continuous and has 
    \begin{align}\label{characteristiccond}
        \int \int k(\mathbf{x},\mathbf{y}) d\mathcal{P}(\mathbf{x}) d\mathcal{P}(\mathbf{y}) > 0 \  \ \forall \mathcal{P} \in \mathcal{M}_b(\R^d)\setminus \{0\}.
    \end{align}
    Moreover, $\mathbf{y} \mapsto k(\mathbf{y}_0, \mathbf{y})$ is vanishing at infinity, for all $\mathbf{y}_0 \in \R^d$.
    \item[(b)] The kernel $k$ is bounded, shift-invariant, (jointly) continuous and $\nu$ in the Bochner representation in \eqref{eq: bochner} is supported on all of $\R^d$. Moreover, $\mathbf{Y}$ takes its values almost surely in a closed and bounded subset of $\R^d$.
\end{itemize}
Then, under the conditions of Theorem \ref{thm: consistency}, we have for any bounded and continuous function $f:\R^d \to \R$ that DRF consistently estimates the target $\tau(\bold{x})=\E[f(\bold{Y})\mid\bold{X}\myeq\bold{x}]$ for any $\bold{x} \in [0,1]^p$:
$$\sum_{i=1}^n w_{\mathbf{x}}(\mathbf{x}_i) f(\mathbf{y}_i) \,\stackrel{p}{\to}\, \E[f(\bold{Y})\mid\bold{X}\myeq\bold{x}].$$
Moreover,
\begin{align*}
    \hat{F}_{\mathbf{Y} \mid\bold{X}\myeq\bold{x}}(\mathbf{t}) \, &\stackrel{p}{\to}\, F_{\mathbf{Y}\mid\bold{X}\myeq\bold{x}}(\mathbf{t}) \\
    \hat{F}_{Y_i\mid\bold{X}\myeq\bold{x}}^{-1}(t) \, &\stackrel{p}{\to}\, F_{Y_i\mid\bold{X}\myeq\bold{x}}^{-1}(t),
\end{align*}
for all points of continuity $\mathbf{t} \in \R^d$ and $t \in \R$ of $F_{\mathbf{Y}\mid\bold{X}\myeq\bold{x}}(\cdot)$ and $F_{Y_i\mid\bold{X}\myeq\bold{x}}^{-1}(\cdot)$ respectively.
\end{corollary}

\begin{proof}
As shown in \cite[Theorem 3.2]{optimalestimationofprobabilitymeasures}, (a) implies that $k$ metrizes weak convergence. Similarly, from Theorem 9 in \citet{sriperumbudur10a}, it follows that $k$ is characteristic on the compact subspace of $\R^d$ in which $Y$ takes its value almost surely. Thus, ignoring the Null set, Theorem 23 in \citet{sriperumbudur10a} implies that $k$ metrizes the weak convergence in this case as well. Thus in both cases $\| \hat{\mu}_{n(k(l))}(\mathbf{x}) -\mu(\mathbf{x})\|_{\mathcal{H}} \to 0$ implies weak convergence of $\hat{\mu}_{n(k(l))}(\mathbf{x})$ to $\mu(\mathbf{x})$. 

From Theorem \ref{asplemma}, for any subsequence, we can choose a further subsequence, such that
\[
\| \hat{\mu}_{n(k(l))}(\mathbf{x}) -\mu(\mathbf{x})\|_{\mathcal{H}} \to 0 \text{ , a.s.}
\]
and since it is assumed that $\|\cdot \|_{\mathcal{H}}$ metrizes weak convergence, $\hat{\mu}_{n(k(l))}(\mathbf{x})$ converges weakly to $\mu(\mathbf{x})$ on a set $A$, depending on the subsequence, with $\P(A)=1$. Let $C_b(\R^d)$ denote the space of all bounded continuous functions on $\R^d$. By the Portmanteau theorem (see e.g.\ \citet{dudley}), this implies that on $A$
\begin{itemize}
    \item[(I)]$ \int f d \hat{\P}(\mathbf{Y}\mid \mathbf{X}=\mathbf{x}) \to  \int f d \P(\mathbf{Y}\mid \mathbf{X}=\mathbf{x})$ for all $f \in C_b(\R^d)$ 
    \item[(II)] $\hat{F}_{\mathbf{Y} \mid\bold{X}\myeq\bold{x}}(\mathbf{t}) \, \to \, F_{\mathbf{Y}\mid\bold{X}\myeq\bold{x}}(\mathbf{t})$ for all continuity points $\mathbf{t} \in \R^d$ of $F_{\mathbf{Y}\mid\bold{X}\myeq\bold{x}}(\cdot)$,
\end{itemize}
where we omitted the dependence on the subsequence. But, since the subsequence $n(k)$ was arbitrary, this immediately implies
\begin{itemize}
    \item[(I')]$ \int f d \hat{\P}(\mathbf{Y}\mid \mathbf{X}=\mathbf{x}) \stackrel{p}{\to}  \int f d \P(\mathbf{Y}\mid \mathbf{X}=\mathbf{x})$ for all $f \in C_b(\R^d)$ 
    \item[(II')] $\hat{F}_{\mathbf{Y} \mid\bold{X}\myeq\bold{x}}(t) \, \stackrel{p}{\to} \, F_{\mathbf{Y}\mid\bold{X}\myeq\bold{x}}(t)$ for all continuity points $t$ of $F_{\mathbf{Y}\mid\bold{X}\myeq\bold{x}}(\cdot)$,
\end{itemize}
for the overall sequence.

On the other hand, (II) implies that on $A$, for the given subsequence $n(k(l))$, $\hat{F}_{Y_i \mid\bold{X}\myeq\bold{x}}(t) \, \to \, F_{Y_i\mid\bold{X}\myeq\bold{x}}(t)$ for all $t \in \R$ at which $F_{Y_i\mid\bold{X}\myeq\bold{x}}(\cdot)$ is continuous. Using for each $\omega \in A$ the arguments in \citet[Chapter 21]{vanderVaart}, this implies that 
\[
  \hat{F}_{Y_i\mid\bold{X}\myeq\bold{x}}^{-1}(t) \, \to\, F_{Y_i\mid\bold{X}\myeq\bold{x}}^{-1}(t), \text{ for all continuity points $t$ of $F_{Y_i\mid\bold{X}\myeq\bold{x}}^{-1}(t)$}
\]
on $A$ for the given subsequence. Again, as the subsequence $n(k)$ was arbitrary, this implies the result.
\end{proof}

\section{Simulation Details}
\label{appendix: simulation details}
In this section we describe in detail all our simulations shown in the main paper, together with the data used in the analysis. The data sets are available in the R-package \texttt{drf} as well.

\subsection{Air quality data}
\paragraph{Data.}
This data is obtained from the website of the Environmental Protection Agency website (\url{https://aqs.epa.gov/aqsweb/airdata/download_files.html}). We have daily measurements for $5$ years of data (2015-2019) for $6$ 'criteria' pollutants that form the Air Quality Index (AQI):
\begin{itemize}
    \item O$_3$ - ground ozone ($8$ hours' average, expressed in pieces per million (ppm))
    \item SO$_2$ - sulfur dioxide ($1$ hour average, expressed in pieces per billion (ppb))
    \item CO - carbon monoxide ($8$ hours' average, expressed in pieces per million (ppm))
    \item NO$_2$ - nitrogen dioxide ($1$ hour average, expressed in pieces per billion (ppb))
    \item PM$2.5$ - fine particulate matter smaller than $2.5$ micrometers ($24$ hours' average, expressed in $\nu g/m^3$)
    \item PM$10$ - large particulate matter, smaller than $10$ micrometers ($24$ hours' average, expressed in $\mu g/m^3$)
\end{itemize}
For the above quantities, we have the maximal and mean value within the same day. In our analysis we have used only the maximal intraday values.

The pollutants are measured at different measurement sites. For each site we have information about
\begin{itemize}
    \item site address (street, city, county, state, zip code)
    \item site coordinates (longitude and latitude)
    \item site elevation
    \item location setting (rural, urban, suburban) 
    \item how the land is used within a $1/4$ mile radius (agricultural, forest, desert, industrial, commercial, residential, blighted area, military reservation, mobile)
    \item date when the measurement site was put in operation
    \item date when the measurement site was decommissioned (NA if the site is still operational)
\end{itemize}
We have information about $19'739$ sites, much more than the number of $2'419$ sites from which we have measurements in years 2015-2019, since many sites were only operating in the past and are decommissioned.

In total there is $5'305'859$ pollutant measurements. Many pollutants are measured at the same site, but it is important to note that not every site measures every pollutant, so there is a lot of 'missing' measurements. It can also occur that there are several measuring devices for the same pollutant at the same site, in which case we just average the measurements across the devices and do not report those measurements separately.

\paragraph{Analysis.}
Since we have a lot of missing data, we use only the data points (identified by the measurement date and the measurement site) for which we have measurements of all the pollutants chosen as the responses. For that reason we also do not train DRF with all $6$ pollutants as the responses, but only those that we are interested in, since only $64$ sites measure all pollutants. For computational feasibility, we only use $50'000$ of the available measurements for the training step. We also omit the states Alaska and Hawaii and the US territories for plotting purposes.

To obtain the results displayed in Figure \ref{air_data}, we train the DRF with the measurements (intraday maximum) of the two pollutants PM$2.5$ and NO$_2$ as the responses, and the site longitude, latitude, elevation, land use and location settings as the predictors. We manually choose two decommissioned measurement sites (for which we have no measurements in years 2015-2019) as the test points. For each test point we obtain the weights to all training measurements. We further combine the weights for all measurements corresponding to the same site, which is represented by the symbol size in the top row. The bottom row shows the estimated distribution of the response, where the transparency (alpha) each training point corresponds to the assigned weight. We also add some estimated contours. 

For all plots in Figures \ref{air_data} and \ref{fig: functionals}, we train the single DRF with the same set of predictor variables and take the three pollutants O$_3$, SO$_2$ and PM$2.5$ as the responses. In this way we still have training data from many different sites (see the above discussion on missing data) and moreover, those are the 3 pollutants that most likely cross the threshold for the "Good" AQI category set by the EPA. Carbon monoxide (CO), for example, almost never crosses this threshold. 

In left plot of Figure \ref{comparison}, we compare the estimated CDF value with the standard classification forest which has the indicator $\1(\text{O}_3 < 0.055\text{ppm}, \text{SO}_2 < 36\text{ppb}, \text{PM}2.5 < 12.1\mu g/m^3)$ as the univariate response. In the right plot, we obtain the estimated CDF by fitting for each threshold a separate classification forest with an indicator $\1(\text{O}_3 \leq \text{threshold})$. We pick a test point such that the classification performs bad, just to illustrate that its estimated CDF need not be monotone, which cannot happen with DRF. In most of the cases, the estimated CDFs are very similar, as can also be seen from the left plot in Figure \ref{comparison}.

\subsection{Benchmark Analysis}
In this part we compare the resulting distributional estimates of DRF with several benchmark methods on a number of data sets. Because our target of estimation is now the whole conditional distribution one needs to use a distributional loss, and there appears to be no well-established choice in the literature. Furthermore, for any test point $\bold{x}_i$ we only have one observation $\bold{y}_i$ from $\P(\bold{Y}\mid \bold{X}=\bold{x})$, which makes performance evaluation of our estimator $\hat{\P}(\bold{Y} \mid \bold{X}=\bold{x})$ hard. We thus use the following performance measure: 

\begin{itemize}
\item \textbf{(NLPD loss)} For a fixed conditional distribution estimator $\hat{\P}(\bold{Y} \mid \bold{X}=\bold{x}_i)$, we sample a set of $m=500$ observations from which we estimate the conditional density via a Gaussian kernel estimator, using the $L_2$ loss with scale components and the median heuristic for the choice of the bandwidth parameter. We then evaluate the negative log-likelihood of the test observation $\bold{y}_i$ implied by the kernel estimate of the distribution and average over the test set (consisting of multiple pairs $(\bold{x}_i,\bold{y}_i)$). To reduce the dependence of these results on single large values of the log-likelihood, we use an $0.05-$trimmed mean to average the losses over the training set.
\end{itemize}
This loss definition provides a fair way to compare the ability to estimate the conditional distribution since most of the candidate methods only allow for sampling from the estimated conditional distribution $\hat{\P}(\bold{Y} \mid \bold{X}\myeq\bold{x}_i)$.

\subsubsection{Competing methods}
We compare DRF that uses the MMD splitting criterion with many existing methods that can be used for estimation of the conditional distribution.

\begin{itemize}
\item \textbf{Nearest Neighbor (k-NN)}: The standard k-nearest neighbors algorithm with the Euclidean metric. An estimated conditional distribution $\hat{\P}(\bold{Y}|\bold{X} = \bold{x})$ at a test point $\bold{x}$ is defined by a uniform distribution over the $k$ nearest observations in the training set. $k$ is chosen to be the square root of the training set size.

\item \textbf{Gaussian kernel (kernel)}: The estimate of the conditional distribution $\hat{\P}(\bold{Y}|\bold{X} = \bold{x})$ at a test point $\bold{x}$ is obtained by assigning to each training observation $(\bold{x}_i, \bold{y}_i)$ the weight proportional to the Gaussian kernel $k(\bold{x}, \bold{x}_i)$, analogously to usual kernel estimation methods. Median heuristic is used for bandwidth selection.

\item \textbf{Homogeneous distribution models}: This method makes the homogeneity assumption that the residuals have constant distribution and only the conditional mean changes. The estimate of the conditional distribution $\hat{\P}(\bold{Y}|\bold{X} = \bold{x})$ is obtained by first fitting a regression method of choice, computing the residuals, assigning the same weight to every residual and then adding those residuals to the predicted mean. We chose three different methods for the mean estimation:
\begin{enumerate}
\item \textbf{Random Forests (RF)}, a classical univariate regression forest is fitted independently for each response component;
\item \textbf{Extreme Gradient Boosting (XGBoost)}, a tree gradient boosting model (as described in \cite{Chen:XGB}) is fitted independently for each response;
\item \textbf{Deep Neural Network (DNN)}, a single deep neural network is fitted to predict the conditional mean of each response.
\end{enumerate}

\item \textbf{Conditional Generative Adversarial Neural Network (CGAN)}: The estimated conditional distribution $\hat{\P}(\bold{Y}|\bold{X} = \bold{x})$ is obtained through sampling from the discriminator with conditional feature $\bold{x}$. The implementation of the CGAN is taken from \cite{aggarwal2019benchmarking}. The architecture of the neural networks was taken to be the best one for the considered data sets among a set of candidates.

\item \textbf{Conditional Variational Auto-Encoder (CVAE)}: The estimated conditional distribution $\hat{\P}(\bold{Y}|\bold{X} = \bold{x})$ is obtained through sampling from the decoder of the CVAE with conditional feature $\bold{x}$. The implementation of the CVAE follows the one in \cite{sohn2015learning}. The architecture of the neural networks was taken to be the best one for the considered data sets among a set of candidates.

\item \textbf{Masked Autoregressive Flow (MAF)}: The estimated conditional distribution $\hat{\P}(\bold{Y}|\bold{X} = \bold{x})$ is obtained through sampling from the normalizing flow model with conditional feature $\bold{X}$. The implementation of the model follows the one presented in \cite{papamakarios2017masked}. The number of layers is chosen to be the best value from a set $\{5, 10\}$ for the considered data set.

\item \textbf{Conditional Mean Embedding (CME)}: The CME is calculated as in \eqref{formwewant} with the weights given as in e.g., \cite{CMEinDynamicalSystems, kernelmeanembeddingreview, OurapproachtoCME}. We choose both kernels to be Gaussian kernels with $\sigma=0.01$, as in \cite{OurapproachtoCME} and set $\lambda=0.01$. The estimated conditional distribution $\hat{\P}(\bold{Y}|\bold{X} = \bold{x})$ is obtained through sampling from the obtained weights, renormalized such that they lie in $[0,1]$ and sum up to one.

\end{itemize}

\subsubsection{Benchmark data sets}

Many benchmark data sets used come from the multiple target regression literature, where only the conditional means of the multivariate response is considered. We have used the data sets: \textbf{jura}, \textbf{slump}, \textbf{wq}, \textbf{enb}, \textbf{atp1d}, \textbf{atp7d}, \textbf{scpf}, \textbf{sf1} and \textbf{sf2} collected in the \texttt{Mulan} \citep{mtr} library.
Description about the dimensionality of the data sets, together with the descriptions of the outcomes and the regressors can be found in \citet{mtr} with links to the relevant papers introducing these data sets. In each data set categorical variables have been represented by the one-hot dummy encoding, the observations with missing data were removed together with constant regressors.

We additionally added $5$ data sets obtained from the data sets used in the main paper:
\begin{itemize}
    \item \textbf{copula}: Simulated Gaussian copula example where the response $Y$ is bivariate and whose marginal distribution is $N(0,1)$, but the correlation between $Y_1$ and $Y_2$ depends on  $X_1$.
    \item \textbf{birth1}: This data set is created from the CDC natality data and contains many covariates as predictors and the pregnancy length and birthweights as the responses. 
    \item \textbf{birth2}: This data set is similar as the above one, but we take pregnancy length as the predictor and add 3 more measures of baby's health as the response: APGAR score measured 5 minutes after birth and indicators whether there were any abnormal conditions and congenital anomalies.
    \item \textbf{wage}: This data set is created from the 2018 American Community Survey. We take the logarithmic hourly wage and gender as the response, as it was done in the fairness example in the main paper.
    \item \textbf{air}: This data set is obtained from the EPA air quality data. All six pollutants were taken as the response and we add both the information about the measuring site (location, which setting it is in, etc.), as well as the temporal information when the measurement has taken place (month, day of the week).
\end{itemize}

\subsection{Births data}
\paragraph{Data.}
This data set is obtained from the CDC Vital Statistics Data Online Portal (\url{https://www.cdc.gov/nchs/data_access/vitalstatsonline.htm}) and contains the information about the $\approx 3.8$ million births in 2018. However, as we do not need this many data points, we subsample $300'000$ of them. Even though the original data contains a lot of variables, we have taken only the following variables from the source data:
\begin{itemize}
    \item mother's age, height, weight before the pregnancy and BMI before pregnancy
    \item mother's race (black, white, asian, NHOPI, AIAN or mixed), marital status (married or unmarried) and the level of education (in total $8$ levels)
    \item father's age, race and education level
    \item month and year of birth
    \item plurality of the birth (how many babies were born at once)
    \item whether and when the prenatal care started
    \item length of the pregnancy
    \item delivery method (vaginal or C-section)
    \item birth order - the total number of babies born by the same mother (including the current one)
    \item birth interval - number of months passed since last birth (NA if this is the first child)
    \item number of cigarettes smoked per day on average during the pregnancy
    \item birthweight (in grams) and gender of the baby
    \item APGAR score (taken after 5min and 10min)
    \item indicators whether baby had any abnormal condition or some congenital anomalies
\end{itemize}

\paragraph{Analysis.}
After removing the data points with any missing entries and taking only the data points where the race of both parents is either black, white or Asian (for nicer plotting), we are left with $183'881$ data points. We use randomly chosen $100'000$ data points for training the DRF. We take the birthweight and the pregnancy length as the bivariate response and for the predictors we take: mother's age, race, education, marital status, height, BMI; father's age, race and education level; birth plurality, birth order, delivery method, baby's gender, number of cigarettes and indicator whether prenatal care took place.

For arbitrary test points from the data we can get the estimated weights by the fitted DRF, thus estimating the joint distribution of birthweight and pregnancy length conditional on all other variables mentioned above. Two such distributions are shown in Figure \ref{birthweight}. In addition we use the weights to fit a parametric model for the mean and $0.1$ and $0.9$ quantiles. This is done as follows:
\begin{itemize}
    \item We slightly upweight the data points where the pregnancy length is significantly above or below the usual range. This is to avoid the bulk of the data points to dominate the fit obtained for very long or short pregnancies.
    \item We apply the transformation $f(\cdot) = \log(\log(\cdot))$ on both the pregnancy length and the birthweight since then the scatterplots look much nicer.
    \item  We estimate the mean with smoothing splines with a small manually chosen number of degrees of freedom.
    \item The fitted mean is subtracted from the response (birthweight). The residuals seem well behaved with maybe slight, seemingly linear trend in standard deviation.
    \item We fit the $0.1$ and $0.9$ quantiles as the best linear functions that minimize the sum of quantile losses, by using the \texttt{quantreg} package \citep{koenker2012package}.
    \item The data is transformed back on the original scale by using the function $f^{-1}(\cdot) = \exp(\exp(\cdot))$.
\end{itemize}

For the right plot in Figure \ref{birthweight}, we have the following causal graph, as mentioned in the main paper:
$$\begin{tikzpicture}[
        > = stealth, 
        shorten > = 1pt, 
        auto,
        node distance = 3cm, 
        semithick 
    ]

    \tikzstyle{every state}=[
        draw = black,
        thick,
        fill = white,
        minimum size = 4mm
    ]

    \node[circle, draw=black] (T) at (0,0) {$T$};
    \node[circle, draw=black] (Z) at (2,1.3)  {$\boldsymbol{Z}$};
    \node[circle, draw=black] (L) at (2,-1) {$L$};
    \node[circle, draw=black] (B) at (4, 0) {$B$};
    \path[->] (T) edge node {} (L);
    \path[->,line width=1.4pt] (T) edge node {} (B);
    \path[dashed,->] (Z) edge node {} (T);
    \path[dashed,->] (Z) edge node {} (L);
    \path[dashed,->] (Z) edge node {} (B);
    \path[->] (L) edge node {} (B);
\end{tikzpicture}$$

We want to determine the direct effect (indicated in bold) of the twin pregnancy $T$ on the birthweight $B$ that is due to sharing of resources by the babies (space, food etc.) and is not due to the fact that twin pregnancy causes shorter pregnancy length $L$, which in turn causes the smaller birthweight. Another big issue is that we have confounding factors $\bold{Z}$ which can directly affect $B$, $L$ and $T$. For example, the number of twin pregnancies significantly depends on the parents' race, but so do the pregnancy length and the birthweight, e.g.\ black people have more twins, shorter pregnancies and smaller babies. We take all other variables as the potential confounders $\bold{Z}$ and adjust for all of them (mother's age, race, education, marital status, height, BMI; father's age, race and education level; birth plurality, birth order, baby's gender, number of cigarettes and indicator whether prenatal care took place). In order to do it, we fit the same DRF as before, where $\bold{Z}$ and $T$ are the predictors and ($B$, $L$) is the bivariate response for which we can fit the parametric model described above. We compute then the interventional distribution $\P(B\mid do(T=t, L=l))$ for all values of $t$ and $l$, by using the do-calculus to adjust the confounding $\bold{Z}$ via the backdoor criterion \citep{pearl2009causality}, where we also use the obtained parametric regression fit. In this way we can generalize the fit well, which is important when doing the do-calculus, since we are interested in some hypothetical combinations of covariates which might not occur frequently in the observed data, such as very long twin pregnancies.

\subsection{Wage data}
\paragraph{Data.}
The PUMS (Public Use Microdata Area) data from the 2018 1-Year American Community Survey is obtained from the US Census Bureau API (\url{https://www.census.gov/content/dam/Census/data/developers/api-user-guide/api-guide.pdf}). The survey is sent to $\approx 3.5$ million people annually and aims to give more up to date data than the official census that is carried out every decade. The 2018 data set has $3'214'539$ anonymized data points for the 51 states and District of Columbia. Even though the original survey contains many questions, we have retrieved only the subset of variables that might be relevant for the salaries:
\begin{itemize}
    \item person's gender, age, race (AIAN, black, white, asian, mix, NHOPI, other), indicator of hispanic origin, state of residence, US citizenship indicator (5 ordered levels), indicator whether the person is foreign-born
    \item person's marital status, number of own children in the same household and the number of family members in the same household
    \item person's education level (24 ordered levels) and level of English knowledge (5 ordered levels)
    \item person's employment status (employed, not at work, not in workforce, unemployed)
    \item for employed people we have annual salary earnings, number of weeks worked in a year and average number of hours worked per week
    \item for employed people we have employer type (government, non-profit company, for-profit company, self-employed), occupation (530 levels), industry where the person works (271 levels) and the geographical unit where the person works (59 levels)
    \item statistical weight determined by the US Census Bureau which aims to correct sampling bias
\end{itemize}

For our purposes, since we want to analyze the unfairness of the gender pay gap, we consider only employed people that are at least $17$ years of age, have worked full-time (at least $48$ weeks in a year) and have worked at least $16$ hours a week on average. We also omit the self-employed persons, since they often report zero annual salary and the pay gap there, if exists, cannot be called unfair as the salary is not determined by any employer. Since there are no missing data which would need to be omitted, we finally end up with $1'071'866$ data points.

\paragraph{Analysis.} 
We scale the salary with the amount of time spent working (determined from the number of weeks worked and average hours worked per week) to compute the logarithm of the hourly wages. The scaling with the time spent working is necessary, since full-time employed men spend on average $11\%$ more time working than women. The logarithmic transformation is used since the salaries are very skewed (positively) and logarithmic wages show nice behavior.

We also reduce the large number of levels of some of the categorical variables: for the occupation we use the group of $530$ jobs into $20$ categories provided in the SOC system (\url{https://www.bls.gov/soc/}); for the industry information we group the $271$ possibilities in $23$ categories as is done in the NIACS classification (\url{https://www.bls.gov/bls/naics.htm}); for the work place we group the $59$ US states and foreign territories into 9 economic regions (including the "abroad" category), as determined by the Bureau of the Economic Analysis (\url{https://apps.bea.gov/regional/docs/regions.cfm}).

We want to investigate how the logarithmic hourly wage $W$ is affected by the gender $G$, depending on the other factors $\bold{Z}$: age, race, hispanic origin, citizenship, being foreign-born, marital status, family size, number of children, education level, knowledge of English, occupation, industry type and place of work. To do this, we train DRF with bivariate response $(W, G)$ and predictors $\bold{Z}$ on a subsample of $300'000$ data points. With it we can answer the following: For fixed values of the covariates $\bold{Z}=\bold{z}$, what are the distributions of salaries of men and women. In addition, we can determine the "propensities", i.e. the proportion of men and women corresponding to $\bold{Z}=\bold{z}$. This information is displayed in the top row of Figure \ref{fig: wage} for a combination of covariates corresponding to some person in the left-out data. It illustrates how the distribution of salaries and their relationship can vary with different covariates $\bold{Z}$.

We do not only want to determine how different covariates $\bold{Z}$ affect the salary distribution, but we want to quantify the overall fairness of the pay, after appropriate adjustments. In Figure \ref{fig: wage}, we can see that the observed salaries of men and women differ noticeably, and this difference in the logarithmic wages means that an average woman has $17\%$ smaller salary than an average men. However, the question is how much of this difference is "fair". For example, the effect of the gender on the salary can be mediated through some variables such as, for example, the occupation, workplace or the level of education and we are only interested in the direct effect. This is illustrated in the following causal graph:

\begin{center}
\begin{tikzpicture}[
        > = stealth, 
        shorten > = 1pt, 
        auto,
        node distance = 3cm, 
        semithick 
    ]

    \tikzstyle{every state}=[
        draw = black,
        thick,
        fill = white,
        minimum size = 4mm
    ]

    \node[circle, draw=black] (G) at (0,0) {$G$};
    \node[circle, draw=black] (Z) at (2,1.3)  {$\boldsymbol{Z}$};
    \node[circle, draw=black] (W) at (4, 0) {$W$};
    \path[->,line width=1.4pt] (G) edge node {} (W);
    \path[->] (G) edge node {} (Z);
    \path[->] (Z) edge node {} (W);
\end{tikzpicture}
\end{center}

If we assume that people have the freedom to choose such variables themselves, the pay gap which arises from such different choices for men and women is fair and those variables are resolving variables \citep{kilbertus2017avoiding}. Another way that the pay gap can be explained is that some of the variables are not statistically independent of the gender in the population of full-time employed people (e.g.\ the race or the age), but they themselves have an effect on the salary.

In order to address those issues, we compute the distribution of the nested counterfactual $W(\text{male}, \bold{Z}(\text{female}))$, corresponding to the wages of a person that has characteristics $\bold{Z}$ as a woman, but which was treated as a man for obtaining the salary. Such distribution can be computed from the DRF, as described in the main paper: we randomly draw a female person and for its characteristics $\bold{z}$ we obtain the conditional distribution of wages of men with those characteristics $\P(W\mid G=\text{male}, \bold{Z}=\bold{z})$ via the weights. Those distributions are averaged over random draw of $1'000$ women (that were not used in the training step of the DRF). In case that the difference in salary is fair, the distribution of the counterfactual salary $W(\text{male}, \bold{Z}(\text{female}))$ should be exactly the same as the observed distribution of women's wages. However, we can see that this is not the case and that the median salaries of the two distributions differ by $11\%$. Even though this is smaller than the $17\%$ we obtain by comparing only the observational distributions, it still shows that women are paid less compared to men.

\section{Additional Synthetic Examples}
\label{appendix: additional examples}
\subsection{Univariate distributional regression.} 

The univariate response (case $d=1$) is by far the most studied case in the regression literature. However, at the level of the whole conditional distribution and compared to the multivariate case, the range of practically interesting targets $\tau(\bold{x})$ is quite reduced, e.g.\ conditional mean of some functional $\E(f(Y) \mid \bold{X})$ or conditional quantiles $Q_{\alpha}(Y \mid \bold{X})$. 
In Figure \ref{fig: univariate}, we have compared the performance of $\text{DRF}$ (which uses the MMD splitting criterion) with $3$ different tree-based univariate methods that can estimate the conditional quantiles in the univariate case:
\begin{itemize}
    \item \textbf{QRF}: the quantile regression forest introduced in \citet{meinshausen2006quantile}, which is equivalent to  $\text{DRF}_\text{CART}$ in the univariate case and uses the standard forest construction \citep{breiman2001random} to get the weights.
    \item \textbf{GRF}: the quantile forest proposed in \citet{athey2019generalized} based on the generalized random forest algorithm.
    \item \textbf{TRF}: the transformation forest, a model-based recursive partitioning approach, introduced in \citet{hothorn2017transformation}.
\end{itemize}

Additionally to the visual inspection of the performance given in Figure \ref{fig: univariate}, we present here a formal performance comparison for the three simulation scenarios also described in the main paper. The first two scenarios correspond exactly to the examples given in \citet{athey2019generalized} for the quantile  version of the GRF, which serve to illustrate its advantage compared to the conventional quantile regression forest (QRF) \citep{meinshausen2006quantile}. Scenario $3$, in addition, aims at assessing the ability to detect a change of distribution that does not relate to a change in the first two moments.

\begin{table}[h]
\scriptsize\setlength{\tabcolsep}{1pt}
\centering
\begin{tabular}{ccccccccccccccccc}
    \toprule \textbf{method} &
    \textbf{0.1} & \textbf{0.3} & \textbf{0.5} & \textbf{0.7} & \textbf{0.9} &\textbf{0.1} & \textbf{0.3} & \textbf{0.5} & \textbf{0.7} & \textbf{0.9} &\textbf{0.1} & \textbf{0.3} & \textbf{0.5} & \textbf{0.7} & \textbf{0.9}\\
    \midrule 
    $\text{DRF}$ & $\bf{0.180}$ \rule{0pt}{4ex}     & $\bf{0.353}$ & $\bf{0.402}$ & $\bf{0.349}$ & $\bf{0.177}$ & $\bf{0.267}$ \rule{0pt}{4ex}     & $\bf{0.518}$ & $\bf{0.589}$ & $\bf{0.514}$ & $\bf{0.264}$&  $0.140$ \rule{0pt}{4ex}     & $\textbf{0.298}$ & $0.371$ & $\textbf{0.351}$ & $\bf{0.198}$ \\
    $\text{QRF}$ & $0.182$ \rule{0pt}{4ex}     & $0.357$ & $0.482$ & $0.351$ & $0.179$ & $0.285$ \rule{0pt}{4ex}     & $0.526$ & $0.592$ & $0.521$ & $0.281$& $0.144$ \rule{0pt}{4ex}     & $0.299$ & $0.376$ & $0.357$ & $0.204$\\
    $\text{GRF}$ & $0.183$ \rule{0pt}{4ex}     & $0.359$ & $0.409$ & $0.354$ & $0.180$ &  $0.278$ \rule{0pt}{4ex}     & $0.522$ & $0.590$ & $0.517$  & $0.274$&$\bf{0.139}$ \rule{0pt}{4ex}     & $0.299$ & $\bf{0.371}$ & $\textbf{0.351}$ & $0.200$\\
    $\text{TRF}$ & $0.183$ \rule{0pt}{4ex}     & $0.358$ & $0.408$ & $0.353$ & $0.180$ &$0.272$ \rule{0pt}{4ex}     & $0.519$ & $0.590$ & $0.516$ & $0.268$&$0.145$ \rule{0pt}{4ex}     & $0.300$ & $0.373$ & $\bf{0.351}$ & $0.200$ \\
    $\text{5-NN}$ & $0.232$ \rule{0pt}{4ex}     & $0.402$ & $0.452$ & $0.404$ & $0.239$ &  $0.354$ \rule{0pt}{4ex}     & $0.587$ & $0.657$ & $0.584$ & $0.340$&  $0.187$ \rule{0pt}{4ex}     & $0.348$ & $.424$ & $0.406$ & $0.260$ \\
    $\text{20-NN}$ & $0.192$ \rule{0pt}{4ex}     & $0.368$ & $0.418$ & $0.366$ & $0.192$ &$0.290$ \rule{0pt}{4ex}     & $0.535$ & $0.606$ & $0.533$ & $0.283$&$0.146$ \rule{0pt}{4ex}     & $0.310$ & $0.382$ & $0.365$ & $0.211$ \\
    $\text{40-NN}$ & $0.187$ \rule{0pt}{4ex}     & $0.364$ & $0.413$ & $0.360$ & $0.185$ &$0.283$ \rule{0pt}{4ex}     & $0.528$ & $0.596$ & $0.522$ & $0.273$&   $0.141$ \rule{0pt}{4ex}   & $0.303$ & $0.376$ & $0.357$ & $0.204$\\
    \bottomrule
\end{tabular}
\caption{Average quantile losses for scenarios 1 (left), 2 (middle), 3 (right) over the repeated out-of-sample validations.}
\label{tab: scenarios}
\end{table}

The performance of each method is evaluated as follows: We consider the quantile (pinball) loss for the resulting quantile estimates provided by each candidate method for the different percentiles $\alpha \in \{0.1,0.3,0.5,0.7,0.9\}$. The losses are presented and computed based on repeated ($10$ times) out-of-sample validation (with a $70-30\%$ ratio between the training and testing sets sizes). The results are presented respectively for each scenario in Table \ref{tab: scenarios}. We additionally include the estimates obtained by $k$-nearest neighbor algorithm for several different values of $k$. 

\begin{table}[!htb]
\setlength{\tabcolsep}{1pt}

\centering

\caption{Average mean squared errors (MSE) for the three scenarios described above over $10$ repeated out-of-sample validations for estimating the conditional mean.}
\label{tab cond mean}

\medskip

\begin{tabular}{cccc}
    \toprule \textbf{method} &
    \textbf{SC1} & \textbf{SC2} & \textbf{SC3}  \\
    \midrule 
    $\text{RF}$ & $1.0545$  \rule{0pt}{4ex} & $2.4940$ & $0.9624$  \\
    $\text{DRF}$ & $\bf{1.0412}$ \rule{0pt}{4ex} & $\bf{2.4561}$ & $\bf{0.9340}$ \\
    \bottomrule
\end{tabular}
\end{table}

Furthermore, Table \ref{tab cond mean} shows non-inferiority of $\text{DRF}$ compared to the standard Random Forest for the classical task of estimating the conditional mean. We observe that $\text{DRF}$ has a good relative performance that makes it on par with existing algorithms, some of which specially designed for the problem of estimating conditional quantiles. Furthermore, it seems that the MMD splitting criterion improves the CART criterion for distributional regression in a general heterogeneous case (see e.g.\ scenarios 2 and 3), since the CART criterion is suitable only for detecting the change in the conditional mean, unlike MMD.

Dependence of the estimated quantiles on $X_1$ for each method (except the $k$-nearest neighbors) is displayed in the main paper in Figure \ref{fig: univariate}. In addition, the estimates of $2$-Wasserstein distance to the true conditional distribution, quantifying the difference in the estimated CDFs, are shown in Figure \ref{fig: wasserstein}.

\begin{figure}[h]
    \centering
    \includegraphics[width=1\linewidth]{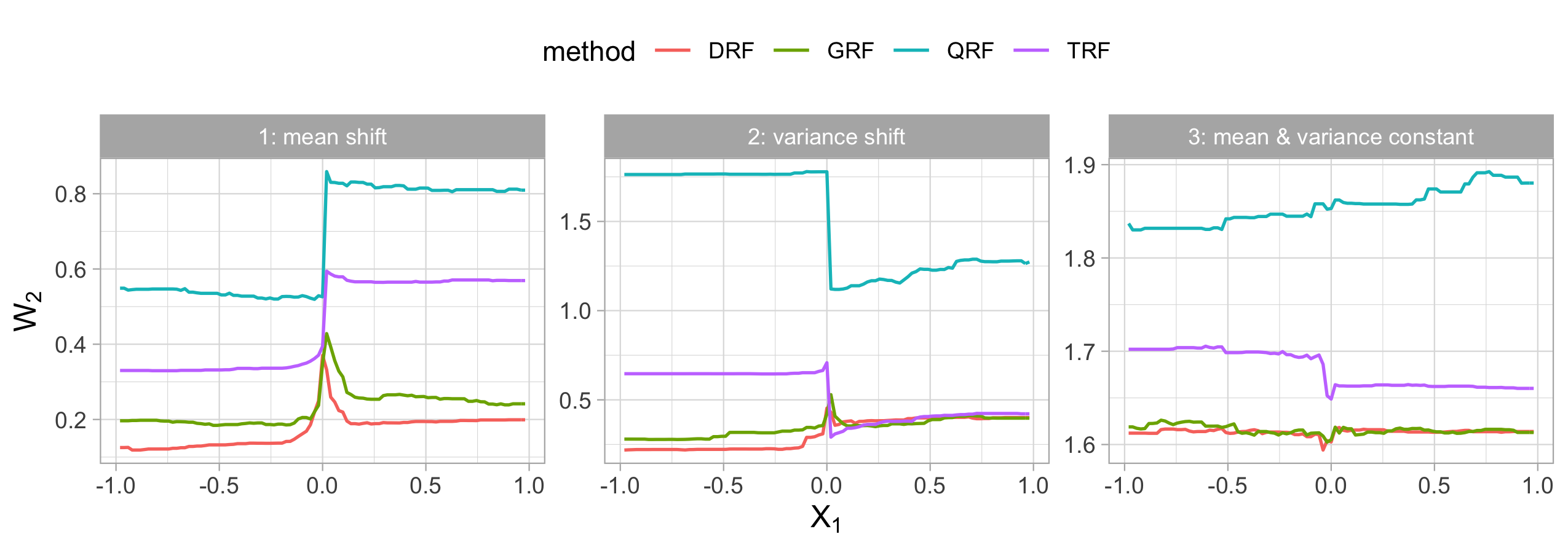}
    \caption{Scatter plot of discrete estimates of the 2-Wasserstein distance between the estimated and true conditional distribution against $X_1$ for a grid of test points of the form $(x_1, 0, \ldots, 0)$. The 2-Wasserstein distance is estimated over a grid of $100$ quantiles with levels equally spaced on $[0,1]$. Different colors corresponds to different methods: DRF (red), GRF (green), QRF (blue), TRF (purple).}
    \label{fig: wasserstein}
\end{figure}

\subsection{Heterogeneous regression and causal effects}
We explore here the performance of DRF on the synthetic data for the setup of heterogeneous regression, where we want to obtain the regression fit of $Y$ on the explanatory (or treatment) variables $\bold{W}$, but where this fit might change depending on values of $\bold{X}$. This can be done by DRF by using $\bold{X}$ as predictors and $(\bold{W}, Y)$ as the response and then using some standard regression method for regressing $Y$ on $\bold{W}$ in the second step, having already obtained the weights that describe the conditional distribution $\P((\bold{W}, Y) \mid \bold{\mathbf{X}=\mathbf{x}})$.

The most important such setup is when the data come from the following causal graph:
\begin{center}
\begin{tikzpicture}[
        > = stealth, 
        shorten > = 1pt, 
        auto,
        node distance = 3cm, 
        semithick 
    ]

    \tikzstyle{every state}=[
        draw = black,
        thick,
        fill = white,
        minimum size = 4mm
    ]

    \node[circle, draw=black] (W) at (0,0) {$\boldsymbol{W}$};
    \node[circle, draw=black] (X) at (2,1.3)  {$\boldsymbol{X}$};
    \node[circle, draw=black] (Y) at (4, 0) {$Y$};
    \path[->,line width=1.1pt] (W) edge node {} (Y);
    \path[->] (X) edge node {} (W);
    \path[->] (X) edge node {} (Y);
\end{tikzpicture}
\end{center}
In the case of such a causal graph, $\bold{X}$ are confounding variables, which we need to adjust for to understand the causal effect of $\bold{W}$ on $Y$. Not only can the marginal distributions of $Y$ and $\bold{W}$ be affected by $\bold{X}$, but also the regression fit (e.g.\ the regression coefficients).

\subsubsection{CATE and ATE}
One special case of this setup that is intensively studied in the causal literature is when $W$ is a (univariate) binary treatment variable. In this case we are interested in the distribution of the potential outcomes $Y(W=0)$ and $Y(W=1)$ and especially in their difference. It is commonly measured by using the Conditional Average Treatment Effect
$$\text{CATE}(\bold{x}) = \E[Y(W=1) - Y(W=0) \mid \bold{\mathbf{X}=\mathbf{x}}]$$
and the Average Treatment Effect 
$$\text{ATE} = \E[Y(W=1) - Y(W=0)] = \E[\text{CATE}(\bold{X})].$$

\paragraph{Competing methods.} We will compare the performance of the DRF with the following methods, specially designed for estimation of the CATE (or ATE)
\begin{itemize}
    \item Double Machine Learning (DML) of \citet{chernozhukov2018double}, which assumes the model $Y = m(\bold{X}) + W\theta + \epsilon$ with constant treatment effect and can thus only be used for estimating ATE and not CATE.
    \item X-learner (XL) introduced in \citet{kunzel2019metalearners} (the version with RF learners)
    \item Causal Forest (CF) introduced in  \citet{wager2018estimation, athey2019generalized} (we use the GRF version \citep{athey2019generalized} with local centering that substantially improves on the version in \citet{wager2018estimation})
\end{itemize}
In order to make the comparison fair, we use the local centering approach for DRF as well.

\paragraph{Data.} We will use the following data models for our simulations, where the first three are taken directly from \citet{athey2019generalized}:
\begin{enumerate}
    \item In this model $X_3$ is a confounder affecting both $W$ and $Y$:
    \begin{align*}
        &\mathbf{X} \sim U(0,1)^p, \quad W \mid \bold{X} \sim \text{Bernoulli}\left(\frac{1}{4}(1+\beta_{2,4}(X_3))\right),\\
        &Y \mid \bold{X}, W \sim 2\left(X_3 - \frac{1}{2}\right) + N(0,1),
    \end{align*}
    where $\beta_a,b(x)$ is the density of the beta random variable with parameters $a$ and $b$.
    
    \item In this model the treatment effect is heterogeneous, i.e. how $W$ affects $Y$ changes with $X_1$ and $X_2$:
    \begin{align*}
        &\mathbf{X} \sim U(0,1)^p, \quad W \mid \bold{X} \sim \text{Bernoulli}(0.5),\\
        & Y \mid \bold{X}, W \sim \left(W - \frac{1}{2}\right)\eta(X_1)\eta(X_2) + N(0,1),
    \end{align*}
    where $\eta(x) = 1+\left(1+e^{-20(x - \tfrac{1}{3})}\right)^{-1}.$

    \item This model is a combination of the previous two, so the treatment effect is heterogeneous and we have confounding:
    \begin{align*}
        &\mathbf{X} \sim U(0,1)^p, \quad W \mid \bold{X} \sim \text{Bernoulli}\left(\frac{1}{4}(1+\beta_{2,4}(X_3))\right), \\
        &Y \mid \bold{X}, W \sim 2\left(X_3 - \frac{1}{2}\right) + \left(W - \frac{1}{2}\right)\eta(X_1)\eta(X_2) + N(0,1).
    \end{align*}
    
    
    \item The following model is similar to above, with slightly different structure, where $X_2$ induces the confounding effects and $X_1$ makes the treatment heterogeneous:
    \begin{align*}
        &\mathbf{X} \sim U(0,1)^p, \quad W \mid \bold{X} \sim \text{Bernoulli}\left(\text{expit}(4X_2 - 2)\right), \\
        &Y \mid \bold{X}, W \sim 100X_2^2 + \left(W - \frac{1}{2}\right)\sin(3X_1) + N(0,1).
    \end{align*}
\end{enumerate}

\paragraph{Results.} For every model we generate $n$ data points $(X_1,\ldots, X_p, W, Y)_{i=1,\ldots,n}$. We run all methods and compute the root mean squared error of the obtained CATE estimate on a randomly generated test set $\mathbf{X}_{\text{test}}$ containing $1000$ data points. CATE corresponds to the coefficient of $W$ in the data generating mechanism of $Y$. We repeat the same procedure $100$ times and report the average result. For methods other than the DML, we estimate ATE by averaging the CATE estimates over the randomly generated test set. The results can be seen in Table \ref{CATE results} and Figure \ref{CATE plot}. Even though DRF is performing less well in general compared to the methods that are specially designed for the task of estimating CATE, we can still see that its estimates are fairly good.

\begin{table}[h]
\footnotesize\setlength{\tabcolsep}{2pt}
\begin{minipage}{.4\linewidth}
\centering

\caption{RMSE for the CATE, averaged over $1000$ test points and $100$ overall repetitions.}
\label{CATE results}

\medskip
\begin{tabular}{ccc@{\hskip 0.2in}ccc}
\toprule model & $n$ & $p$ & \textbf{DRF} & \textbf{CF} & \textbf{XL}\\
\hline
$1$ & $800$ & $10$ & $0.140$ \rule{0pt}{4ex} & $\bf{0.109}$ & $0.149$\\
$1$ & $1600$ & $10$ & $0.119$ \rule{0pt}{4ex}& $\bf{0.085}$ & $0.122$ \\
$1$ & $800$ & $20$ & $0.125$ \rule{0pt}{4ex} & $\bf{0.094}$ & $0.128$ \\
$1$ & $1600$ & $20$ & $0.105$ \rule{0pt}{4ex} & $\bf{0.076}$ & $0.107$ \\
\hline

$2$ & $800$ & $10$ & $0.452$ \rule{0pt}{4ex} & $0.319$ & $\bf{0.288}$\\
$2$ & $1600$ & $10$ & $0.285$ \rule{0pt}{4ex} & $0.234$ & $\bf{0.228}$ \\
$2$ & $800$ & $20$ & $0.568$ \rule{0pt}{4ex} & $0.336$ & $\bf{0.306}$ \\
$2$ & $1600$ & $20$ & $0.341$ \rule{0pt}{4ex} & $0.254$ & $\bf{0.241}$ \\
\hline

$3$ & $800$ & $10$ & $0.621$ \rule{0pt}{4ex} & $0.328$ & $\bf{0.319}$\\
$3$ & $1600$ & $10$ & $0.453$ \rule{0pt}{4ex} & $0.243$ & $\bf{0.237}$ \\
$3$ & $800$ & $20$ & $0.708$ \rule{0pt}{4ex} & $\bf{0.343}$ & $0.346$ \\
$3$ & $1600$ & $20$ & $0.533$ \rule{0pt}{4ex} & $0.257$ & $\bf{0.256}$ \\
\hline

$4$ & $800$ & $10$ & $0.320$ \rule{0pt}{4ex} & $\bf{0.273}$ & $0.682$\\
$4$ & $1600$ & $10$ & $0.285$ \rule{0pt}{4ex} & $\bf{0.228}$ & $0.389$ \\
$4$ & $800$ & $20$ & $0.316$ \rule{0pt}{4ex} & $\bf{0.289}$ & $0.722$ \\
$4$ & $1600$ & $20$ & $0.291$ \rule{0pt}{4ex} & $\bf{0.248}$ & $0.412$ \\
\bottomrule
\end{tabular}
\end{minipage}
\begin{minipage}{.1\linewidth}
\hspace{1cm}
\end{minipage}
\begin{minipage}{.4\linewidth}
\centering
\caption{RMSE for the ATE, averaged over $100$ repetitions.\\}
\label{ATE results}

\medskip
\begin{tabular}{ccc@{\hskip 0.2in}cccc}
\toprule model & $n$ & $p$ & \textbf{DRF} & \textbf{CF} & \textbf{XL} & \textbf{DML}\\
\hline
$1$ & $800$ & $10$ & $0.0841$ \rule{0pt}{4ex} & $\bf{0.0806}$ & $0.0858$ & $0.0843$\\
$1$ & $1600$ & $10$ & $0.0526$ \rule{0pt}{4ex}& $\bf{0.0517}$ & $0.0523$ & $0.0538$ \\
$1$ & $800$ & $20$ & $0.0786$ \rule{0pt}{4ex} & $\bf{0.0762}$ & $0.0802$ & $0.0785$ \\
$1$ & $1600$ & $20$ & $\bf{0.0585}$ \rule{0pt}{4ex} & $0.0588$ & $0.0609$ & $0.0625$ \\
\hline

$2$ & $800$ & $10$ & $\bf{0.0844}$ \rule{0pt}{4ex} & $0.0880$ & $0.0877$ & $0.0891$\\
$2$ & $1600$ & $10$ & $\bf{0.0567}$ \rule{0pt}{4ex}& $0.0587$ & $0.0593$ & $0.0584$ \\
$2$ & $800$ & $20$ & $0.0783$ \rule{0pt}{4ex} & $\bf{0.0767}$ & $0.0788$ & $0.0872$ \\
$2$ & $1600$ & $20$ & $\bf{0.0645}$ \rule{0pt}{4ex} & $0.0657$ & $0.0665$ & $0.0656$ \\
\hline

$3$ & $800$ & $10$ & $\bf{0.0914}$ \rule{0pt}{4ex} & $0.0916$ & $0.0932$ & $0.1116$\\
$3$ & $1600$ & $10$ & $\bf{0.0573}$ \rule{0pt}{4ex}& $0.0581$ & $0.0599$ & $0.0778$ \\
$3$ & $800$ & $20$ & $\bf{0.0858}$ \rule{0pt}{4ex} & $0.0917$ & $0.0922$ & $0.1114$ \\
$3$ & $1600$ & $20$ & $\bf{0.0623}$ \rule{0pt}{4ex} & $0.0673$ & $0.0599$ & $0.0925$ \\
\hline

$4$ & $800$ & $10$ & $\bf{0.1061}$ \rule{0pt}{4ex} & $0.1075$ & $0.2554$ & $0.9734$\\
$4$ & $1600$ & $10$ & $\bf{0.0677}$ \rule{0pt}{4ex}& $0.0665$ & $0.1028$ & $0.5542$ \\
$4$ & $800$ & $20$ & $\bf{0.1008}$ \rule{0pt}{4ex} & $0.1046$ & $0.2308$ & $2.7805$ \\
$4$ & $1600$ & $20$ & $\bf{0.0655}$ \rule{0pt}{4ex} & $0.0660$ & $0.0947$ & $1.834$ \\
\bottomrule
\end{tabular}

\end{minipage} 
\end{table}

\begin{figure}[h]
    \centering
    \includegraphics[width=1\linewidth]{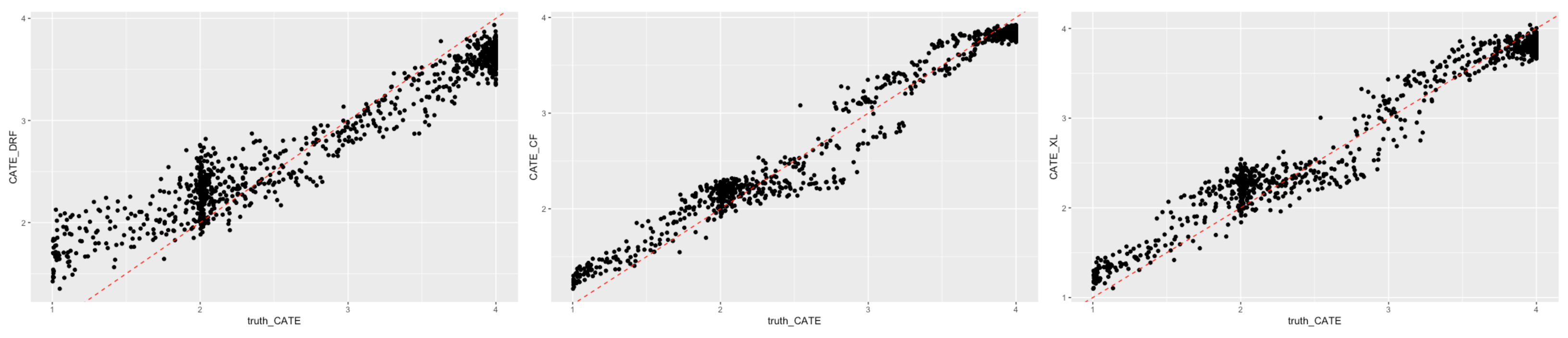}
    \caption{Estimates of the CATE for DRF (left), Causal Forest (middle), and X-learner (right) plotted against the true cate on the x-axis for Model $3$ with $n=1600$, $p=20$.}
    \label{CATE plot}
\end{figure}

\subsubsection{Continuous \texorpdfstring{$W$}{W}, linear treatment effect}
When the treatment variable $W$ is continuous, many methods designed for binary treatments, such as the X-learner \citep{kunzel2019metalearners} cannot be used. However, many important real-word examples fall within this framework. As an example, we might be interested in how the amount of medicine $W$ affects some biological parameter of interest $Y$ (conditionally on $\bold{X}$). When $W$ affects $Y$ linearly conditionally on $\bold{X}$, one can still use the Causal Forest (CF) \citep{athey2019generalized, wager2018estimation} method, which makes the splits based on the slope of the conditional linear fit $Y\sim W$. Due to its generality and versatility, DRF can trivially be used in such setting as well. 

To illustrate this, we consider the Model 3, as described in the previous section, which is also taken from \citet{athey2019generalized}, but where we change the distribution of the binary treatment variable so that it is continuous and it has a normal distribution with the same mean and variance, which depend on $\bold{X}$. In this model $W$ affects $Y$ linearly, which is a crucial assumption for the CF approach to work. We take $n=10000$ and $p=10$. The concept of CATE does not exist in this form in such setup and therefore we consider how the forest obtained by each method estimates both the intercept and the slope of the fit $Y\sim W$, conditionally on $\bold{X}$. The results can be seen in Figure \ref{CATE-cts}. We see that the estimate of the slope for DRF is slightly worse than for the CF, whose forest construction is specially designed for estimating the conditional slope. However, DRF estimates the intercept significantly better than the CF, especially in combination the local centering approach, which uses the centered data $Y - \widehat{Y}(\bold{X})$ and $W - \widehat{W}(\bold{X})$ instead. In this example the slope depends only on $X_2$, whereas the intercept depends on both $X_1$ and $X_2$. Since CF targets only the slope for forest construction, it will split mostly on $X_2$ and not on $X_1$, which leads to poor estimate of the intercept term. On the other hand, DRF splits both on $X_1$ and $X_2$, depending on the size of their effect on the joint distribution of $(W, Y)$. For many applications, especially in causality (see also the example in next section), it is essential to know the whole conditional distribution $\P(Y \mid W, \bold{X})$ so the DRF approach might be more beneficial than the CF. 

\begin{figure}
    \centering
    \includegraphics[width=1\linewidth]{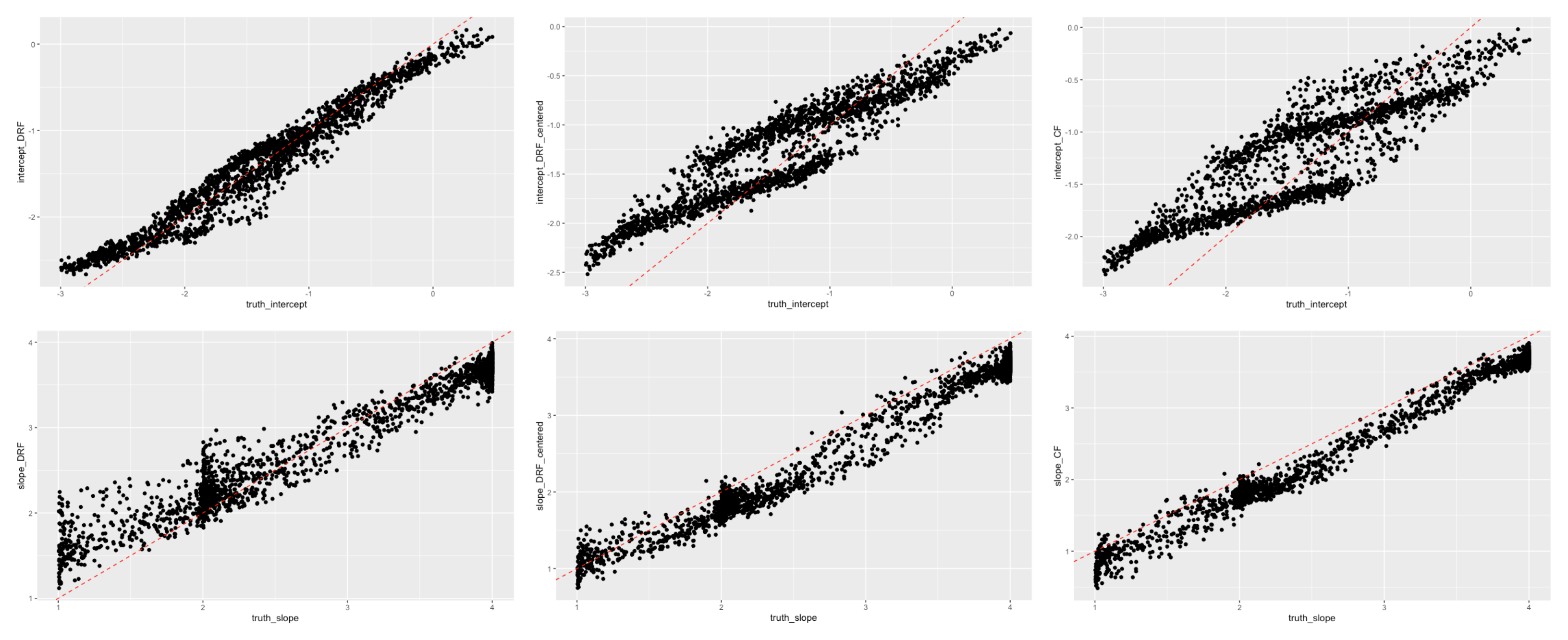}
    \caption{The estimates of the intercept (top row) and the slope (bottom row) for the linear fit conditional on $\bold{X}$, against the true values on the x-axis, obtained for the plain DRF (left), DRF with local centering as in \citet{athey2019generalized} (middle) and CF (right), which uses the local centering approach. The data is generated from Model 3 described in the previous section, with continuous treatment variable $W$ with the same mean and variance conditionally on $\bold{X}$.}
    \label{CATE-cts}
\end{figure}

\subsubsection{Nonlinear treatment effect}
There are very few methods that can estimate the treatment effect when the treatment variable $W$ is continuous and affects $Y$ nonlinearly, as it is commonly the case in real world settings. Here we demonstrate that DRF can be easily used in such setup as well, as opposed to the CF, which assumes a linear, though heterogeneous, treatment effect of $W$ on $Y$. We further show how the obtained regression fits $Y\sim W$ conditional on $\bold{X}$ can be used to estimate the causal effect $\E[Y \mid do(W=w)]$, as it is done in the main paper in the birth data example:
$$\E[Y \mid do(W=w)] = \int \E[Y \mid W=w, \bold{X}=\bold{x}]\P(\bold{X}=\bold{x})d\bold{x}.$$
We compare the performance of DRF with the straightforward and commonly used approach, where we first regress $Y$ on $(W,\bold{X})$ and use this regression fit which estimates $\E[Y \mid W=w, \bold{X}=\bold{x}]$ together with the above formula to estimate the causal effect.

\paragraph{Data} We consider the following example, similar to the previous examples:
\begin{align*}
    &\bold{X} \sim U(0,1)^p, \quad W \mid \bold{X} \sim \frac{1}{2}\left|1 + 4X_3 + N(0,1)\right|,\\ 
    & Y \mid \bold{X}, W \sim 3\left(X_3 - \frac{1}{2}\right) + 3X_1\sin(3W) + X_2N(0,1).
\end{align*}

Therefore, $W$ affects $Y$ highly nonlinearly through a sine function. $X_3$ is a confounding variable that affects the marginal distributions of $Y$ and $W$. $X_2$ regulates the error level for $Y$, whereas $X_1$ makes the treatment effect heterogeneous.
This is illustrated in the following plot:
\begin{figure}[H]
    \centering
    \includegraphics[width=1\linewidth]{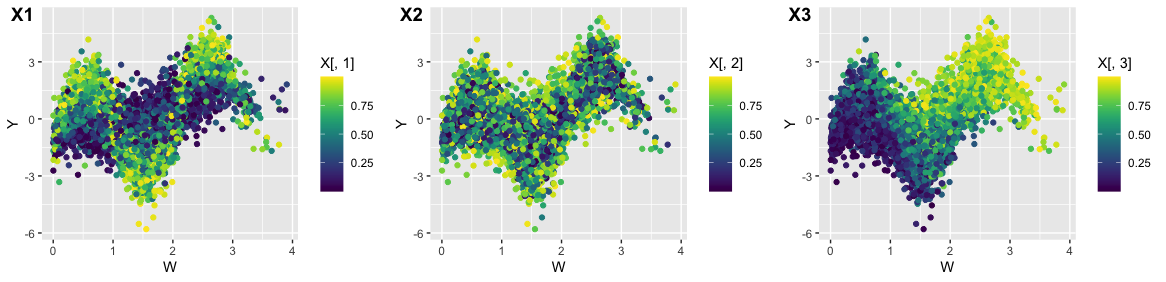}
    \caption{Visualization how $X_1, X_2, X_3$ affect the conditional nonlinear regression fit $Y\sim W$. $X_1$ changes the effect size, $X_2$ changes the noise level, whereas $X_3$ is a confounding variable which affects the means of $W$ and $Y$.}
    \label{heterogeneity}
\end{figure}

\begin{figure}[h]
    \centering
    \includegraphics[width=0.99\linewidth]{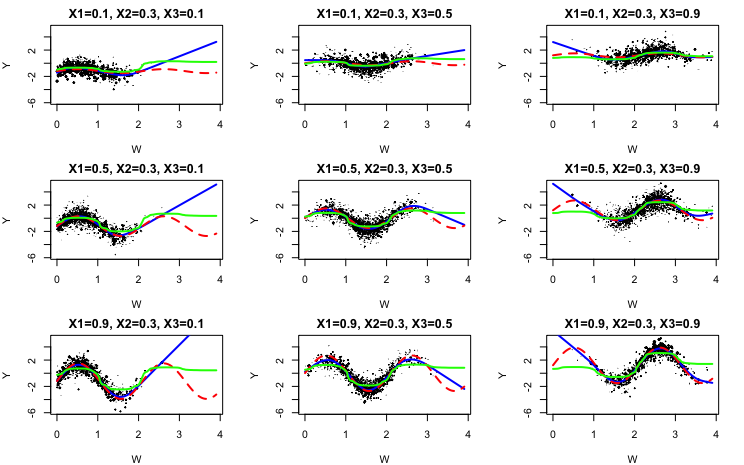}
    \caption{For a grid of values for the test point $\bold{x}$, the scatterplot illustrates the estimated joint distribution $(Y,W)$ by DRF. The subsequent regression fit using smoothing splines is denoted in blue, whereas the true conditional mean $\E[Y\mid W\myeq w, \bold{X}\myeq\bold{x}]$ is denoted with red dashed line. Green line shows the estimate of the conditional mean $\E[Y\mid W\myeq w, \bold{X}\myeq\bold{x}]$ with plain random forest.}
    \label{hetero_regression}
\end{figure}

\begin{figure}[h]
    \centering
    \includegraphics[scale=0.8]{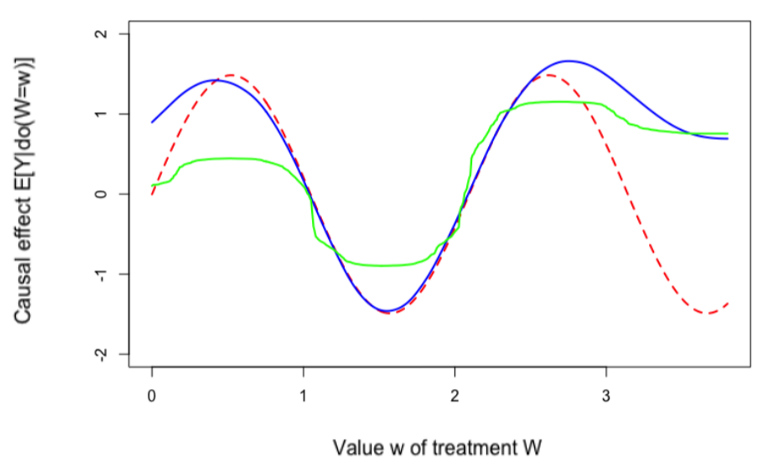}
    \caption{The estimated causal effect $\E[Y \mid do(W\myeq w)]$ with DRF (blue) and with conventional method which regresses $Y$ on $(W, \bold{X})$ using plain Random Forest (green). The true value is denoted by a red dashed line.}
    \label{causal_effect}
\end{figure}

\paragraph{Results}
In Figure \ref{hetero_regression} we can see the estimated joint distribution of $(Y,W)$ conditionally on $\bold{X}\myeq\bold{x}$, where the values of $X_1$ and $X_3$ vary, while the rest are fixed (even though $X_2$ also affects the conditional distribution, the effect is much weaker than for $X_1$ of $X_3$, see Figure \ref{hetero_regression}). We see that the estimated distribution matches the true regression line, denoted in red, very well. The estimated distribution induced by the DRF weights enables us to fit some specialised regression method for regressing $Y$ on $W$ for every fixed value of $\bold{X}$. The blue line indicates the fit obtained by using smoothing splines. Compared to the green line, which shows the predicted values for regression $Y\sim (W, \bold{X})$, it is nicer looking and also is able to extrapolate much better to values of $W$ which have low probability conditionally on $X$.

This extrapolation is crucial for causal applications, since for computing $\E[Y \mid do(W\myeq w)]$ we are interested in what would happen with $Y$ when our treatment variable $W$ is fixed to be $w$, regardless of which values are achieved by $\bold{X}$. However, it can easily happen that for this specific combination of $\bold{X}$ and $W$ there are very few observed data points, which makes the estimation hard \citep{pearl2009causality}. In this example, $W$ tends to be small for small values of $X_3$ and vice-versa and thus is hard to say what would happen with $Y$ when $X_3$ is large and $W$ is set to a small value by an outside intervention. 

In Figure \ref{causal_effect}, we indeed see that the estimates of the causal effect $\E[Y \mid do(W\myeq w)]$ by DRF are much better. One can still see that the error increases for the border values of $W$, which have small probability for some values of $X$, since the estimation there is much harder, but this error is much less pronounced for DRF than for the standard regression approach. 

\clearpage
\small
\bibliography{bibfile}

\end{document}